\documentclass[USenglish,oneside,twocolumn]{article}

\usepackage[utf8]{inputenc}
\usepackage[big]{pets_dgruyter_NEW}

\usepackage{amsmath}
\usepackage{amsfonts}
\usepackage{bm}

\def\eqref#1{equation~\ref{#1}}
\def\Eqref#1{Equation~\ref{#1}}

\def\1{\bm{1}}

\def\rl{{\textnormal{l}}}

\DeclareMathAlphabet{\mathsfit}{\encodingdefault}{\sfdefault}{m}{sl}
\SetMathAlphabet{\mathsfit}{bold}{\encodingdefault}{\sfdefault}{bx}{n}

\DeclareMathOperator*{\argmax}{arg\,max}

\usepackage{arydshln}

\makeatletter
\def\adl@drawiv#1#2#3{%
        \hskip.5\tabcolsep
        \xleaders#3{#2.5\@tempdimb #1{1}#2.5\@tempdimb}%
                #2\z@ plus1fil minus1fil\relax
        \hskip.5\tabcolsep}
\newcommand{\cdashlinelr}[1]{%
  \noalign{\vskip\aboverulesep
           \global\let\@dashdrawstore\adl@draw
           \global\let\adl@draw\adl@drawiv}
  \cdashline{#1}
  \noalign{\global\let\adl@draw\@dashdrawstore
           \vskip\belowrulesep}}
\makeatother

\usepackage{xparse}
\usepackage{xspace}
\usepackage{algorithm}
\usepackage[noend]{algpseudocode}
\usepackage{fixltx2e}
\usepackage{amsmath}
\usepackage{tikz}
\usepackage{float}
\usepackage{multirow}
\usepackage{multicol}
\usepackage{colortbl}
\usepackage{wrapfig}

\usepackage{adjustbox}
\usepackage{microtype}
\usepackage{graphicx}
\usepackage{subfigure}
\usepackage{booktabs} 
\usepackage[shortlabels]{enumitem}

\usepackage{hyperref}
\usepackage{amsthm}
\usepackage{enumitem}
\usepackage{wrapfig}
\usepackage{blindtext}
\usepackage{placeins}
\usepackage{bbm}

\usepackage{cuted}
\setlist[itemize]{leftmargin=*}
\setlist[enumerate]{leftmargin=*}

\usepackage[capitalize, noabbrev]{cleveref}
\crefalias{df}{definition}
\crefalias{lem}{lemma}
\crefalias{prop}{proposition}
\crefalias{thm}{theorem}
\crefalias{cor}{corollary}
\crefalias{conj}{conjecture}
\crefalias{emlp}{example}
\crefalias{rem}{remark}
\Crefname{algocf}{Algorithm}{Algorithms}
\Crefformat{equation}{Equation~(#2#1#3)}
\crefformat{equation}{Eq.~(#2#1#3)}

\newtheorem{theorem}{Theorem}[section]

\newtheorem{prop}[theorem]{Proposition}
\newtheorem{lemma}[theorem]{Lemma}

\newtheorem{definition}{Definition}
\renewcommand\qedsymbol{$\blacksquare$}

\renewcommand{\algorithmicrequire}{\textbf{Input: }}
\renewcommand{\algorithmicensure}{\textbf{Output: }}

\newcommand{\ie}{\textit{i.e.,}\@\xspace}
\newcommand{\eg}{\textit{e.g.,}\@\xspace}

\newcommand{\capc}{CaPC\@\xspace}

\newcommand{\single}{single-label\xspace}
\newcommand{\multi}{multi-label\xspace}

\newcommand{\Multi}{Multi-label\xspace}

\newcommand{\renyi}{R\'enyi\xspace}

\newcommand{\pate}{PATE\@\xspace}
\newcommand{\tpate}{$\tau$-PATE\@\xspace}

\renewcommand{\binary}{Binary\@\xspace}
\newcommand{\powerset}{Powerset\@\xspace}

\newcommand{\answeringParties}{\textit{answering parties}\@\xspace}

\newcommand{\teacherModels}{\textit{teacher models}\@\xspace}

\newif\ifdraft
\draftfalse

\newcommand{\new}[1]{\textcolor{black}{#1}}

\ifdraft

\newcommand{\ahmad}[1]{\textcolor{red}{[Ahmad: #1]}}

\newcommand{\chris}[1]{\textcolor{red}{Chris: #1}}
\newcommand{\vinith}[1]{\textcolor{blue}{Vinith: #1}}
\newcommand{\yunxiang}[1]{\textcolor{cyan}{Yunxiang: #1}}

\definecolor{chocolate(traditional)}{rgb}{0.48, 0.25, 0.0}

\definecolor{darkpastelgreen}{rgb}{0.01, 0.75, 0.24}
\newcommand{\natalie}[1]{\textcolor{darkpastelgreen}{natalie: #1}}
\definecolor{amber(sae/ece)}{rgb}{1.0, 0.49, 0.0}
\newcommand{\adam}[1]{\textcolor{red}{[Adam: #1]}}
\newcommand{\franzi}[1]{\textcolor{red}{[Franzi: #1]}}
\newcommand{\jiaqi}[1]{\textcolor{blue}{Jiaqi: #1}}
\newcommand{\emmy}[1]{\textcolor{orange}{Emmy: #1}}
\newcommand{\amrita}[1]{\textcolor{red}{Amrita: #1}}

\else

\newcommand{\chris}[1]{}
\newcommand{\vinith}[1]{}
\newcommand{\adam}[1]{}
\newcommand{\yunxiang}[1]{}
\newcommand{\natalie}[1]{}
\newcommand{\franziska}[1]{}
\newcommand{\franzi}[1]{}
\newcommand{\jiaqi}[1]{}
\newcommand{\emmy}[1]{}
\newcommand{\amrita}[1]{}
\newcommand{\ahmad}[1]{}

\fi

\DOI{foobar}

\cclogo{\includegraphics{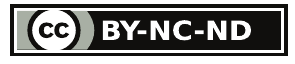}}
  
\begin{document}

  \author*[1]{Corresponding Author}

  \author[2]{Second Author}

  \author[3]{Third Author}

  \author[4]{Fourth Author}

  \affil[1]{Affil, E-mail: email@email.edu}

  \affil[2]{Affil, E-mail: email@email.edu}

  \affil[3]{Affil, E-mail: email@email.edu}

  \affil[4]{Affil, E-mail: email@email.edu}

  \title{\huge Private Multi-Winner Voting for Machine Learning}

    \title{\huge Private Multi-Winner Voting for Machine Learning}


\begin{abstract}
{
Private multi-winner voting is the task of revealing $k$-hot binary vectors satisfying a bounded differential privacy (DP) guarantee. This task has been understudied in machine learning literature despite its prevalence in many domains such as healthcare. We propose three new DP multi-winner mechanisms: Binary, $\tau$, and Powerset voting. Binary voting operates independently per label through composition. $\tau$ voting bounds votes optimally in their $\ell_2$ norm for tight data-independent guarantees.
Powerset voting operates over the entire binary vector by viewing the possible outcomes as a power set. Our theoretical and empirical analysis shows that Binary voting can be a competitive mechanism on many tasks unless there are strong correlations between labels, in which case Powerset voting outperforms it.
We use our mechanisms to enable privacy-preserving multi-label learning in the \emph{central setting} by extending the canonical single-label technique: PATE. We find that our techniques outperform current state-of-the-art approaches on large, real-world healthcare data and standard multi-label benchmarks. We further enable multi-label `confidential and private collaborative' (CaPC) learning and show that model performance can be significantly improved in the multi-site setting.
}
\end{abstract}

  \journalname{Proceedings on Privacy Enhancing Technologies}
\DOI{Editor to enter DOI}
  \startpage{1}
  \received{..}
  \revised{..}
  \accepted{..}

  \journalyear{..}
  \journalvolume{..}
  \journalissue{..}

\maketitle

\section{Introduction}

\label{Introduction}
Differential privacy techniques for machine learning have predominantly focused on two techniques: differentially private stochastic gradient descent (DPSGD)~\cite{dp-sgd} and private aggregation of teacher ensembles (PATE)~\cite{papernot2017semi}. At the core of these techniques are two mechanisms. The Gaussian mechanism enables arbitrary private queries of data (e.g., gradients in DPSGD~\cite{dp-sgd}). Instead, the noisy $\argmax$ can only reveal the max count. This is used to reveal the predicted label in \textit{Private Aggregation of Teacher Ensembles} (PATE)~\cite{papernot2017semi}.
These two mechanisms appeal well to the canonical
\textit{\single} classification per input (a.k.a. multi-class classification).

In contrast, more real-world tasks, such as \emph{\multi} classification~\cite{tsoumakas2007multi}, can be modeled using multi-winner elections~\cite{10.5555/3298239.3298313,Faliszewski2017MultiwinnerVA}. These are settings where more than one candidate can win, i.e., each input can have $>1$ class present. Outside of 
elections, other multi-winner election scenarios include 
canonical computer visions tasks like object recognition, where models must recognize all objects present in an image~\cite{boutell2004learning}. Another scenario is the task of inferring which topics were written about in a corpus of text (document): e.g., a news article may discuss politics, finance, and/or education.
The principal setting we consider is healthcare, 
in which patient data (\eg symptom reports or X-rays) may be indicative of multiple conditions~\cite{chexpert2019, johnson2019mimiccxrjpg}. 

We thus focus on creating private mechanisms for releasing the outcome of a multi-winner election. We first formalize the multi-winner election. Then, we propose \binary voting, a simple yet powerful solution that answers the $k$ single-winner (for each class) elections independently. We prove that it is optimal when there is a lack of correlation among the outcomes of particular candidates and show empirically that this mechanism can outperform state-of-the-art baselines based on DP-SGD.
Recognizing that correlations can often exist in data, we then propose $\tau$ and Powerset voting for obtaining tighter guarantees. $\tau$ voting obtains a tight data-independent privacy bound by limiting the number of votes from each voter, obtained via an $\ell_2$ bound on their ballot. Powerset voting is created by casting the multi-winner election to an analogous single-winner election---thus, Powerset voting reveals the result for all $k$ candidates jointly.

By replacing the noisy $\argmax$ mechanism in (\single) PATE, we create \multi PATE for DP \multi semi-supervised machine learning. We conduct extensive empirical evaluation on large datasets, including Pascal VOC, a common \multi benchmark, as well as CheXpert~\cite{chexpert2019}, MIMIC~\cite{johnson2019mimiccxrjpg}, and PadChest~\cite{padchest2020} which are healthcare datasets. Despite different public data assumptions, we compare against DPSGD (including an improved adaptive variant~\cite{AdaptiveDPSGD2021}) because it is the only private baseline in \multi classification thus far. 
We find that with the modest assumption of public unlabeled data (required for PATE), we achieve new state-of-the-art DP \multi models with 85\% AUC (17\% better than DPSGD) on Pascal VOC. 
Because many \multi settings may benefit from distributed, multi-site learning (e.g., healthcare), we further integrate our methods with the \textit{CaPC} protocol for \textit{Confidential and Private Collaborative} learning~\cite{capc2021iclr}.

We find that training with our mechanisms can improve model performance significantly on large, real-world healthcare data with sensitive attributes. Our main contributions are as follows:
\begin{enumerate}
    \item We create three new DP aggregation mechanisms for private multi-winner voting: \binary voting, $\tau$ voting, and \powerset voting. We show theoretically and empirically that \binary voting performs better unless there is high correlation among labels.
    
    \item We enable private \multi semi-supervised learning that achieves SOTA performance on large real-world tasks: Pascal VOC and 3 healthcare datasets.  

    \item We enable \multi collaborative learning in the multi-site scenario and show empirically that this significantly improves model performance.  
\end{enumerate}

\section{Background and Related Work}\label{sec:background}
\subsection{DP for Machine Learning}
\label{ssec:dp}

Differential Privacy (DP) is the canonical framework for measuring the privacy leakage of a randomized algorithm~\cite{dwork2006calibrating}. It requires the mechanism (in our work, the training algorithm), to produce statistically indistinguishable outputs on any pair of \textit{adjacent} datasets: those differing by any but only one data point. This bounds the probability of an adversary inferring properties of the training data from the mechanism's outputs.

\begin{definition}[Differential Privacy]
\label{differential_privacy}
A randomized mechanism $\mathcal{M}$ with domain $\mathcal{D}$ and range $\mathcal{R}$ satisfies $(\varepsilon, \delta)$-differential privacy if for any subset $\mathcal{S} \subseteq \mathcal{R}$ and any adjacent datasets $X, X' \in \mathcal{D}$, i.e. $\|X - X'\|_{1} \leq 1$, the following inequality holds: $
{\rm Pr}\left[\mathcal{M}(X) \in \mathcal{S}\right] \leq e^{\varepsilon} {\rm Pr}\left[\mathcal{M}(X') \in \mathcal{S}\right] + \delta.$
\end{definition}

We use \renyi Differential Privacy (RDP)~\cite{mironov2017renyi}. Because RDP bounds the privacy loss with the \renyi-divergence, it enables tighter accounting for our mechanisms that use Gaussian noise. RDP is a generalization of 
$(\varepsilon,\delta=0)-DP$. We use standard conversions to report $(\varepsilon,\delta)-DP$.

\begin{definition}[Rényi Differential Privacy]
\label{renyi_differential_privacy}
A randomized mechanism $\mathcal{M}$ is said to satisfy $\varepsilon$-Rényi differential privacy of order $\lambda$, or $(\lambda, \varepsilon)$-RDP for short, if for any adjacent datasets $X, X' \in \mathcal{D}$:
\begin{flalign*}
&D_{\lambda}(\mathcal{M}(X) \, || \, \mathcal{M}(X')) = &&\\
&= \frac{1}{\lambda - 1} \log\mathbb{E}_{\theta \sim \mathcal{M}(X)}\left[\left(\frac{{\rm Pr}[\mathcal{M}(X) = \theta]}{{\rm Pr}[\mathcal{M}(X') = \theta]}\right)^{\lambda - 1}\right] \leq \varepsilon, &&
\end{flalign*}
\end{definition}
where $D_{\lambda}(P\, || \, Q)$ is the \renyi Divergence between distributions $P$ and $Q$ defined over a range $\mathcal{R}$.
It is convenient to consider RDP in its functional form as $\varepsilon_{\mathcal{M}}(\lambda)$, which is the RDP $\varepsilon$ of mechanism $\mathcal{M}$ at order $\lambda$. Our privacy analysis builds on the following result.

\begin{lemma}
\label{lemma:rdp-gaussian}
\textbf{RDP-Gaussian mechanism (for a \single setting)}~\cite{mironov2017renyi}. Let $f:\mathcal{X} \rightarrow \mathcal{R}$ have bounded $\ell_2$ sensitivity for any two neighboring datasets $X, X^\prime$, i.e., $\left\lVert f(X) - f(X^\prime) \right\rVert_2 \le \Delta_2$. The Gaussian mechanism $\mathcal{M}(X) = f(X) + \mathcal{N}(0,\sigma^2)$ obeys RDP with $\varepsilon_{\mathcal{M}}(\lambda) = \frac{\lambda \Delta_2^2}{2 \sigma^2}$.
\end{lemma}

Another notable advantage of RDP over $(\varepsilon,\delta)$-DP is that it composes naturally. This will be helpful when analyzing our approach which repeatedly applies the same private mechanism to a dataset. 


\begin{lemma}
\label{lemma:rdp-composition}
\textbf{\new{RDP Adaptive Sequential Composition}}~\cite{mironov2017renyi}. \new{Given a sequence of $k$ adaptively chosen mechanisms 
$\mathcal{M}_1,\dots,\mathcal{M}_k$, 
 each saatisfying 
 $(\lambda,\varepsilon_i)$-RDP, then 
their union
satisfies $(\lambda,\sum_{i=1}^{k} \varepsilon_i)$-RDP.}
\end{lemma}

\subsection{PATE and \capc}
\label{sec:PATE}

\textbf{PATE} (Private Aggregation of Teacher Ensembles)~\cite{papernot2017semi} is a \single semi-supervised DP learning approach. First, an ensemble of models is trained on (disjoint) partitions of the training data. Each (teacher) model is then asked to predict (label) a test input by voting for one class. To provide the DP guarantee, only the noisy $\argmax$ of the teacher votes are released---instead of each vote directly. To do so, teacher votes are organized into a histogram where $n_i(x)$ indicates the number of teachers that voted for class $i$. Then,
${\rm argmax}\{n_i(x) + \mathcal{N}(0, \sigma_{G}^2)\}$ is released, where the Gaussian variance $\sigma_{G}^2$ controls the prediction's privacy loss.
Finally, these DP-labeled test inputs are used to train a student model which can be used by the model owner. Note that unlike DPSGD, this requires the additional but modest assumption of unlabeled data.
Here, each noisy label incurs additional privacy loss (rather than gradients in DPSGD). Loose data-independent guarantees are obtained through advanced composition~\cite{dwork2014algorithmic}. Tighter data-dependent guarantees, which can be safely released with additional analysis, are possible when many teachers agree on the predicted label~\cite{papernot2018scalable}.
To further reduce privacy loss, \textit{Confident GNMax} only reveals predictions when $\max_i\{ n_i(x)\} +  \mathcal{N}(0, \sigma_{T}^2) > T$, \ie there is high consensus among teachers.

\textbf{\capc} (Confidential and Private Collaboration) is a distributed collaborative learning framework that extends \pate from the central to the collaborative setting~\cite{capc2021iclr}. In addition to bounding the training data DP privacy loss, \capc also protects the confidentiality of test data and teacher model parameters by introducing new cryptographic primitives. In \capc, teacher models are known as \answeringParties and are distributed across sites, communicating via these cryptographic primitives---we use \answeringParties when referring specifically to \capc and \teacherModels otherwise.

We now clarify the two privacy protection terms we use.
\textit{Confidentiality} refers to the notion that no other party (or adversary) can view, in plaintext, the data of interest---it protects the inputs to our mechanism. This differs from \textit{DP} which reasons about what can be inferred from the outputs of our mechanism, in our case about the training set for each teacher model. Both protect data privacy in different ways. 
In \capc, each answering party is a unique protocol participant with their own (teacher) model. The querying party initiates the protocol by (1) encrypting unlabeled \new{data;} other (answering) parties return an encrypted label. Then, (2) the encrypted vote is secretly shared with both the answering party and the Privacy Guardian (PG), such that each has only one share (neither knows the label in plaintext); the PG follows PATE to achieve an \new{$(\varepsilon,\delta)-DP$} bound for all the answering parties. Finally, (3) the PG returns the DP vote to the querying party via secure computation. The goal is to maximize the number and quality of votes that can be returned to the querying party under a chosen \new{$(\varepsilon,\delta)-$}DP guarantee.


\subsection{Prior work on \Multi Classification}
The closest to our work is that of~\cite{privateknn2020cvpr}, which provides a privacy-preserving nearest neighbors algorithm formulated using PATE. The neighboring training points are the teachers and their labels the votes. Our mechanisms can similarly be used to extend~\cite{privateknn2020cvpr} to the \multi setting. However, their work assumes the existence of a publicly-available embedding model (used to project the high-dimensional data to a low-dimensional manifold for meaningful nearest-neighbor distance computation). This is a strong assumption that is often not applicable to the multi-site scenario we consider \new{and especially for sensitive tasks like healthcare.} 
Our approach does not make this assumption. 
Further, we show in Section~\ref{sec:mainscheme} that their $\ell_1$ vote clipping scheme is not optimal for the commonly used Gaussian distribution. Instead, we prove tighter bounds by clipping votes to an $\ell_2$ ball. 

The multi-label voting problem can also be formulated as simultaneously answering multiple counting queries which has been studied extensively in the theoretical literature, e.g.,~\cite{dagan2020boundednoise}. 
These have yet to show promise empirically in settings such as \multi voting. We explored these approaches in our setup (not shown) finding looser data-independent guarantees compared to our proposed mechanisms.
Another alternative \multi method is differentially private stochastic gradient descent (DPSGD), which is agnostic to the label-encoding because the mechanism releases gradient updates. \citet{abadi2016deep} proposed and evaluated this approach in the \single setting. \citet{AdaptiveDPSGD2021} applied it to the multi-label setting to train \multi models (for $k=5$ labels). 
We compare with~\citep{AdaptiveDPSGD2021} and find that our methods perform better. Many works explore using public labeled datasets to reduce privacy costs, e.g.,~\cite{feifei_2021_CVPR}---this is a stronger assumption than ours.
Finally, we stress that our evaluation includes large-scale healthcare datasets with a natural need for privacy and the additional challenge of being severely imbalanced.

\new{Private \textit{heavy hitters} is a related problem to private multi-label classification, where a set of $n$ users each has an element from a universe of size $d$. The goal is to find the elements that occur (approximately) most frequently, but do so under a local differential privacy constraint~\cite{NIPS2017_heavy_hitters, Zhu2020FederatedHH}. However, these methods release the value of the heavy hitter rather than the argmax.}
\section{Differential Privacy Mechanisms for Multi-Winner Elections} \label{sec:mainscheme}

Multi-winner election systems are direct generalizations of
popularly used single-winner elections. They are adopted commonly in real-life elections and their comprehensive analysis can be found in~\cite{10.5555/3298239.3298313} and~\cite{Faliszewski2017MultiwinnerVA}. Here, we formally define a multi-winner system and propose DP mechanisms to facilitate the private release of the outcome of a multi-winner election.

Compared to single-winner elections, multi-winner elections are more challenging for designing a DP mechanism: single-winner elections only output a single scalar per collection of ballots, indicating the sole winner, whereas multi-winner elections output a
vector of winners and thus have higher sensitivity (use more $\varepsilon$) per query.
DP mechanisms for single-winner voting have been well studied before, where the noisy $\argmax$ achieves tight
privacy loss resulting from tight sensitivity and \textit{information minimization}~\cite{dwork2014algorithmic}. However, to the best of our knowledge, no prior work formally defines and analyzes DP mechanisms for multi-winner election systems; hence, we provide and analyze Definition~\ref{def:multi-winner-vote}
to then get formal
privacy loss bounds in Section~\ref{sec:multi-pate}.

\begin{definition} 
\label{def:multi-winner-vote}
\textbf{$\theta$-multi-winner election.}
Given a set of $n$ voters, each with a ballot of $k$ candidates (coordinates)
${\bm b}\in\{0,1\}^k$ such that $\|{\bm b}\|_2 \leq \theta$. The outcome of a $\theta$-multi-winner election with threshold $T$ is defined as
\begin{equation*}
    f({\bm b}_1,\ldots,{\bm b}_n) \triangleq \left\{i : \sum_{j=1}^n {\bm b}_j[i] > T\right\},
\end{equation*}
\end{definition}
where $T=\frac{n}{2}$ and with the binary decision per label $i$, $\sum_{j=1}^{n} \bm{b}_j[i]$ represents the number of positive votes. Note that this definition is deterministic since no noise is added and setting the threshold $T=\frac{n}{2}$ is equivalent to using the number of negative votes ($n - \sum_{j=1}^{n} \bm{b}_j[i]$) as $T$.

\subsection{DP $\infty$-Multi-Winner Elections}\label{sec:sqrtk}
Observe that in~\Cref{def:multi-winner-vote}, we only require the release of a binary vector. Intuitively, this makes other related private mechanisms that release the entire privatized histogram sub-optimal, \eg multiple counting queries and the Gaussian mechanism---they should require no less (and often much more) budget than mechanisms that only release the binary vector because of information minimization~\cite{dwork2014algorithmic}. 

Thus, we instead focus on private release of the $\infty$-multi-winner-election\footnote{We use $\infty$ to emphasize that any vote $\bm{b} \in \{0, 1\}^k$ is allowed but a tighter bound exists: $\|{\bm b}\|_2 \leq \sqrt{k}<\infty$.}. Because the noisy $\argmax$ is a well-studied mechanism for single-winner elections, 
our first multi-winner approach is to decompose it into $k$ separate single-winner elections. Then, we can obtain a tight final guarantee on the multi-winner ballot via RDP composition over the $k$ independent applications (Lemma~\ref{lemma:rdp-composition}).
We call this \binary voting and state it formally in Definition~\ref{def:binary-voting} below. We again use RDP composition to obtain a total privacy budget across separate queries (instances of the multi-winner elections). 
\begin{definition}[DP Binary Voting Mechanism]\label{def:binary-voting}
For an $\infty$-multi-winner election, the binary voting mechanism based on noisy $\argmax$ is:
    \begin{equation*}
        \mathcal{M}_\sigma({\bm b}_1,\ldots,{\bm b}_n) \triangleq \left\{i : \left( \sum_{j=1}^n {\bm b}_j[i] + \mathcal{N}(0, \sigma^2_G) \right) > T_i \right\},
    \end{equation*}
\end{definition}
where $T_i = n-\sum_{j=1}^{n} \bm{b}_j[i]$ is the number of negative votes. \new{We use this per-candidate threshold $T_i$ instead of a fixed $T=n/2$ because this reduces the privacy cost. However, to preserve privacy, we must add noise to both the positive and negative vote counts (each bin of the histogram). Though we show only the noise added for positive votes, this is equivalent to adding less noise for each label. Note that this definition follows directly from PATE by viewing the $i$-th label release as a binary task of the label being present or not.}

The distinguishing feature of the above \binary mechanism is that it releases a binary result independently, per candidate. Thus, there are many possible \binary voting mechanisms satisfying this goal. We base ours on the noisy $\argmax$ mechanism because of its advantageous privacy guarantees stemming from \textit{information minimization}~\cite{dwork2014algorithmic}: changing the training set to any adjacent one can only modify a single ballot. Thus, this mechanism has advantageous properties both theoretically, and as well empirically as presented in~\cite{papernot2017semi,papernot2018scalable}. 
Next, we show that an independent, per-candidate voting mechanism (e.g., our \binary voting) is optimal in the coordinate-independent setting (Definition~\ref{def:coordindep} below). 
\begin{definition} 
\label{def:coordindep}
A function $f(X):\new{\mathcal{R}^d} \rightarrow \mathcal{R}^k$ is \textit{coordinate-independent} if the $i$'th output coordinate of $f$, $f_i$, is determined only by a unique subset $P_i(X)$ of the input $X$. \new{$P_i$ is a function that partitions (selects) columns of its input $X$ and uniqueness means that that $\forall i,j$ if $i \neq j$ then $P_i \cap P_j = \emptyset$ so no columns are shared by any $P_i$.} 
\end{definition}
We first build an intuition for our guarantee and then provide the formal statement (in Proposition~\ref{prop:binary-pate}). Recall from Section~\ref{sec:background} that the crux of DP analysis is to bound the sensitivity of the random mechanism $\mathcal{M}$, which is the noisy output of the function\footnote{Because the random mechanism differs minimally from the function, we use the two interchangeably.} $f$. 
\begin{definition} \label{def:sensitivity}
Sensitivity is the maximum deviation 
\begin{equation*}
\Delta_p f= \underset{(X,X'): ||X-X'||_1 = 1}{\max} ||f(X)-f(X')||_p 
\end{equation*}
of $f$'s output for any adjacent pair of inputs $X$ and $X'$, i.e., have a Hamming distance $||X-X'||_1$ of $1$, on their rows~\cite{dwork2014algorithmic}.
\end{definition}

Recall $\varepsilon_{\mathcal{M}}(\lambda) \propto\Delta_p \mathcal{M}$; thus, we seek to decompose a \multi function (with range $k$) in terms of the sensitivities of its corresponding $k$ \single functions, i.e., the scenario of \binary voting.
Observe that each output $f_i: i\in [k]$ of a \binary voting mechanism 
depends only on a unique subset of columns
$P_i(X)$ of the input ballots $X$. \new{Without loss of generality, let each output coordinate be binary, i.e., $f_i(X) \in \{0,1\}$, and let} each output depend on only one input column, i.e., let \new{$d=k$} and $P_i(X)=X_i$ such that $f_i(X)=f_i(X_i)$.
Then, $\Delta_p f$ \new{is upper bounded by $k^{1/p}$.
To see this, recall sensitivity is defined as the function's maximum $p$-norm deviation between any (but only) two datasets $X$ and $X'$. Thus, pick $X$ and $X'$ so that each binary output $f_i(\cdot)$ is flipped between predicted (1) and not predicted (0).}
\new{If $f$ is coordinate-independent, then its sensitivity will always be $k^{1/p}$ because $f_i(X)=f_i(X_i)$ and thus $X$ ($X'$) are the horizontal stacks of $X_i$
 ($X'_i$) such that $f_i(X_i)-f_i(X'_i)=1$.
}
\new{In other words, }the sensitivity of the \multi function is the sum of sensitivities of the \single functions that comprise it. 
Because \binary voting is applied independently on each \single function, it thus achieves \new{the} optimal sensitivity in the coordinate-independent setting.
Proposition~\ref{prop:binary-pate} below formally states our result. Appendix~\ref{app:binary-pate} contains supporting lemmas and the proofs. 
\begin{prop}\label{prop:binary-pate}
For a coordinate-independent multi-label 
function 
$f(X): \new{\mathcal{R}^d} \rightarrow \mathcal{R}^k$, the $\ell_1$ sensitivity 
$\Delta_1 f$ is equal to that of \binary voting applied per label and thus simultaneous private release of $f$ has no tighter privacy loss than $k$ applications of \binary voting.
\end{prop}
Proposition~\ref{prop:binary-pate} shows that \binary voting has optimal sensitivity if $f$ is coordinate (candidate) independent. However, 
if the function were coordinate-dependent, \binary voting may be sub-optimal. To see this, consider the case where output coordinates $i$ and $j$ both solely depend on the input coordinate $l$. Then it may not be possible to find a pair of databases that simultaneously flips both outputs $i$ and $j$. For instance, if $\max f_i$ occurs at $X_l = a$ and $\min f_i$ at $X_l = b$, but one of $\max f_j$ or $\min f_j$ does not occur at these values, then now $\Delta_p f_j$ cannot be maximized while simultaneously maximizing $\Delta_p f_i$.
Thus, $\Delta_p f<\sum_i^k \Delta_p f_i$. 
Since the sensitivity may be less, the privacy loss may also be lower. This implies that if the underlying function is coordinate-dependent, \binary voting may now have larger sensitivity and thus a sub-optimal privacy loss.
We will discuss one method for leveraging such correlations in Definition~\ref{def:powerset-pate} and this overall approach in Section~\ref{sec5:limitations}.

We also find that \binary voting may be sub-optimal when the sensitivity is calculated via $p>1$ norms. This is a direct result of H\"older's Inequality~\cite{holder1889ueber}. We explore this in Section~\ref{ssec:tvote} by clipping in the $\ell_2$ norm.

\subsection{DP Mechanisms for $\tau$-Multi-Winner Election}\label{ssec:tvote}
Many real-world ballots may not have many chosen candidates, e.g., in Pascal VOC there are on average fewer than $2$ positive labels out of $20$ \new{, enabling $\tau<\sqrt{k}$}.
Thus, instead of allowing arbitrarily many votes per ballot, the $\tau$-multi-winner election restricts the votes in their $\ell_2$ norm, in turn limiting the sensitivity of the mechanism. In doing so, 
tighter data-independent privacy guarantees are obtained.  
However, this can come at a cost in utility if ballots were cast with more candidates selected than the chosen $\tau$ allows for, because the ballot must now be `clipped'. This is analogous to \textit{norm clipping} in DPSGD~\cite{dp-sgd}, which is commonly studied in DP literature.
To be meaningful, we require $\tau<\sqrt{k}$; otherwise, this is equivalent to \binary voting (Definition~\ref{def:binary-voting}). 

\begin{definition}[DP $\tau$ Voting Mechanism]
\label{def:pate-clip}
\textbf{Mechanism for $\tau$-Multi-Winner Election.} For a $\tau$-multi-winner election the $\tau$ Voting mechanism is
\begin{align*}
\mathcal{M}_{\sigma}(\bm{b}_1,\ldots,\bm{b}_n) & \triangleq \bigg\{i: \bigg(\sum_{j=1}^{n} \bm{v}_j[i] + \mathcal{N}(0,\sigma_G^2) \bigg) > T_i \bigg\}, & \\ \bm{v}_j & \triangleq \min(1, \frac{\tau}{\left\lVert \bm{b}_j\right\rVert_2})\bm{b}_j, &
\end{align*}
\end{definition}

\noindent where we choose $T_i = n-\sum_{j=1}^{n} \bm{v}_j[i]$, which represents the number of \textit{clipped} negative votes. The clipping itself was done on the positive votes (through the $v_j$'s). To obtain the number of negative votes, we subtract the number of clipped positive votes from the total number of votes (in this case $n$) and thus call these \textit{clipped} negative votes. To analyze the privacy loss of this new mechanism, we propose an extension of Lemma~\ref{lemma:rdp-gaussian} for this multi-voting task.
\begin{lemma}
\label{lemma:rdp-gaussian-multilabel}
\textbf{RDP-Gaussian mechanism for a \multi setting}. Let $f:\mathcal{X} \rightarrow \mathcal{R}^{k}$ obey $\left\lVert f(X) - f(X^\prime) \right\rVert_2 \le \Delta_2$ for neighboring datasets $X, X^\prime$. The Gaussian mechanism $\mathcal{M}(X) = f(X) + \mathcal{N}_{k}(0,\sigma^2I)$ obeys RDP with $\varepsilon_{\mathcal{M}}(\lambda) = \frac{\lambda \Delta_2^2}{2 \sigma^2}$.
\end{lemma}
The following proof is similar to the one for Proposition 7 in~\citep{mironov2017renyi} (the crucial steps are 6 and 7 to transform the integral and obtain the Gauss error function):
\begin{proof}
\begin{align}
&D_{\lambda} (\mathcal{M}(X) || \mathcal{M}(X'))\\
&=D_{\lambda} (\mathcal{N}_k(f(X), \sigma^2I) || \mathcal{N}_k(f(X'), \sigma^2I)) \\
&=\frac{1}{\lambda-1} \log \bigg\{ \frac{1}{\sigma^k\sqrt{(2\pi)^k}} \int_{-\infty}^{\infty} \exp \bigg( \frac{-\lambda}{2\sigma^2} \left\lVert \theta - f(X) \right\rVert_2^2 \bigg)  \\
& \cdot \exp \bigg( \frac{-(1-\lambda)}{2\sigma^2} \left\lVert \theta - f(X') \right\rVert_2^2 \bigg) d\theta \bigg \} \\
&=\frac{1}{\lambda-1} \log \bigg\{ \frac{1}{\sigma^k\sqrt{(2\pi)^k}} \int_{-\infty}^{\infty} \\
&\exp \bigg( \frac{-\lambda \left\lVert \theta - f(X) \right\rVert_2^2 -(1-\lambda)\left\lVert \theta - f(X') \right\rVert_2^2}{2\sigma^2} \bigg) d\theta \bigg \} \\
&=\frac{1}{\lambda-1} \log \bigg\{ \frac{\sigma^k\sqrt{(2\pi)^k}}{\sigma^k\sqrt{(2\pi)^k}} \exp \frac{(\lambda-1)\lambda}{2\sigma^2} \left\lVert f(X) - f(X') \right\rVert_2^2 \bigg\} \\
&= \frac{\lambda \left\lVert f(X) - f(X') \right\rVert_2^2}{2\sigma^2} \le \frac{\lambda \Delta_2^2}{2 \sigma^2}
\end{align}
\end{proof}

\subsubsection{Norms for Clipping}
\label{sec:norms-clipping-appendix}
Recall that RDP provides tight privacy analysis for the Gaussian mechanism and that it analyzes sensitivity using the $\ell_2$ norm. This motivates our choice of $\ell_2$ vote clipping, because it exactly bounds the mechanism's sensitivity. Compared with the $\ell_1$ clipping proposed in
~\cite{privateknn2020cvpr}
our bound is no larger, and often much \textit{smaller}, than theirs.
We can see this by comparing the privacy loss in terms of the mechanism's sensitivity (see Definition~\ref{def:sensitivity}).
\begin{align}
\epsilon_{\mathcal{M}}(\lambda) =
\frac{\lambda \tau_2^2}{\sigma^2} = \frac{\lambda \Delta_2^2}{2\sigma^2} \le \frac{\lambda \Delta_1^2}{2\sigma^2} = \frac{\lambda (2\tau_1)^2}{2\sigma^2} = \frac{\lambda 2\tau_1^2}{\sigma^2}
\end{align}
where in the first inequality we use the fact that $\left\lVert x \right\rVert_2 \le \left\lVert x \right\rVert_1$. Empirically, we confirm that $\tau$ voting achieves a much tighter bound on many \multi datasets. Take $\tau_1$ to represent clipping in the $\ell_1$ norm and similarly $\tau_2$ to represent clipping in the $\ell_2$ norm. On Pascal VOC with $k=20$ labels, we find that we can set a much smaller $\tau_2=1.8$ but only a $\tau_1=3.4$ without deteriorating the mechanism's
balanced accuracy.
This leads to a $>\hspace{-1mm}6$x tighter (data-independent) privacy loss with the $\ell_2$ norm clipping compared to~\cite{privateknn2020cvpr} and \binary voting.
In practice, we select $\tau$ to be marginally larger than the average p-norm of labels in the training data. 

\subsection{Casting $\infty$-Multi-Winner Elections as Single-Winner Elections}
Similar to Section~\ref{sec:sqrtk}, we again explore how single-winner elections can be used to reveal outcomes of multi-winner elections. Instead, here we remove the independence assumption and cast the multi-winner election to one single-winner election (not $k$ independent ones). This improves on one of the major shortcomings of \binary voting: by simultaneously revealing all $k$ candidate's results in one outcome, we can leverage correlations in their results to reduce the privacy loss.

We achieve this new mechanism by encoding each of the $2^k$ different outcomes as a separate candidate for a single-winner voting mechanism. We call this mechanism the \powerset voting mechanism. 
\begin{definition}[DP Powerset Mechanism]\label{def:powerset-voting}
Denote the \powerset operator as $\mathcal{P(\cdot)}$. For a $\infty$-multi-winner election with ordered outcomes $\mathcal{P}(\{0,1\}^k)$, where $|\mathcal{P}(\{0,1\}^k)| = 2^k$, 
\begin{flalign*}
    & \mathcal{M}_\sigma (\bm{b}_1,\ldots,\bm{b}_n) \in \{0,1\}^k & \\ & \triangleq \argmax_i \left\{\mathcal{P}(\{0,1\}^k)_i: \sum_{j=1}^n \mathbbm{1}_{\bm{b}_j = \mathcal{P}(\{0,1\}^k)_i} + \mathcal{N}(0, \sigma_G^2) \right\}
\end{flalign*}
\end{definition}
\noindent where $\mathbbm{1}$ represents the indicator function. Importantly, this mechanism only approximates the true Multi-Winner solution because it can only reveal an outcome that was cast as a ballot. If the set of true plurality candidates had no associated ballot, then this result cannot be returned. However, with enough ballots, the resulting utility degradation is reduced.

\section{From Private Elections to Private Multi-Label Classification}
\label{sec:multi-pate}
Single-winner elections, and their DP mechanisms, have played an important role in privacy-preserving machine learning. Notably, PATE adopted the noisy $\argmax$ mechanism to enable training of \single models via semi-supervised knowledge transfer. By replacing the noisy $\argmax$ mechanism with our proposed DP multi-winner mechanisms, we thus create new \multi PATE systems for private \multi semi-supervised learning.


\textbf{Binary PATE, $\tau$ PATE, and Powerset PATE: }
We refer readers to Section~\ref{sec:PATE} for a detailed description of \single PATE. Multi-label PATE is nearly unchanged from a systems perspective because we only replace the DP mechanisms used. Notably, both forms of PATE leverage a bank of unlabeled data which is used for the semi-supervised knowledge transfer; this data is often widely available since it is unlabeled~\cite{papernot2017semi}. Each query uses some $\varepsilon-DP$ which, once the total composed RDP budget is exhausted, terminates the labelling process. Our mechanisms are used to label and calculate the per-query $\varepsilon$. These $m$ privately labeled data points are then used to train the student model which can be released. 
All of our mechanisms are compatible with propose-test-release~\cite{dwork2009differential} (used in Confident GNMax) which we leverage to reduce the privacy budget usage. We show the full $\tau$-PATE in~\Cref{alg:extended-pate}.

\begin{algorithm}
  \caption{\textbf{$\tau$ PATE}: \Multi classification with $\tau$-clipping in $\ell_2$-norm and with the Confident GNMax.}\label{alg:extended-pate}
  
  \algorithmicrequire{Data point $x$, clipping threshold $\tau_2$, Gaussian noise scale $\sigma_G$, Gaussian noise scale for Confident GNMax $\sigma_T$, $n$ teachers, each with model $f_j(x)\in \{0,1\}^k$, where $j \in [n]$.}
  
  \algorithmicensure{Aggregated vector $V$ with $V_i=1$ if returned label/feature present, otherwise $V_i=0$.} 

  \begin{algorithmic}[1]
      \ForAll{teachers $j \in [n]$}
         \State $v_j \gets \min(1, \frac{\tau_2}{\left\lVert f_j(x) \right\rVert_2}) f_j(x)$\Comment{$\tau$-clipping in $\ell_2$ norm}
      \EndFor
      \State $V^1 = \sum_{j=1}^{n} v_j$ \Comment{Number of positive votes per label}
      \State $V^0 = n - V^1$ \Comment{Number of negative votes per label}

      \ForAll{labels $i \in [k]$}
         \If {$\max\{V^0_i,V^1_i\} + \mathcal{N}(0,\sigma_T) < T$}
            \State $V_i = \perp$ \Comment{Confident-GNMax}
         \Else
             \State $V^0_i \gets V^0_i +  \mathcal{N}(0,\new{0.5\sigma^2_G)}$ \Comment{Add Gaussian noise}
             \State $V^1_i \gets V^1_i +  \mathcal{N}(0,\new{0.5\sigma^2_G})$ \Comment{for privacy protection}
             \If {$V^1_i > V^0_i$} \Comment{Decide on the output vote}
                \State $V_i = 1$
             \Else
                \State $V_i = 0$
             \EndIf
         \EndIf
      \EndFor
  \end{algorithmic}
\end{algorithm}
\vspace{-1em}

\section{Analysis of Multi-Label PATE}
We now analyze the privacy guarantees of our mechanisms within \multi PATE. Importantly, we have designed our mechanisms to be compatible with the data-dependent analysis of~\cite{papernot2018scalable} (by making them extensions of the noisy $argmax$) because it is often tighter than the data-independent guarantee which we discuss below.

\subsection{Data-Independent and Data-Dependent Analysis}\label{ssec:independent-analysis}
\textbf{Data Independent Analysis:}We first compare our data-independent guarantees with respect to the noisy $\argmax$ mechanism in a \single setting, for one query. When it satisfies a $\varepsilon_{\mathcal{M}}(\lambda)$-RDP, our \binary voting achieves a data-independent bound of $\sum_{i=1}^k\varepsilon_{\mathcal{M}}(\lambda)_i$, $\tau$-voting of $\frac{\lambda\tau^2}{\sigma^2}$, and \powerset voting of $\varepsilon_{\mathcal{M}}(\lambda)$. 

\textbf{Data-Dependent Analysis:} However, tighter privacy budgets may be reported when there is high empirically observed consensus on the binary vector, \ie there is a high likelihood for the returned outcome.
The tight data-dependent analysis of~\cite{papernot2018scalable} bounds the privacy loss in terms of the probability of the returned \single outcome, calculated from the histogram of teacher votes.
%
This data-dependent privacy loss is often much lower than the data-independent bounds above---even after the required sanitation before release. This sanitation is achieved by smooth sensitivity analysis of the mechanism to protect against information leakage resulting from the release of the budget---because its particular depends on the data itself. We follow the steps and additionally report the expected sanitized privacy budgets.

Because each of our \multi mechanisms view the votes in different manners, the resulting data-dependent privacy budgets can vary significantly. Below, we compare the data-dependent guarantees for \binary and \powerset PATE, where we empirically observe that $\tau$ PATE only marginally differs from \binary PATE resulting from how $\ell_2$ clipping impacts the vote distributions.

\label{sec:powerset-binary-appendix}
\subsection{Comparing Binary and Powerset PATE}
We begin with preliminaries on our \binary and \powerset mechanisms, and the data-dependent DP analysis of~\cite{papernot2018scalable}.

Recall that \binary PATE has $2\cdot k$ events: the presence or absence independently for each of the $k$ labels. Instead, \powerset PATE operates on the entire binary vector, where each configuration (subset) of this vector is an event---thus, there are $2^k$ events. This leads to three concrete differences of \powerset PATE: (1) it has exponentially more possible events, (2) when \new{any} labels \new{$i\in [k]$} occur simultaneously, it rapidly improves the privacy loss, and
(3) it only adds noise once to the entire binary vector. These directly impact the calculated data-dependent bounds.

The data-dependent privacy bound is a function of the probability distribution between events, estimated using the observed votes as follows. Calculate the difference between the two most likely events (label presence/absence in \binary voting, binary vector subsets in \powerset voting).
The gap is this difference scaled by the inverse of the noise standard deviation, which bounds the probability of any non-plurality event occurring.
To provide a final guarantee for the mechanism, as in \single PATE, we upper bound the Renyi divergence using Theorem~\ref{thm:pate-rdp-bound} below. The main random variable influencing this privacy loss is the gap, $q(\bar{n})$, which we calculate using Proposition~\ref{prop:gap} below. Theorem~\ref{thm:pate-rdp-bound} shows how to select ideal higher order moments $\mu_1,\mu_2$ and $\varepsilon_1,\varepsilon_2$, using the data-dependent value for $q(\bar{n})$ so that the RDP bound is a nonlinear function of $q(\bar{n})$ solely.

\begin{prop}[From~\cite{papernot2018scalable}]\label{prop:gap}
For a GNMax aggregator $\mathcal{M}_\sigma$, the teachers' votes histogram $\bar{n}=(n_1,\cdots,n_k)$, dataset $X$, and for any $i^*\in \mathcal{P}(\{0,1\}^k)$, where $\mathcal{P(\cdot)}$ denotes the \powerset, we have 
\begin{equation*}
    \mathbf{Pr}[\mathcal{M}_\sigma(X) \neq i^*] \leq q(\bar{n}),
\end{equation*}
and
\begin{equation*}
    q(\bar{n}) \triangleq \frac{1}{2} \sum_{i\neq i^*} erfc\left(\frac{n_i^*-n_i}{2\sigma}\right)
\end{equation*}
Where this is a minimal modification of~\cite{papernot2018scalable} to go from \single PATE to \multi \powerset PATE.
\end{prop}

\begin{theorem}[From~\cite{papernot2018scalable}]\label{thm:pate-rdp-bound}
Let $\mathcal{M}$ be a randomized algorithm with $(\mu_1, \varepsilon_1)$-RDP and $(\mu_2, \varepsilon_2)$-RDP guarantees and suppose that there exists a likely outcome $i$ given a dataset $X$ and bound $\Tilde{q} \leq 1$ such that $\Tilde{q} \geq \Pr[\mathcal{M}(X) \neq i]$. Additionally suppose the $\lambda \leq \mu_1$ and $\Tilde{q} \leq exp^{(\mu_2-1)\varepsilon_2}\big(\frac{\mu_1}{\mu_1-1}\cdot \frac{\mu_2}{\mu_2-1}\big)^{\mu_2}$. Then for any neighboring dataset $X'$ of $X$, we have:
\begin{align*}
& D_\lambda(\mathcal{M}(X) || \mathcal{M}(X')) \leq & \\
& \frac{1}{\lambda-1}\log\big((1-\Tilde{q})\cdot\mathbf{A}(\Tilde{q}, \mu_2, \varepsilon_2)^{\lambda-1} + \Tilde{q}\cdot \mathbf{B}(\Tilde{q}, \mu_1, \varepsilon_1)^{\lambda-1}\big) &
\end{align*}
where $\mathbf{A}(\Tilde{q}, \mu_2, \varepsilon_2) \overset{\Delta}{=} (1-\Tilde{q})/\bigg(1-(\Tilde{q}e^{\varepsilon_2})^{\frac{\mu_2-1}{\mu_2}}\bigg)$ and $\mathbf{B}(\Tilde{q}, \mu_1, \varepsilon_1) \overset{\Delta}{=} exp^{\varepsilon_1}/\Tilde{q}^{\frac{1}{\mu_1-1}}$.
\end{theorem}


When the top $3$ vote counts, $n_1 > n_2 > n_3$ satisfy $n_1 - n_2, n_2 - n_3 \gg \sigma$ (the gap is sufficiently large), the RDP bound can be well approximated as 
\begin{equation}
\varepsilon^{powerset}_{\mathcal{M}_\sigma}(\lambda) \leq exp(-2\lambda/\sigma^2)/\lambda, \textit{where} \, \lambda = (n_1-n_2)/4,
\label{eq:simple-rdp}
\end{equation}
which is from~\cite[Corollary 11]{papernot2018scalable}. Though $q(\bar{n})$ is the formal gap, we will often refer to $\tilde{q}(\bar{n})=n_1 - n_2$ as the gap too when these assumptions are satisfied.
We plot the cdf of the gaps across votes in CheXpert and Pascal VOC in Figure~\ref{fig:cdf-gaps} of Appendix~\ref{app:additional-experiments}, finding that \binary PATE is indeed in this regime. Assuming \powerset PATE is too,
we analyze under what circumstances it outperforms 
\binary PATE. To do this, we will analyze the respective data-dependent bounds under the expected gap $\mathrm{E}[\tilde{q}(\bar{n})]$.


\subsection{Data-Dependent Binary PATE Analysis}
We aim to specify $\mathrm{E}[\new{\varepsilon_{binary}}]$ for \binary PATE in terms of the probability of the most likely outcome, denoted $p$. By using the same $p$ for \powerset PATE, we can compare how the privacy budgets vary with respect to $p$.

We model
each teacher's independent label prediction as a Bernoulli trial of probability $p_i$, $i \in [k]$. With $t$ independent trials, we can model each gap as a binomial distribution of probability $p_i$. To simplify our exposition we assume $\forall i,p_i=p$ and will later discuss how this impacts our analysis. Using composition, we get an expected privacy loss of
\begin{align*}
   \mathrm{E}[\varepsilon_{binary}] &= k\cdot\sum_{\tilde{q}=\tilde{q}'}^{\tilde{q}=t}\left( P(\tilde{q}=\tilde{q}') \frac{4}{\tilde{q}}exp\left(\frac{-\tilde{q}}{2\cdot\sigma^2}\right) \right)\\
   &= k\cdot\sum_{\tilde{q}=\tilde{q}'}^{\tilde{q}=t}\left( {t\choose \tilde{q}}p^{\tilde{q}}(1-p)^{1-\tilde{q}} \frac{4}{\tilde{q}}exp\left(\frac{-\tilde{q}}{2\cdot\sigma^2}\right) \right)\\
\end{align*}

\vspace{-1em}
Due to the assumptions of~\Eqref{eq:simple-rdp}, we require that low gap events ($\tilde{q}' \to 0$) have sufficiently low probability so that they can be ignored. Figure~\ref{fig:cdf-gaps} of Appendix~\ref{app:additional-experiments} provides justification for this.

\subsection{Data-Dependent \powerset PATE}
We now estimate how $\mathrm{E}[\tilde{q}(\bar{n})]$ changes with changes in $p$ for \powerset PATE---\new{this directly impacts $\mathrm{E}[\varepsilon^{powerset}_{\mathcal{M}_\sigma}]$}. This is more involved because it operates over the entire binary vector. If each label has probability $p_i=p$ of being present, then we can express the probability (denoted $P$, note the upper case) of an outcome (binary vector subset) as the union of these probabilities. Though exact, this is infeasible to calculate and analyze for all $2^k$ subsets when $k$ is large. This is one reason we estimate only $\tilde{q}(\bar{n})$ rather than $q(\bar{n})$--- we need only estimate $n_1$ and $n_2$.

We define the \powerset  mechanism in Definition~\ref{def:powerset-pate} of Appendix~\ref{app:binary-pate}, which is similar to \powerset Voting (Definition~\ref{def:powerset-voting}) except that we use standard PATE notation. Here, our histogram $n$ counts subsets with $n_i(x)$ representing the vote count for the $i-$th subset.

We sort the votes $n_i(x)$ such that $n_i \geq n_j \forall i > j$.
We seek to upper bound our best-case expected privacy loss for \powerset PATE. Using \Eqref{eq:simple-rdp}, we get
\begin{equation*}
    \mathrm{E}[\varepsilon^{powerset}_{\mathcal{M}_\sigma}(\new{\tilde{q}})] = \int \new{\frac{4}{\tilde{q}}} exp(\frac{-\tilde{q}}{2\cdot\sigma^2}) \new{PDF(\tilde{q})} d\tilde{q},
\end{equation*}
 \new{where $PDF(\tilde{q})$ is the probability density function of $\tilde{q}$ and we again require that low gap events $\tilde{q}<1$ have no density.}
 
 $\mathrm{E}[\varepsilon^{powerset}_{\mathcal{M}_\sigma}(\lambda)]\rightarrow0$ as the density of $\tilde{q} \rightarrow \infty$ as expected. The base case gap with $O(1)$ probability can be found by viewing the binary vectors of \powerset PATE as a non-uniform balls and bins problem. We have $2^k$ subsets, i.e., $|\mathcal{P}(\{0, 1\}^k)|=2^k$. Each subset $i \in [2^k]$ is represented by a bin with probability $P_i$ of having a ball land in it (be voted on by a teacher). The $t$ teachers ($\Sigma_i \bar{n_i}=t$) each vote independently for a bin (subset). 

We estimate $\tilde{q}$ by calculating the load of the maximally loaded bin (most voted outcome), denoted as $\tau$, when \powerset PATE \new{has the highest gap, i.e., }where each other bin has at most $1$ ball. Then, $\tilde{q}=\tau-1$; by upper-bounding $\tau$, we upper bound the gap.

To upper-bound $\tau$, we use an indicator random variable $\mathcal{C}_c$ representing the event of $c$ collisions occurring in any bin, \ie $c$ votes for that bin. Using Markov's inequality followed by Stirlings approximation, we get (see also Appendix~\ref{sec:sterling-factorial}):
\begin{align*}
    & Pr[\mathcal{C}_c \geq 1] \leq \mathrm{E}[\mathcal{C}_c] = {t\choose c}\sum_i^{2^k}P_i^c \leq (\frac{t\cdot e}{c})^c\sum_i^{2^k}P_i^c.\\
\end{align*}
As with $\bar{n}$, we sort our probability $[P_1,\cdots,P_{2^k}]$ such that $P_1 \geq P_2 \cdots \geq P_{2^k}$. Because $n_1 \gg n_2$, we have that $\forall i\neq1, \ P_1 \gg P_i$. As the gap increases, $P_1 \to 1$, and $P_i^c \to 0$, we get that $\sum_i^{2^k} P_i^c\leq P_1^c$ and thus $Pr[\mathcal{C}_c \geq 1] \leq(\frac{t\cdot e}{c})^c P_1^c$. In particular, we care about the regime where the max load occurs with high-probability ($Pr[\mathcal{C}_c]\rightarrow 1$):

\begin{align*}
    Pr[\mathcal{C}_c \geq 1] \leq (\frac{t\cdot e\cdot P_1}{c})^c &\rightarrow 1\\
    \exp( \ln (\frac{t\cdot e\cdot P_1}{c})^c) &\rightarrow 1\\
    \implies c\cdot \ln(t\cdot P_1) + c - c\cdot \ln(c) \rightarrow 0 \\
    \implies c\cdot \ln(t\cdot P_1) + c = c\cdot \ln(c). \addtocounter{equation}{1}\tag{\theequation} \label{eq:powerset-pate-privacy-loss}
\end{align*}

Specifying $t$ and $P_1$, we can then directly calculate the max collisions $c$, then $\tilde{q}$, and finally the best-case expected privacy loss. In the next Section~\ref{app:binary-versus-powerset}, we analyze how the max load changes with respect to $P_1$.

\subsection{Privacy Analysis Comparison} \label{app:binary-versus-powerset}
From the assumptions above, we can directly calculate each $P_i$, for \powerset PATE. We do this by recognizing that a single teacher's binary vector can be modeled as a binomial distribution of $k$ (the number of labels) trials and probability $p$. 
To calculate the probability of any one of the $2^k$ possible binary vectors from the \powerset vector, we express each one as a $l-$hot binary vector ($l$ labels present, independent of their coordinate position). Each possible subset satisfying $l-$hot will have equal probability of occurring (because we fixed $p$ for all labels).
Then, the maximum probability binary vector, when $k$ sufficiently large ($k>5$, approximately) and $p\rightarrow0$ or $p\rightarrow1$ is when $l=0$ or $l=k$ as otherwise the combinatorics leads to a probability split amongst too many choices. This implication may appear unnatural at first; however, it well models our output label distribution: when there are a set of label coordinates that share similarly high (or low) probabilities of predicting presence (or absence) of that label, then their most likely binary vector is all $1$'s or $0$'s. 
We will next analyze the three major contributors to changes in the gap.
\chris{see amrita suggestion below, commented out}

\textbf{Impact of random labels.} To model more complex distributions where in general $p_i \neq p$, we can bucket the ranges of $p$ present and analyze these buckets separately. Let us operate under the assumption of coordinate-independence. In this case, we have the maximal variance for the binomial distribution and thus largest gaps for \powerset and \binary PATE. In this case, we expect both methods to perform poorly: however, because the data-independent bound for \powerset PATE is tighter than for \binary PATE, \powerset PATE performs better, as shown in 
Figure~\ref{fig:binary_powerset_negative_votes}.
This gives us our first two observations: (1) \emph{\powerset PATE has a better worst-case privacy loss (data-independent bound)} and (2) \emph{both \powerset and \binary PATE degrade to their data-independent bounds as $p\rightarrow 0.5$}.



\textbf{Impact of Correlation (Coordinate Dependence).} Even when a bucket of labels has $p\approx1$ or $p\approx0$ (but not equal), we can still have a varying distribution of outcomes. In particular, with coordinate-independence, we can directly use our binomial analysis above and find that there is a uniform probability for all outcomes in each $l-$hot vector. In particular, there are factorially many outcomes each with equal probability. This drastically degrades the gap by reducing $P_1$. However, when there is coordinate-dependence, e.g., the best-case when each label's value directly implies the rest, then we have only two outcomes for each $l-$hot vector. This improves the gap by reducing the possible subsets and improving $P_1$. We formalize this as follows: if we have some subset of $d$ vectors that are dependent, then we can reduce $k$ in our above analyses to $(k-d)$. This gives us our third observation: (3) \emph{coordinate dependence has a combinatoric improvement for \powerset PATE and at best a linear improvement for \binary PATE}. 

\textbf{Impact of Higher Noise Multipliers.} Appealing back to the data-dependent $q(\bar{n})$ calculation of Proposition~\ref{prop:gap}, we see that \binary PATE has a gap calculated as $2P-t/2\sigma$ for $t$ teachers. This gives $q^{binary}(\bar{n})=erfc(2P-t/2\sigma)$ However, for a similar gap, which is bounded for both to a $\max(\tilde{q}(\bar{n}))=t$, we see that \powerset PATE will have $q^{powerset}(\bar{n})=\sum_i^{2^k}erfc(\tilde{q}(\bar{n})/2\sigma)$, which is always worse when the $gaps$ are equal. We can see this effect in Figure~\ref{fig:binary_powerset_negative_votes}. Further, because we union bound across the $2^k$ classes, a  higher noise multiplier has a larger impact on \powerset PATE, as shown in Figure~\ref{fig:powerset-binary-pate-compare-metrics-sigma}. These give our fourth and fifth observations: (4) \emph{in the best case, and for any equal gaps, \powerset PATE has a higher privacy loss}, and (5) \emph{higher noise multipliers will impact \powerset PATE more for similar gaps}. 

Overall, we find that there are two regimes where \powerset voting (PATE) outperforms \binary voting (PATE). 
\emph{The first regime is when there are high correlations between the multi-winner candidates}, i.e., there is strong coordinate dependence. We found empirical evidence of this case, as we will describe later, in Figure~\ref{fig:powerset-binary-pate}.
\emph{The second regime is when neither mechanism can leverage stronger data-dependent privacy guarantees}, the data-independent guarantees of \powerset PATE are tighter (see Section~\ref{ssec:independent-analysis}). We find that though \powerset PATE may not always be capable of returning the plurality vote, the utility of the votes remains high (see Figure~\ref{fig:powerset-binary-pate-compare-metrics-sigma}).

\subsection{Choosing the PATE Mechanism}
As we will see from the empirical results in the next section, we recommend choosing $\tau$ PATE in all cases where a reasonable $\tau<\sqrt{k}$ can be chosen based on the average $\ell_2$ norm of the training data. Otherwise, \binary PATE is generally preferred, unless there is a significant correlation between labels detected (we show one example of this analysis in Appendix~\ref{sec5:limitations}).

\begin{figure*}[!t]
\begin{center}
\centering
\begin{tabular}{cc}
\includegraphics[width=0.34\textwidth]{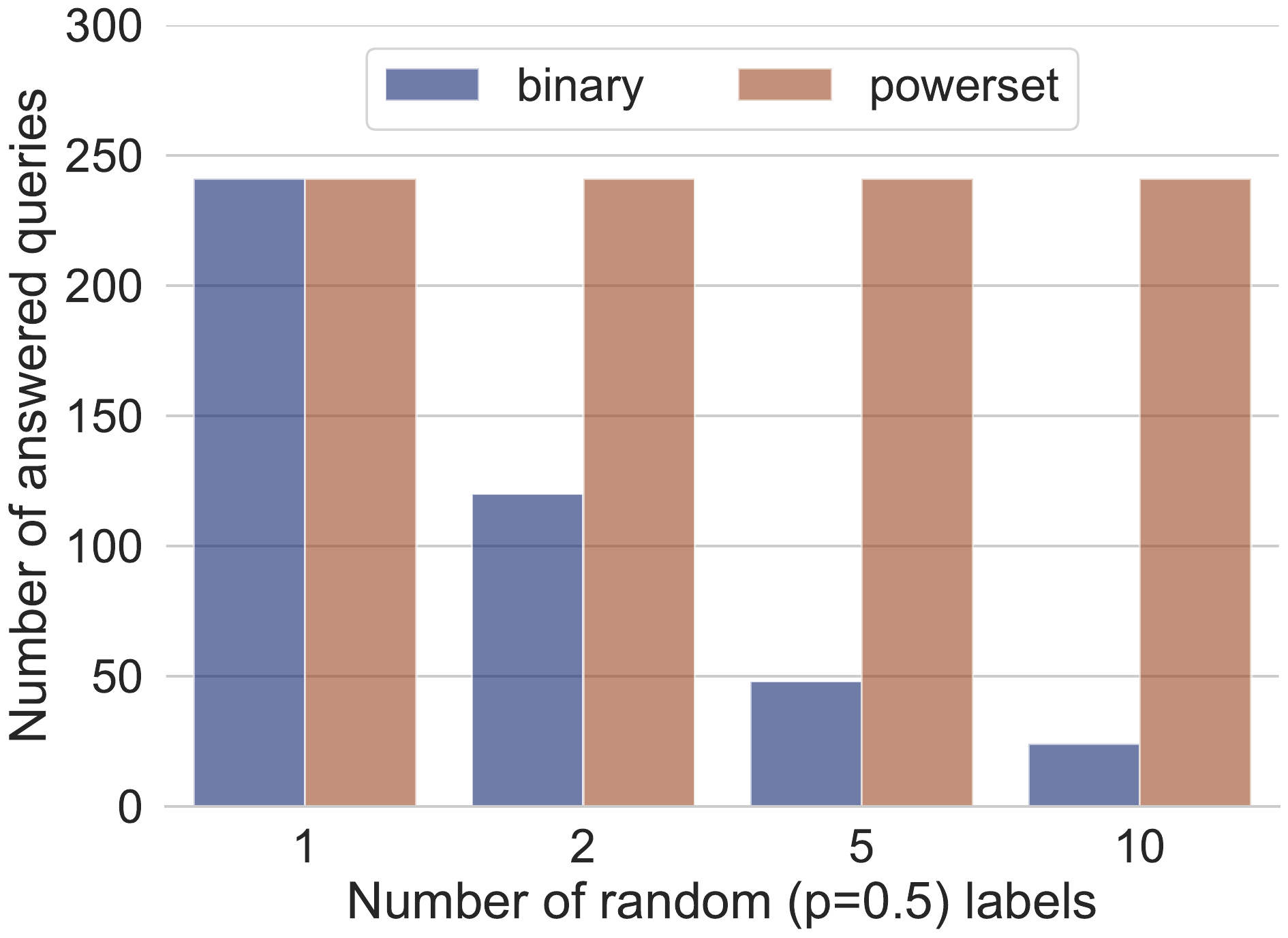}
&\includegraphics[width=0.32\textwidth,trim={1cm 0 0 0},clip]{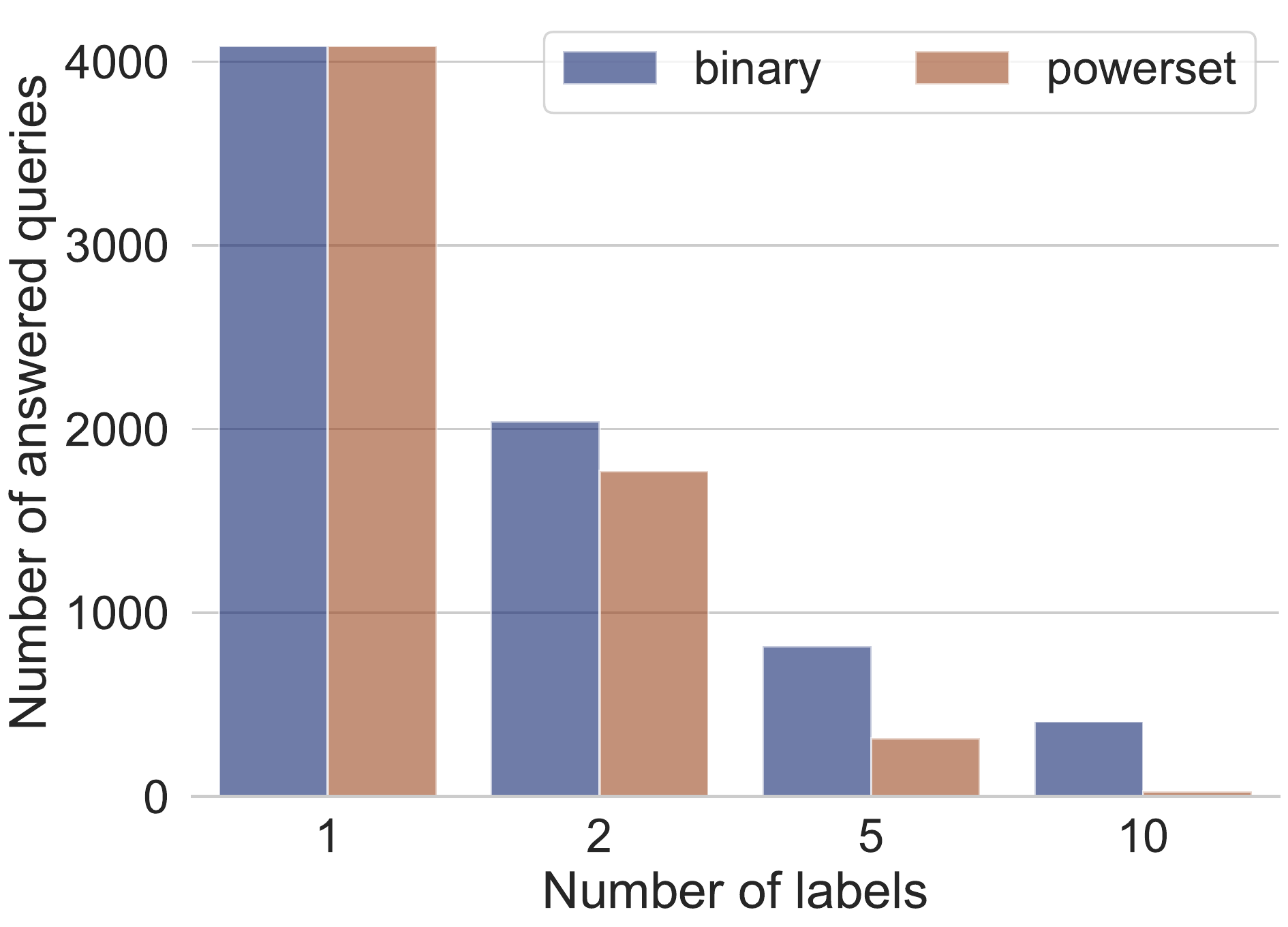}
\end{tabular}
\caption{
\textbf{LEFT: \powerset PATE outperforms \binary PATE as more labels are generated randomly, i.e., p=0.5 with $\varepsilon=20$.} If all votes are random, then both \binary and \powerset PATE fall back on the data independent bound. 
\textbf{RIGHT: \binary PATE outperforms \powerset PATE when they both have a similarly high gap.}. Here, all teachers always output $0$ for all labels. We use $\varepsilon=2.$
BOTH: we use $50$ teachers and privacy noise $\sigma_{GNMax}=7$.
}
\label{fig:binary_powerset_negative_votes}
\end{center}
\vskip -0.0in
\end{figure*}

\section{Empirical Evaluation}\label{sec:experiments}
\label{sec:Experiments4}

We compare our proposed mechanisms for multi-label classification on real-world datasets. Our goal is to answer the following questions: (1) How do the multi-label PATE methods work in practice and which of them should be selected by practitioners? (2) In the centralized learning setting, how does our proposed multi-label PATE compare with the other method of choice - DPSGD?  Finally, (3) how our private multi-label classification performs in the framework of a distributed collaborative learning system, such as CaPC, which preserves privacy and confidentiality? We observe that the $\tau$ PATE is the preferred multi-label method that outperforms other PATE-based methods as well as DPSGD, and can provide benefits when used by collaborating parties.

\subsection{Experimental Setup}
We carry out the evaluation on three medical datasets CheXpert~\cite{chexpert2019}, MIMIC-CXR~\cite{johnson2019mimiccxrjpg}, PadChest~\cite{padchest2020}, and on a vision dataset Pascal VOC 2012~\cite{everingham2010pascal}. 
We use DenseNet-121~\cite{densenet2017} for the medical datasets and ResNet-50~\cite{he2016deep} for the vision dataset. DenseNet has been shown to be the best architecture for X-ray data~\cite{xrayCrossDomain2020}. See Appendix~\ref{app:datasets} for dataset and model descriptions. 
\new{Our main exposition chooses safe $\delta<1/N$ where $N$ is the number of records in the dataset. Only one experiment uses $\delta=1e-4$ to compare directly with~\cite{AdaptiveDPSGD2021}.}
We experiment with
$(\varepsilon=8,\delta=1E-4)$ for 5 labels on CheXpert (see Appendix~\ref{sec:dpsgd-pate-chexpert}), $(\varepsilon=10,\delta=1E-5)$ for predictions on PascalVOC, $(\varepsilon=20,\delta=1E-6)$ for 11 labels on CheXpert or MIMIC, and also $(\varepsilon=20,\delta=1E-6)$ for 15 labels on PadChest (see also Section~\ref{app:diff-epsilons}). 
\new{As an example, we show how the number of answered queries decreases gradually with orders of magnitude lower $\delta$ values for the PascalVOC dataset in~\Cref{tab:pascal-voc-delta-values}.}
\new{We note that though $\tau$ should use DP protection~\cite{papernot2022hyperparameter}, it is commonplace to tune hyperparameters non-privately.}
Our mechanisms add a negligible run-time performance overhead on top of PATE, where the main bottleneck is the training of the teacher models.~\footnote{Our code can be found at this link:
\new{~\href{https://github.com/anonymous-user-commits/private-multi-winner-voting}{https://github.com/anonymous-user-commits/private-multi-winner-voting}}.
}

\textbf{Evaluation Metrics.} We use the following metrics: (1) accuracy (ACC) - a proportion of correct predictions per label, (2) balanced accuracy (BAC) - mean recall and specificity per class~\cite{brodersen2010balanced}, (3) area under the receiver operating characteristic curve (AUC)~\cite{hand2001simple}, and (4) mean average precision (MAP), where average precision per class is the area under the precision-recall curve.
We compute each metric per label using methods from \textit{sklearn.metrics} and then average them across all the labels. 
The definitions per label of the metrics are as follows: 
(1) accuracy per label (ACC): $\frac{TP+TN}{TP+TN+FP+FN}$, (2) Balanced Accuracy per label (BAC) = $\frac{1}{2}(\frac{TP}{TP+FN} + \frac{TN}{TN+FP})$, where $TP$ = True Positive; $FP$ = False Positive; $TN$ = True Negative; $FN$ = False Negative, (3) Area-Under-the-Curve (AUC) = $\int_{0}^{1} t(f) \  df$, where $t(f)$ is the $TP$ rate against the $FP$ rate, (4) mean Average Precision (MAP) = $\frac{1}{n} \sum_{i=1}^{n} P_{i}$, where $P_i$ is the precision for the $i$-th label equal to $\frac{TP}{TP + FP}$ for that label and $n$ is the number of labels per example. 
More details on the evaluation metrics can be found in Appendix~\ref{sec:evalution-metrics} \new{and further information on the experimental setup in Appendix~\ref{app:train-details}.}

\subsection{Queries Answered vs Label Count}
\label{sec:binary-vs-powerset-appendix}

We analyze how our methods scale with a different number of labels in a dataset.
We set the privacy budget $\varepsilon=20$ throughout the whole experiment. The scale of the Gaussian noise $\sigma_{GNMax}$ is adjusted per label count so that we can preserve high values of the metrics (ACC, BAC, AUC, MAP), with maximum performance drop by a few percentage points. 
We find the values of $\sigma_{GNMax}$ separately for the Binary and Powerset methods so that their performance metrics are the same for a given label count.
For example, in Figure~\ref{fig:powerset-binary-pate-compare-metrics-sigma} for 11 labels, we select $\sigma_{GNMax} = 11$ for Binary PATE and $\sigma_{GNMax} = 4$ for Powerset PATE, where the values of the metrics for both methods are ACC=.95, AUC=.66, and MAP = .37 on Pascal VOC. For the CheXpert dataset, we set $\sigma_{GNMax} = 20$ for Binary PATE and $\sigma_{GNMax} = 5$ for Powerset PATE, where the values of the metrics for both methods are ACC=.71, AUC=.68, and MAP = .45. Here, we observe that the private multi-label classification with \binary PATE performs better and preserves higher values of the metrics with more noise added, compared to the \powerset PATE method.

Next, we compare the \binary and \powerset methods across many labels and in terms of how many queries can be answered privately for a given number of labels in the dataset. We run the experiment on the Pascal VOC and CheXpert datasets and present results in~Figure~\ref{fig:powerset-binary-pate}
(as well as in Tables~\ref{tab:powerset-binary-pate-pascal}-\ref{tab:powerset-binary-pate-cxpert} in the Appendix). Our analysis shows that the private multi-label classification with \binary PATE performs better (answers more queries) than \powerset PATE for Pascal VOC. The CheXpert dataset is more noisy in terms of labeling, the pathologies are sparse and difficult to detect; this dataset does not give us a clear preference of the multi-label PATE method in terms of the number of answered queries. 

\begin{figure*}[]
\centering
\begin{tabular}{cc}
\includegraphics[width=\columnwidth,scale=0.9]{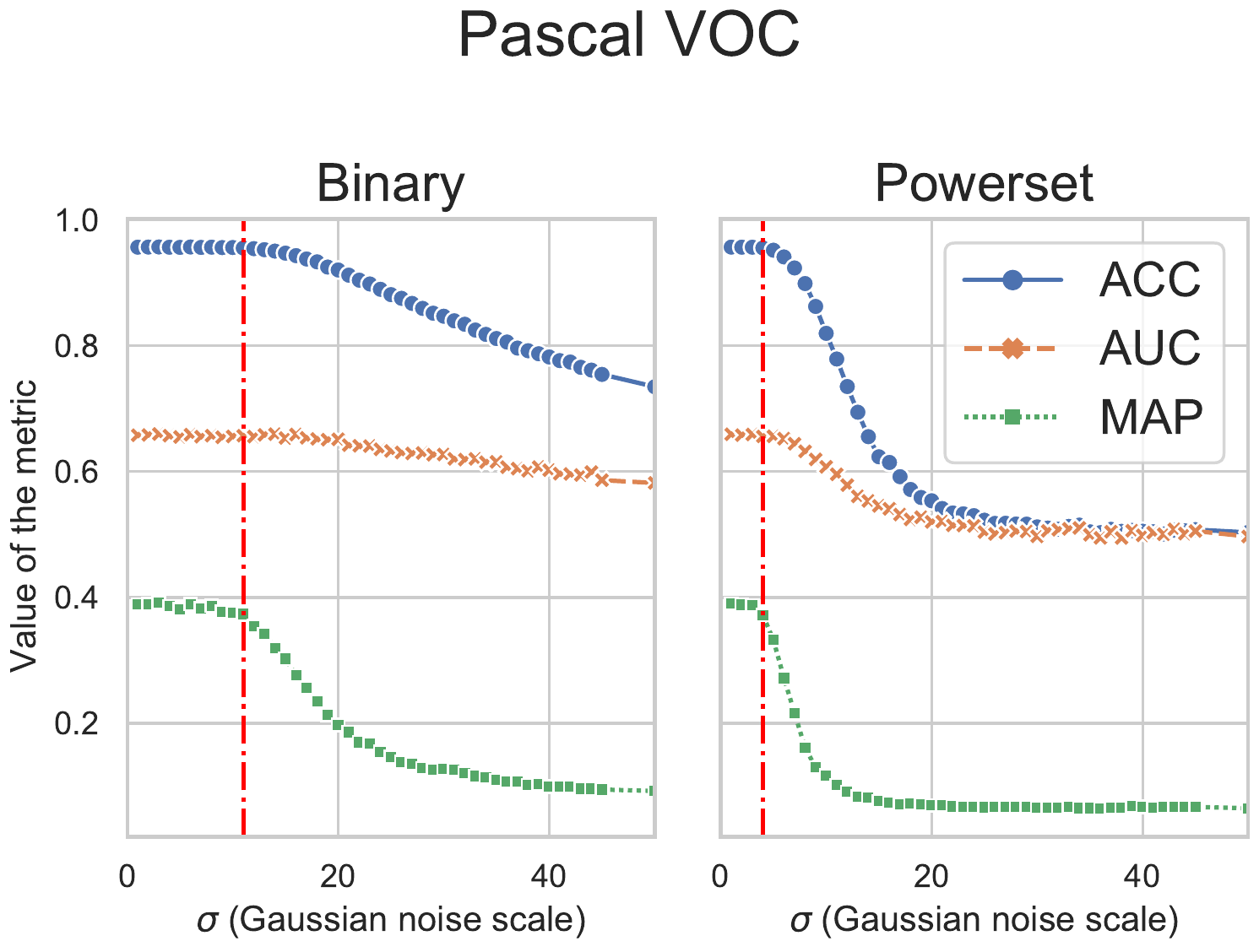} &
\includegraphics[width=1.0\columnwidth,scale=0.9]{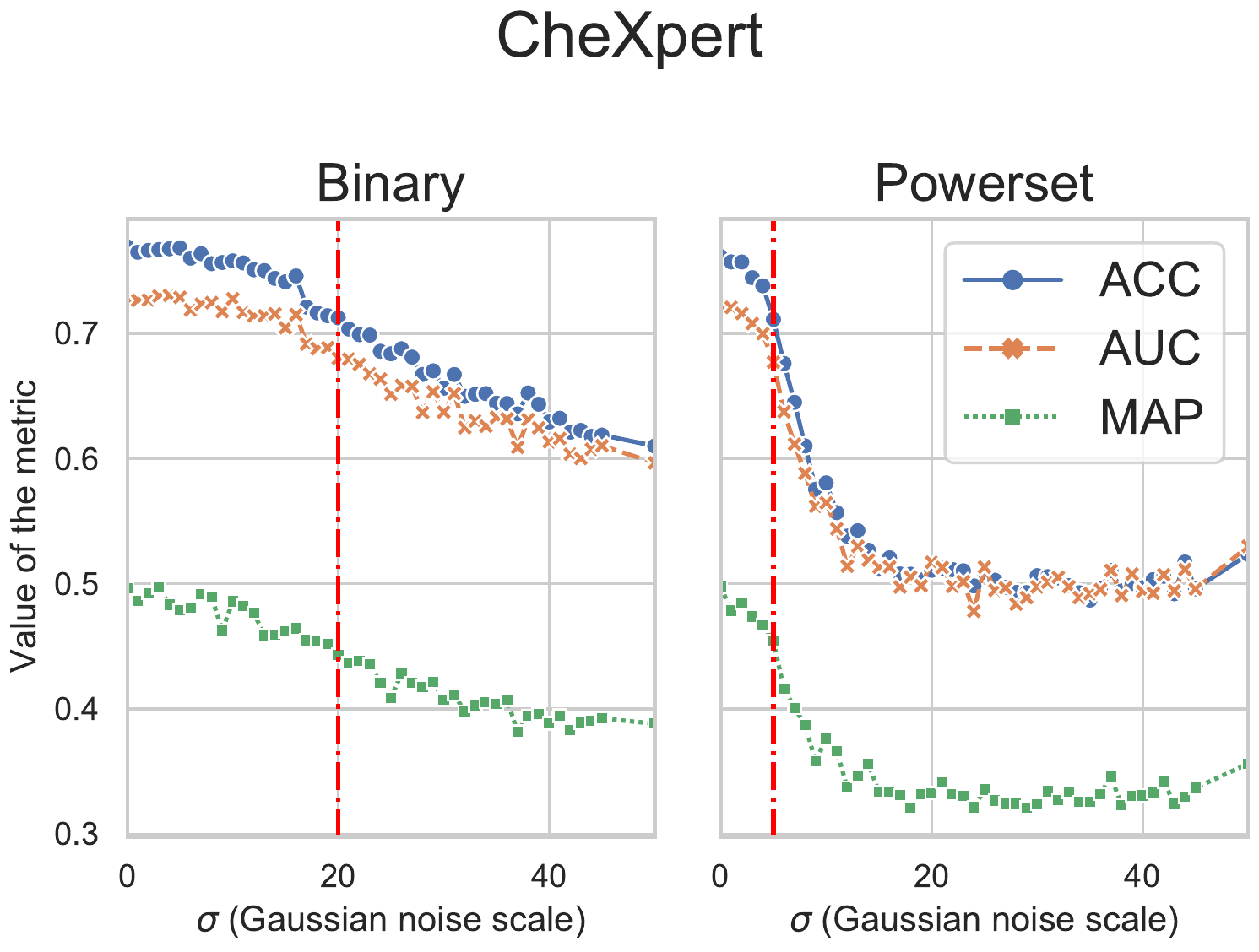} \\
\end{tabular}
\caption{\textbf{Binary vs Powerset PATE performance}. We compare the utility of the private votes returned in terms of ACC, AUC, and MAP, with respect to Gaussian noise scale $\sigma_{GNMax}$. We select the first $11$ labels in Pascal VOC and all 11 labels from CheXpert. We find that \binary PATE can tolerate a much higher noise scale $\sigma_{GNMax}$ compared with \powerset PATE before significant degradation in the votes is observed.}
\label{fig:powerset-binary-pate-compare-metrics-sigma}
\end{figure*}

\begin{figure*}[]
\centering
\begin{tabular}{cc}
\includegraphics[width=0.45\textwidth]{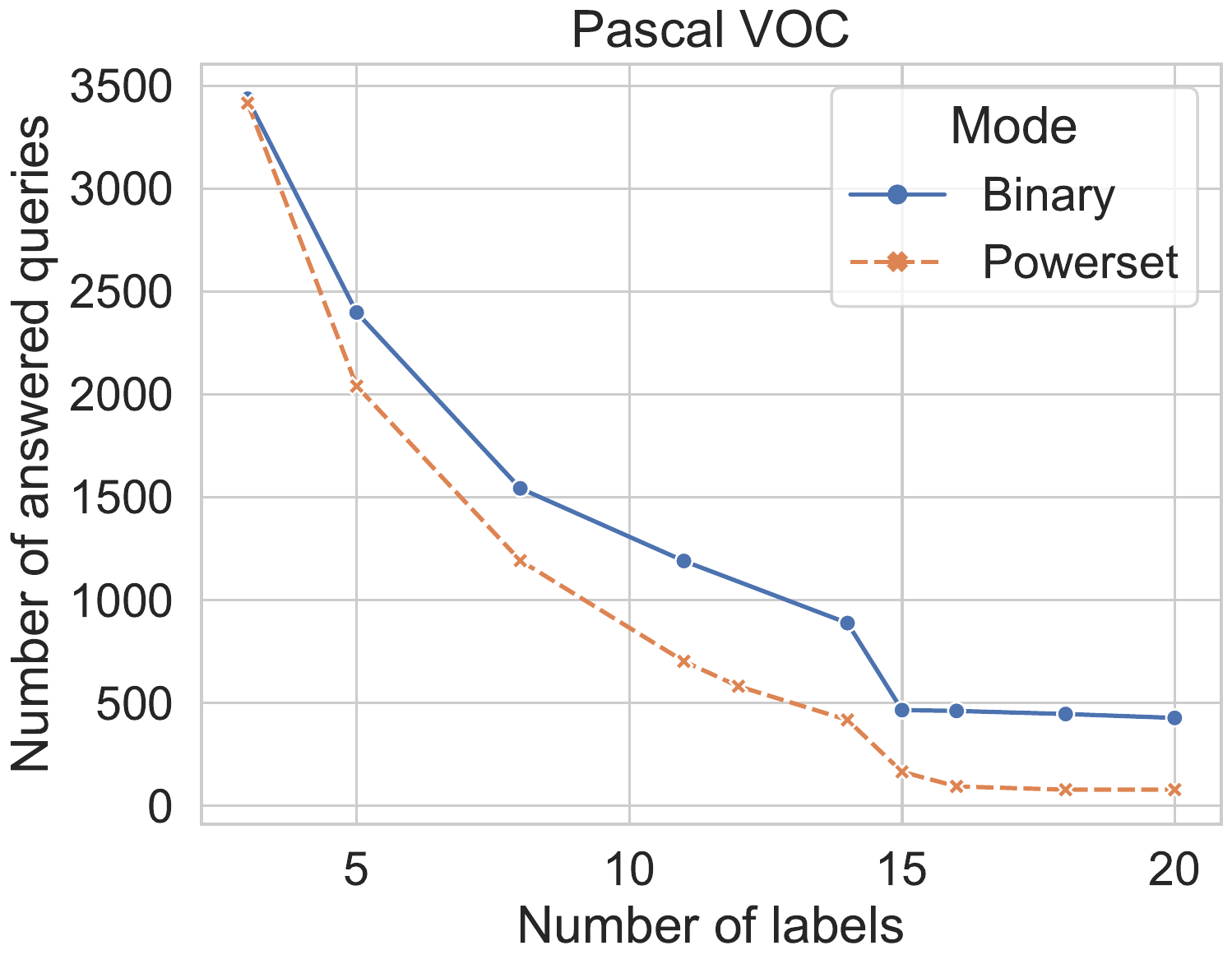} & 
\includegraphics[width=0.43\textwidth, trim={0.8cm 0 0 0},clip]{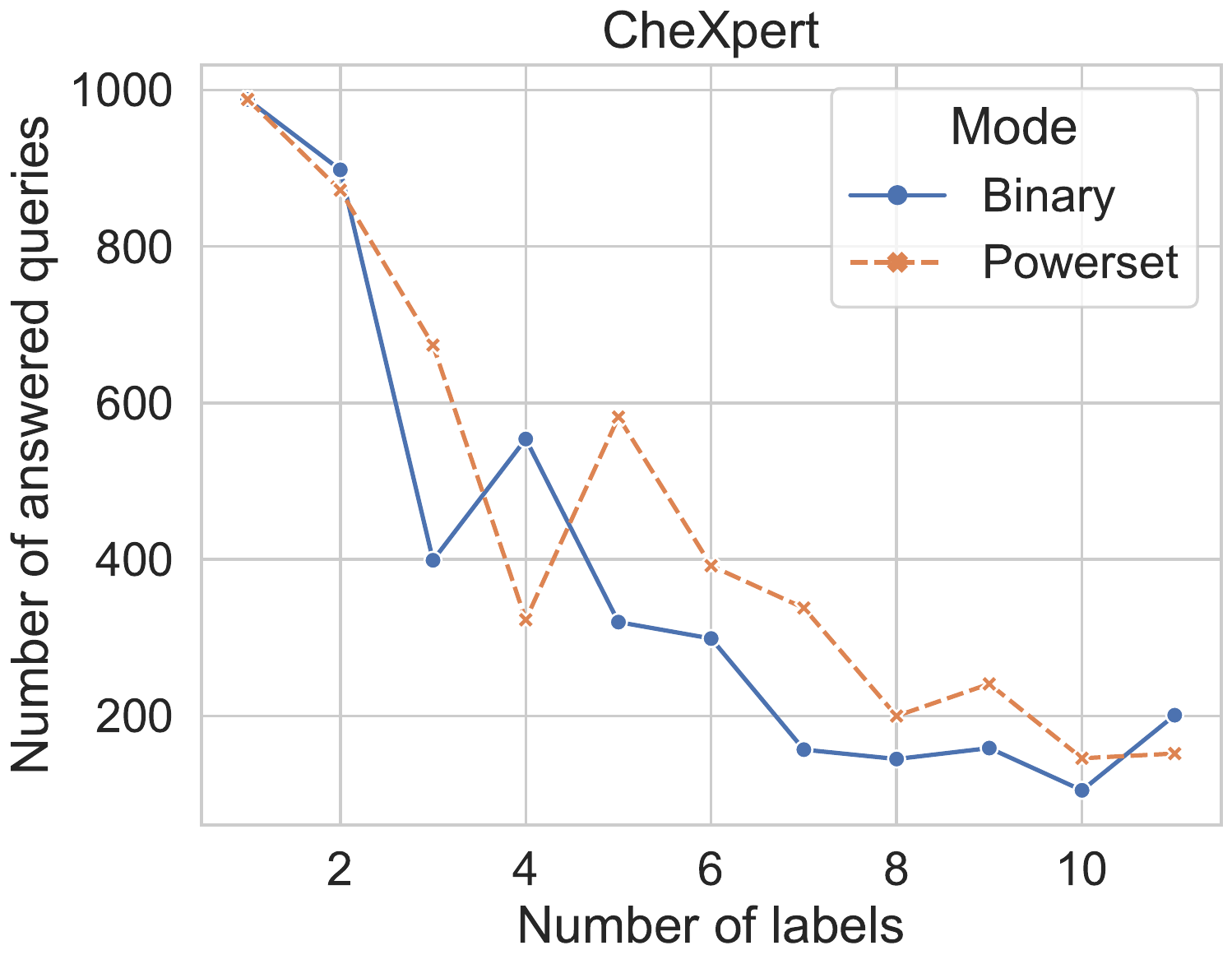}
\\
\end{tabular}
\caption{\textbf{Binary vs Powerset PATE: number of answered queries}. We compare the number of answered queries vs the number of $k$ first labels selected from the Pascal VOC and CheXpert datasets. We keep the privacy budget $\varepsilon=20$. 
}
\label{fig:powerset-binary-pate}
\end{figure*}

\subsection{Query-Utility Tradeoff in PATE}

\begin{figure*}[]
\begin{center}
\centerline{\includegraphics[width=1.0\linewidth]{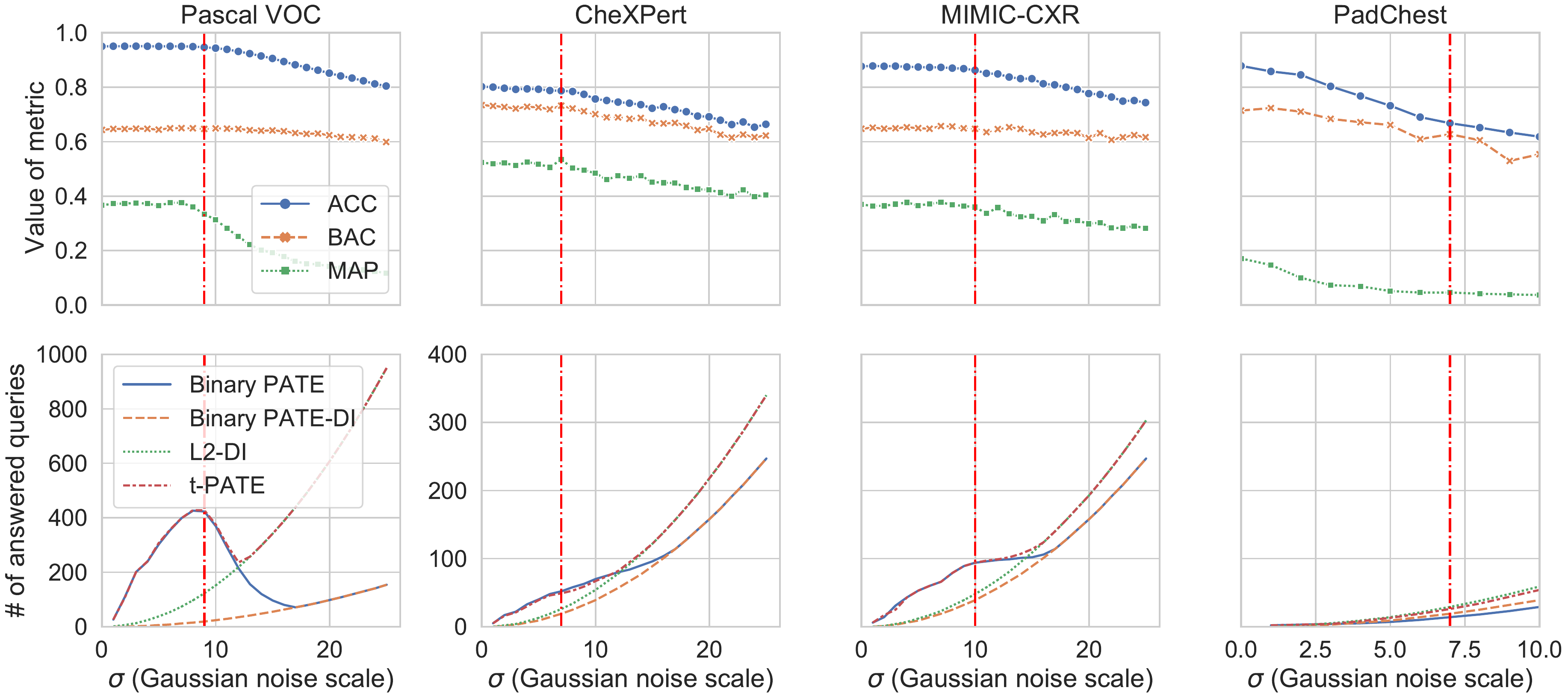}}
\vspace{-10pt}
\caption{\textbf{With sufficient consensus, the best query-utility tradeoff obtained lies in a regime of $\sigma_G$ where the data-dependent bound is used.} In the 1st row, we maximize the $\sigma_G$ of $\tau$ PATE while maintaining sufficiently high values of the performance metrics. The chosen values for $(\sigma_G, \tau)$ are $9,1.8$ for Pascal VOC, $10, 3$ for MIMIC-CXR, and $7,2.8$ for CheXPert and $7,2.7$ PadChest. When there is a lack of consensus (on PadChest), we see that the data-independent $\ell_2$ (L2-DI)  mechanism bound outperforms all others. For a well chosen $\tau$, there is little-to-no impact on the consensus of the data-dependent regime (c.f. \binary PATE and $\tau$ PATE which leverage the data-dependent bound when it reduces privacy loss). Because of this, $\tau$ PATE achieves a competitive query-utility tradeoff. See Figure~\ref{fig:app-tune-sigma} in Appendix~\ref{app:tuning-tau} for tuning of \binary PATE as an $\infty$-Winner Election.
}
\label{fig:tpate}
\end{center}
\vskip -0.2in
\end{figure*}

In comparing how the privacy parameter $\sigma_G$ impacts the query-utility tradeoff of each mechanism, under a fixed $\varepsilon=20$ across all datasets, we find that \emph{under high consensus \binary voting performs best and under lower consensus $\tau$ voting performs best}. In Figure~\ref{fig:tpate}, we compare each multi-winner election mechanism used in $\tau$ PATE and find that each mechanism has a range of $\sigma_G$ where it performs best. In particular, as $\sigma_G \rightarrow 0$, no queries can be answered by any mechanism because doing so would require more than the allotted privacy budget $\varepsilon$. For sufficiently small $\sigma_G$ and with suitable consensus amongst voters, which is the case for Pascal VOC, CheXPert, and MIMIC-CXR, we find the data-dependent analysis for \binary voting outperforms all others while remaining in a regime of high-performance metrics. As $\sigma_G \rightarrow \infty$, the $\tau$ voting mechanism outperforms all others, though in most cases at a decrease in the performance metrics; however, on PadChest, this is the best bound. This can be explained by the fact that on PadChest we can only train $10$ teachers before each individual model's accuracy degrades too much. An important distinction is that $\tau$ PATE can perform worse than \binary PATE when the chosen $\tau$ bound is too small (because ballot clipping can change the vote distribution). For well-chosen values of $\tau$ we observe only marginal decrease in the number of queries answered and the performance metrics (c.f. Figure~\ref{fig:tpate} with Figures~\ref{fig:app-tune-tau2},~\ref{fig:app-tune-tau1}, and~\ref{fig:app-tune-sigma} in Appendix~\ref{app:additional-experiments}). Inspecting the tradeoff of \powerset PATE (see Figure~\ref{fig:powerset-binary-pate-compare-metrics-sigma}), we find that a much lower noise can be tolerated before a steep decline in performance metrics. Because of this, we find that much fewer queries can be answered (see Tables~\ref{tab:powerset-binary-pate-pascal} and~\ref{tab:powerset-binary-pate-cxpert}). Thus, we recommend $\tau$ PATE as the de-facto mechanism.

\subsection{Private Centralized Learning}

We now show that even in the centralized setting, our \multi PATE methods outperform the competitive baselines. We train models on Pascal VOC and CheXpert which include $20$ and $5$ labels for this experiment, respectively. We leverage the entire training set to train a single non-private model and a single private model via DPSGD~\cite{abadi2016deep}. 
For our \multi PATE we instead train 50 teachers each on a separate disjoint partition of the centralized training set (and thus, with $1/50$ number of samples compared with DPSGD and the non-private model). We use these teacher models to privately train a student model using semi-supervised learning with MixMatch~\cite{MixMatch} (where modifications are made to adapt MixMatch to the \multi setting). 
Though the DPSGD algorithm does not assume a public unlabeled dataset, we compare it with our methods because DPSGD is the only baseline for multi-label prediction.
DPSGD for multi-label was studied in prior work~\cite{AdaptiveDPSGD2021}, and cannot directly leverage public unlabeled data. 
On the other hand, our PATE-based approaches leverage a public pool of unlabeled samples that have noisy labels provided by the ensemble of teachers. The added noise protects the privacy of the centralized training data. We then train the student on the newly labeled samples. Our \binary PATE does not use $\tau$ clipping or the confident GNMax improvement so as to fairly compare with \powerset PATE. \binary PATE has the benefit of using non-private learning on public data whereas DPSGD must add noise in training which impedes model learning.
Finally, DPSGD incurs a high computational cost (which \multi PATE does not) due to the expensive per-example gradient computations~\cite{subramani2020enabling}.

\textbf{Pascal VOC.} The student model, non-private baseline, and DPSGD baseline were all pre-trained on ResNet-50 models.
Observing Table~\ref{tab:compare-dpsgd-pate}, we see that our \binary PATE algorithm outperforms all other privacy-preserving techniques by a significant margin. 
The non-private model achieves strong performance across all metrics where the model trained using DPSGD incurs significant degradation across all metrics. 
\begin{table}
  \caption{\textbf{DPSGD vs PATE on Pascal VOC} for all 20 labels. Comparison between standard non-private model, DPSGD, Powerset and Binary multi-label in terms of utility using metrics: Accuracy (ACC), Balanced Accuracy (BAC), Area-Under-the-Curve (AUC), and Mean Average Precision (MAP).}
  \label{tab:compare-dpsgd-pate}
  \begin{sc}
  \begin{center}
  \vspace{-0.5pt}
  \scriptsize
  \begin{tabular}{ccccc}
    \toprule
    \textbf{Method} & \textbf{ACC} & \textbf{BAC} & \textbf{AUC} & \textbf{MAP} \\
    \hline
    \textit{Non-private}     & .97 & .85 & .97 & .85 \\
    \cdashlinelr{1-5}
    \textbf{DPSGD}           & .92 & .50 & .68 & .40 \\
    \textbf{Powerset PATE}   & .94 & .58 & .70 & .29 \\ 
    \textbf{Binary PATE}     & .94 & .62 & .85 & .57 \\
    \bottomrule
  \end{tabular}
  \end{center}
  \end{sc}
\end{table}

Though \binary PATE outperforms DPSGD and \powerset PATE, it falls short of the non-private model by a wide margin, indicating much room for improvement in \multi privacy-preserving techniques---in particular, in extreme \multi settings, which we motivate and expand on in  Appendix~\ref{sec5:limitations}. We observe that \powerset PATE answers much fewer queries (leading to less training data for the student model) than \binary PATE, at only $78$ compared to $427$. Though the student model only trains on $427$ samples compared to $5717$ for DPSGD.  

\textbf{CheXpert.}
\begin{table}[t]
\caption{\textbf{DPSGD vs PATE on Chexpert} for the first 5 labels. We compute the Area-Under-the-Curve (AUC) metric per label. We denote Adaptive DPSGD as \textbf{Adaptive}.}
\label{tab:dpsgd-pate-chexpert}
\vskip -0.3in
\begin{center}
\begin{sc}
\scriptsize
\begin{tabular}{ccccccc}
\toprule
Method & AT & CA & CO & ED & EF & Average \\
\hline
\textit{Non-private}             & \textit{0.84}                             & \textit{0.80}                              & \textit{0.87}                               & \textit{0.90}                       & \textit{0.91}                                  & \textit{0.87}                         \\
\cdashlinelr{1-7}
\textbf{DPSGD}                   & 0.56                                      & 0.53                                       & 0.66                                        & 0.56                                & 0.62                                           & 0.58                                  \\
\textbf{Adaptive}          & 0.75                                      & 0.73                                       & \textbf{0.84}                               & \textbf{0.79}                       & 0.79                                           & 0.78                                  \\
\textbf{Binary PATE}    & \textbf{0.78}                             & \textbf{0.75}                              & \textbf{0.84}                               & 0.76                                & \textbf{0.81}                                  & \textbf{0.79}                         \\
\bottomrule
\end{tabular}
\end{sc}
\end{center}
\vskip -0.0in
\end{table}
We follow the experimental setup provided for Adaptive DPSGD from~\cite{AdaptiveDPSGD2021}. We use the DenseNet-121 model pre-trained on ImageNet and fine-tune only the last fully connected layer while keeping all the other (convolutional) layers fixed. Across all of the experiments, we use $\varepsilon=8$ as the privacy budget. We use the whole CheXpert test set, and here report results for $\delta=1E-4$. 
We set $\delta=1E-4$ only in this case to compare fairly with prior work~\cite{AdaptiveDPSGD2021}. In all other cases, we use the standard $\delta=1/(\mathrm{number\ of\ records})$, which is $\delta=1E-6$.
To evaluate the performance of the models, we use the CheXpert test set (from the valid.csv file in CheXpert-v1.0-small). As shown in the Table~\ref{tab:dpsgd-pate-chexpert}, similarly to the results for Pascal VOC, we also observe a gap in performance (about 8\%) between the non-private model and Binary PATE. For the direct comparison between private methods, in this setting, the Binary PATE also outperforms the state-of-the-art Adaptive DPSGD on four out of five labels, and in the average AUC across all the first 5 labels.

\subsection{Varying Privacy Budget $\epsilon$}
\label{app:diff-epsilons}

\begin{table}[h]
\caption{\textbf{Pascal VOC with 20 labels: Performance of \binary PATE for different values of the privacy budget $\varepsilon$} w.r.t. number of answered queries, ACC, BAC, AUC, and MAP as measured on the test set with the specified $\sigma_{\text{GNMax}}$. PB ($\varepsilon$) is the privacy budget. We use 50 teacher models. We set $\sigma_{\text{GNMax}} = 7$.}
\label{tab:binary-pate-vary-epsilon}
\vskip -0.0in
\begin{center}
\begin{sc}
\scriptsize
\begin{tabular}{ccccccc}\toprule
PB ($\varepsilon$) & \shortstack{Queries \\ answered} & ACC & BAC & AUC & MAP \\
\hline
1 & \textbf{0} & - & - & - & - \\
2 & \textbf{6} & .86 & .62 & .62 & .44 \\
3 & \textbf{13} & .93 & .67 & .67 & .53 \\
4 & \textbf{22} & .93 & .64 & .64 & .44 \\
5 & \textbf{31} & .95 & .63 & .63 & .39 \\
6 & \textbf{40} & .95 & .67 & .67 & .45 \\
7 & \textbf{64} & .95 & .64 & .64 & .35 \\
8 & \textbf{81} & .95 & .66 & .66 & .40 \\
9 & \textbf{101} & .95 & .60 & .60 & .28 \\
10 & \textbf{113} & .96 & .63 & .63 & .30 \\
11 & \textbf{135} & .96 & .64 & .64 & .33 \\
12 & \textbf{165} & .96 & .65 & .65 & .35 \\
13 & \textbf{199} & .96 & .63 & .63 & .32 \\
14 & \textbf{217} & .96 & .64 & .64 & .35 \\
15 & \textbf{239} & .96 & .63 & .63 & .32 \\
16 & \textbf{272} & .96 & .63 & .63 & .31 \\
17 & \textbf{306} & .96 & .63 & .63 & .30 \\
18 & \textbf{332} & .96 & .63 & .63 & .31 \\
19 & \textbf{362} & .96 & .63 & .63 & .30 \\
20 & \textbf{403} & .96 & .63 & .63 & .30 \\
\bottomrule
\end{tabular}
\end{sc}
\end{center}
\vskip -0.1in
\end{table}

Because multi-label classification uses $k>1$ labels, our work significantly outperforms the naive expectation of using $k$ times higher privacy budget from single-label $(k=1)$ classification. We note that since we deal with multi-label scenarios and complicated vision and medical datasets, rather than the typical single-label classification tasks, it is more difficult to attain tighter DP guarantee.
Our choice of $\varepsilon<=20$ falls within the range that is generally considered in prior work~\cite{papernot2018scalable,capc2021iclr} and has been found to be robust to privacy attacks~\cite{choquette-choo21a-label-only,nasr2021adversary}. \new{We explore the privacy utility trade-off.} The value of $\varepsilon$ can be decreased with fewer queries answered by teachers in PATE, which is shown in Table~\ref{tab:binary-pate-vary-epsilon}, where we vary the privacy budget from a tight guarantee ($\varepsilon=1$) to looser guarantees (up to $\varepsilon=20$).

\subsection{Multi-Label \capc}
\label{sec:retrain}

\begin{figure*}[t]
\begin{center}
\centerline{\includegraphics[width=1.0\linewidth,trim={0cm 0 0 0},clip]{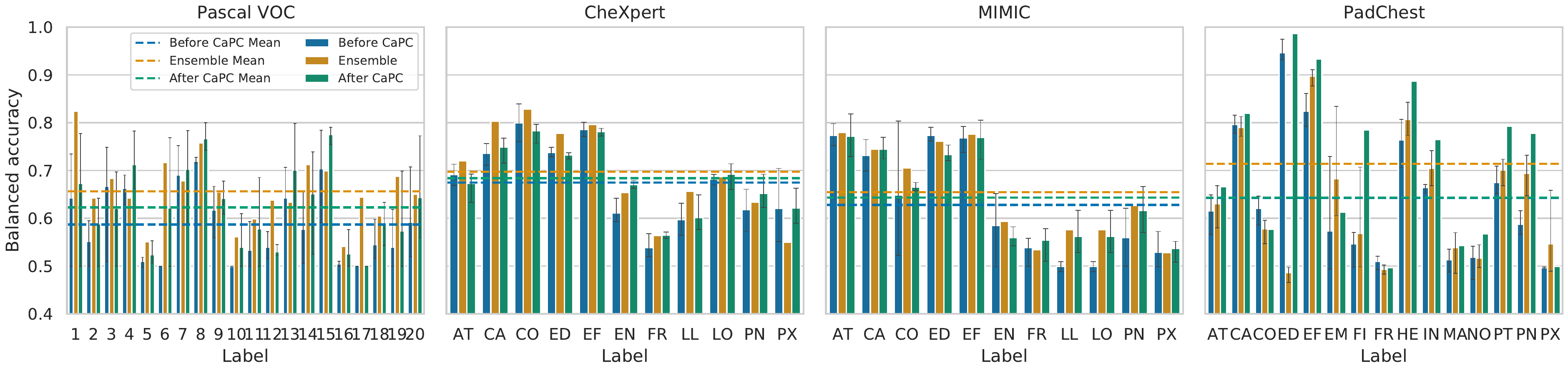}}
\vspace{-0pt}
\caption{\textbf{Using \capc to improve model performance.} \emph{Dashed lines represent mean balanced accuracy (BAC).}\label{fig:retrain} We retrain a given model using additional data labelled by all other models from the same ensemble.
We observe a mean increase of BAC by 2.0 percentage points on CheXpert. 
}
\end{center}
\vskip -0.02in
\end{figure*}

\begin{table*}[t]
\caption{\textbf{Model improvements through retraining with \multi CaPC.}}
\label{tab:data-stats-final}
\vskip -0.02in
\begin{center}
\begin{sc}
\small
\begin{tabular}{cclccccc}\toprule Dataset & \shortstack{\# of \\ Models} &  State & PB ($\varepsilon$) &          ACC &                  BAC &          AUC &          mAP \\
\midrule
\multirow{4}{*}{Pascal VOC} & 1 &  Initial &                 - & .97 &  .85 & .97 & .85 \\
 & 50 &  Before CaPC &  - & .93$\pm$.02 &          .59$\pm$.01 &  .88$\pm$.01 &  .54$\pm$.01 \\
& 50 &   After CaPC &                10 &  \textbf{.94$\pm$.01} &          .\textbf{62$\pm$.01} &  .88$\pm$.01 &  .54$\pm$.01 \\
& 50 &   After CaPC &                20 &  \textbf{.94$\pm$.01} &          \textbf{.64$\pm$.01} & \textbf{.89$\pm$.01} &  \textbf{.55$\pm$.01} \\
\hline
\multirow{3}{*}{CheXpert} & 1 &      Initial &                 - &          .79 &                  .78 &          .86 &          .72 \\   
& 50 &  Before CaPC &                 - &  .77$\pm$.06 &          .66$\pm$.02 &  .75$\pm$.02 &  .58$\pm$.02 \\   
& 50 &   After CaPC &                20 &  .76$\pm$.07 &  \textbf{.69$\pm$.01} &  \textbf{.77$\pm$.01} &  \textbf{.59$\pm$.01} \\
\hline
\multirow{3}{*}{MIMIC} & 1 &      Initial &                 - &         .90 &                  .74 &          .84 &          .51 \\   
& 50 &  Before CaPC &                 - &  .84$\pm$.07 &          .63$\pm$.03 &  .78$\pm$.03 &  .43$\pm$.02 \\   
& 50 &   After CaPC &                20 &  \textbf{.85$\pm$.05} &          \textbf{.64$\pm$.01} &  \textbf{.79$\pm$.01} &  \textbf{.45$\pm$.03 }\\
\hline
\multirow{3}{*}{PadChest} & 1 &      Initial &                 - &          .86 &                  .79 &          .90 &          .37 \\   
& 10 &  Before CaPC &                 - &  .90$\pm$.01 &          .64$\pm$.01 &  .79$\pm$.01 &  .16$\pm$.01 \\   
& 10 &   After CaPC &                20 &  .88$\pm$.01 &          .64$\pm$.01 &  .75$\pm$.01 &  .14$\pm$.01 \\
\bottomrule
\end{tabular}
\end{sc}
\end{center}
\vskip -0.2in
\end{table*}

By replacing the \single PATE in~\cite{capc2021iclr} with our \multi PATE mechanisms, we enable \multi learning in a multi-site setting: the framework of \capc allows for distributed collaboration across models located at different sites. We scale the evaluation of \multi \capc learning to real-world datasets and models by providing a decentralized (independent) evaluation of each answering party,
enabling large models. Our \multi \capc experiments replicate the setup of~\cite{capc2021iclr} using the source code provided.

We train 1 model per participant on separate distinct portions of the training set and then use \multi \capc to improve the performance of $3$ participants. Details are in Appendix~\ref{app:train-details}. Observing Table~\ref{tab:data-stats-final}, we see that \multi \capc consistently improves the BAC, with the greatest improvement of a considerable $5$ percentage point increase on Pascal VOC and improvement across all other metrics. These improvements are echoed in the larger and more privacy-sensitive CheXpert and MIMIC-CXR datasets; however, we observe some performance degradation on PadChest, likely because of the degraded vote utility compared to the original training data. 
As one of the main applications for \capc is the healthcare domain, the on-average performance improvements and associated privacy guarantees demonstrate the utility of \multi \capc in a realistic use case. 

Inspecting Figure~\ref{fig:retrain}, we see that \multi \capc leads to significant improvements in low-sample, low-performance labels (e.g., labels EN and FR on CheXpert). Thus, poorer performing models and classes can gain from the noisy aggregation of more performant teacher models through \multi \capc. This has potential ramifications for fairness because our experiments consistently demonstrate that private \multi \capc can improve model performance on its poorer performing subpopulations---where~\cite{suriyakumar2021chasing} show that differentially private can hurt performance on these subpopulations.
We reiterate that these subpopulations are common in many settings such as healthcare due to imbalanced labels, e.g., rarer diseases (see Figure~\ref{fig:datasets_class_distribution} of Appendix~\ref{app:datasets}).
\vspace{1em}
\section{Conclusions}
Single-label classification requires us to return one categorical output. Instead, multi-label classification is more complicated since we return $k>1$ class-labels for each input. Naively, one could apply single-label classification repeatedly for each of the $k$ labels. However, there are correlations between these $k$ labels. We thus show that this is suboptimal. For example, there are fewer than $\tau=3$ positive labels per query in our medical datasets. Intuitively, a teacher in PATE should only vote for up to $\tau$ positive labels in this setting.

We address the need for privacy in the \multi setting with three new \multi voting mechanisms. We show and prove that, while simple, our Binary voting cannot be outperformed without strong candidate correlations. When these correlations exist, we prove new data-independent bounds for our $\tau$ voting mechanism and theoretically analyze when \powerset voting performs better. 
We also compare all possible norms used for $\tau$-clipping and analytically demonstrate that the $\ell_2$ norm is optimal. 
These fundamental insights allow us to scale privacy-preserving ML to multi-label tasks.
Using these three new mechanisms, we create \multi PATE which outperforms DPSGD in the centralized private learning setting. We further enable \multi learning in multi-site scenarios by creating \multi \capc. We test our methods on standard vision and medical datasets. Our results show new state-of-the-art for private learning in \multi settings and demonstrate a need for further exploration to lessen the gap between private methods and non-private baselines.

\begin{acks}
We would like to acknowledge our sponsors, who support our research with financial and in-kind contributions: CIFAR through the Canada CIFAR AI Chair program, DARPA through the GARD program, Intel, Meta, NFRF through an Exploration grant, and NSERC through the Discovery Grant and COHESA Strategic Alliance. Resources used in preparing this research were provided, in part, by the Province of Ontario, the Government of Canada through CIFAR, and companies sponsoring the Vector Institute. We would like to thank members of the CleverHans Lab for their feedback.
\end{acks}

\newpage
\bibliographystyle{plain}
\bibliography{main}


\appendix
\section{Theoretical Mechanism Analysis}\label{app:binary-pate}


\begin{lemma} \label{lemma:coordindep}
If $f(X):\new{\mathcal{R}^d} \rightarrow \mathcal{R}^k$ is coordinate-independent, there exists a pair of databases $(X,X')$ with $||X-X'||_1 = 1$ (a change in only one row) which achieves the worst-case sensitivity for each of the coordinates $f_i$.
\end{lemma}

\begin{proof}


The proof follows by extension of the case where \new{$d$} = $k$ and $f_{i*}(X)=f_{i*}(X_{i*})$, i.e., each output coordinate is determined by a unique input coordinate. Thus, we can independently maximize the sensitivity for each $f_{i*}$ and horizontally stack these to obtain a single $X$ and $X'$ that maximizes the sensitivity of $f$.

\begin{align*}
    \Delta_1 f_{i*} &= \underset{||X-X'||_1 = 1}{\underset{(X,X')}{\max}} ||f_{i*}(X) - f_{i*}(X')||_1\\
    &= ||f_{i*} ([X_{i | i \in P_{i*}}, X_{i | i \notin P_{i*}}]) - f_{i*} ([X'_{i | i \in P_{i*}}, X'_{i | i \notin P_{i*}}])||_1\\
    &=  ||f_{i*} ([X_{i | i \in P_{i*}}, X_{i | i \notin P_{i*}}]) - f_{i*} ([X'_{i | i \in P_{i*}}, X_{i | i \notin P_{i*}}]) ||_1
\end{align*}

Therefore we can choose the members in $P_{i*}(X)$ arbitrarily and independently to achieve the worst case sensitivity for $f_{i*}$ while the members not in $P_{i*}(X)$ can take any value because they do not affect the output. Using the fact that the $P_{i}$'s are disjoint, we can continue this over all values of $i*$ from 1 to $k$ to get the desired result. 

\end{proof}

\textit{Proof of Proposition~\ref{prop:binary-pate}} From Lemma~\ref{lemma:coordindep}, because $f=[f_1,\hdots,f_k]$, then the worst-case sensitivity is $\Delta f = (\sum_i^k (\Delta f_i)^p)^{1/p}$ by expansion of sensitivity from Definition~\ref{def:sensitivity}. Similarly, we can construct an equivalent query by $k$ applications of Binary voting which gives a sensitivity of $\Delta f = \sum_i^k (\Delta f_i)$. These two are equivalent for $p=1$, i.e., the $\ell_1$ norm. Binary voting is suboptimal for norms of $p>1$ as a result of H\"older's inequality.\qedsymbol

\begin{prop} \label{prop:coorddepend}
For a multi-label function $f(X): \new{\mathcal{R}^d} \rightarrow \mathcal{R}^k$ to have a lower sensitivity than the sum of the worst-case sensitivities for each coordinate, i.e., 
$\Delta_1 f < \Sigma_{i=1}^k\Delta_1 fi$, $f$ must be coordinate-dependent. 
\end{prop}

\begin{proof}
Follows by contradiction since the function cannot be coordinate-independent.
\end{proof}

\begin{definition}[Mechanism for \powerset PATE]\label{def:powerset-pate}
Denote the \powerset operator as $\mathcal{P(\cdot)}$. For a sample $x$ and $2^k$ subsets (classes), let $f_j(x)\in \{0,1\}^k$ denote the $j-$th teacher model binary vector prediction. Let $n_i(x)$ be the vote count for the $i-$th subset (class), \ie $n_i(x) \triangleq |\{j: f_j(x) = P(\{0,1\}^k)_i\}|$. We define the \powerset PATE mechanism as

\begin{equation*}
    \mathcal{M}_\sigma (x) \triangleq \underset{i}{\argmax}\left\{n_i(x) + \mathcal{N}(0, \sigma^2)  \right\}.
\end{equation*}
\end{definition}

\section{Towards Private Extreme Multi-Label Classification}
\label{sec5:limitations}

We face extreme multi-label classifications~\cite{shen2020extreme} in many real-world applications such as semantic segmentation~\cite{zhou2017scene}, hash-tag suggestions for user images~\cite{denton2015user}, product categorisation~\cite{agrawal2013multi} and webpage annotation~\cite{partalas2015lshtc} where both input size and label size are extremely large. In this section, we investigate the privacy-accuracy tradeoffs of private semantic segmentation (a common and underlying example in extreme multi-label settings) that links each image pixel to its corresponding object class (an integer value) with a reasonable accuracy of above $60\%$ but an expensive privacy cost of $\varepsilon \approx 3,000$, for a relative "small" image of size $200\times200$, or $\approx40,000$ pixels. We conclude this section by proposing future directions to alleviate the privacy-accuracy tradeoffs in extreme multi-label settings.

We consider MIT ADE20K semantic segmentation dataset~\cite{zhou2017scene} that contains 150 objects including 35 stuff objects (e.g. sky, building) and 115 discrete objects (e.g. person, car). The label size for each image pixel is fixed (=150). However, the number of predicted labels for each image is the number of pixels, which varies across the dataset. To perform private semantic segmentation, we use PATE to label each image pixel. We split the training set of MIT ADE20K dataset into equally sized partitions for 20 teachers and train a Pyramid Pooling ResNet50-Dilated architecture of the Cascade Segmentation Module. The test accuracy of the ensemble of teachers (using $2000$ of the test images) with respect to the PATE noise standard deviation $\sigma_G$ between 0 and 5 varies from $67\%$ to $47\%$. We observe that the level of noise must be quite small, $\sigma< 3$, or there is a steep drop in accuracy of more than $10$ percentage points. The privacy cost is too high to provide meaningful guarantees $\varepsilon$, due to the small $\sigma$ and large number of pixels required to be labeled.
\begin{figure}[h]
  \begin{center}
    \includegraphics[width=0.23\textwidth]{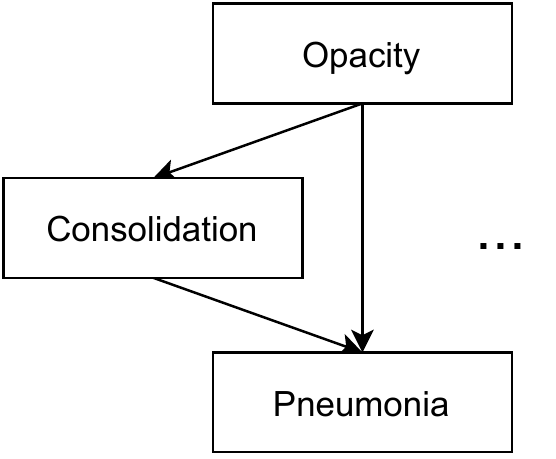}
  \end{center}
  \caption{\textbf{Example hierarchical structure of labels} in Chest radiography setting.}
  \label{fig:treeMulti}
\end{figure}

We believe that privacy analysis in extreme multi-label settings can be tightened by exploiting the semantics of inputs. For example in the semantic segmentation task, we can reduce the privacy costs by taking advantage of the dependency between pixels so that instead of releasing an answer per pixel, we can release only a single label per semantic region (a grouping of pixels). 
Exploring label dependence, rather than assuming label independence, may also enable tighter privacy loss analysis and improve accuracy, as our analysis of Proposition~\ref{prop:binary-pate} suggests. Label dependence is prevalent in many tasks, e.g., in healthcare labels are naturally organised into tree-like hierarchies such that domain experts (e.g. doctors) perform observations and diagnoses conditioned upon their parent node~\cite{van2012relationship}. 
Figure~\ref{fig:treeMulti} shows an example of the label structure where the root label node corresponds to the most generic disease of Opacity, while the leaf label node represents the most specific disease of Pneumonia~\cite{pham2021interpreting}. Pneumonia implies the presence of both Consolidation and Opacity diseases. Thus, there exist many possible methods to optimize the answering of queries. It may be possible to tighten the privacy loss due to the implications (or, correlations) between labels; or, to query labels in a specific order such that the all dependent nodes (Consolidation and Opacity) can be inferred by the agreed presence of parent nodes (Pneumonia) by the teacher ensemble.



In addition to exploiting the knowledge of input  and label domains, our analysis of the optimal settings for Binary PATE shows that privacy mechanisms can be tailored to the multi-label classifications. For example, k-fold adaptive bounds~\cite{kairouz2015composition} that draw tighter ($\ll$ sublinear) privacy bounds for homogenous privacy settings can be extended to heterogeneous $\varepsilon$ per label and per query settings of multi-label classification. However, it is unclear if and under what scenarios we can achieve a tighter bound. For instance, it is possible to take the maximum $\varepsilon$ across all queries, but if there is a large gap $k$-fold adaptive composition may yield looser bounds. These settings of coordinate dependence, high label correlations, and heterogeneous k-fold adaptive composition are interesting for future work.

We design an experiment where the baseline method obtains answers to all the labels while the new proposed method exploits the semantics and queries labels selectively.

\new{First, we generate positive dependency matrices for labels in each dataset and present results in Tables~\ref{tab:pascal_dependency_matrix},\ref{tab:mimic_dependency_matrix},\ref{tab:cxpert_dependency_matrix}, and~\ref{tab:padchest_dependency_matrix}.}
We find that the positive dependencies (e.g., if disease A is present then disease B is present as well) constitute a small fraction of the whole dataset. This is because there are many more negative than positive examples in the CheXpert dataset, which is caused by a class imbalance, a common problem in medical datasets. For instance, we find that if both Pneumonia and Pneumothorax are present then Lung Opacity occurs in 83.3\% of the cases. However, both Pneumonia and Pneumothorax are present in only 0.06\% of samples of the dataset. Thus, we consider negative instead of positive dependencies. For example, if Atelectasis is absent then Consolidation is absent as well in 98.7\% of the cases. After ignoring samples for which at least one of Atelectasis or Consolidation have missing values, the percentage of samples where both labels are negative is 83\%. We obtain the negative dependencies using the training set and generate the dependency matrix in~\Cref{matrix:dependency-chexpert}.

\begin{table*}[t]
\begin{center}
\caption{\textbf{Negative Dependency Matrix} for the first 5 labels form the CheXpert dataset.}
\label{matrix:dependency-chexpert}
\begin{tabular}{|l|l|l|l|l|l|}
\hline
\multicolumn{1}{|c|}{\textbf{}} & \multicolumn{1}{c|}{\textbf{Atelectasis}} & \multicolumn{1}{c|}{\textbf{Cardiomegaly}} & \multicolumn{1}{c|}{\textbf{Consolidation}} & \multicolumn{1}{c|}{\textbf{Edema}} & \multicolumn{1}{c|}{\textbf{Effusion}} \\ \hline
\textbf{Atelectasis}                    & \textit{}                                 & \textit{0.975}                             & \textit{0.987}                              & \textit{0.976}                      & \textit{0.983}                         \\ \hline
\textbf{Cardiomegaly}                   & 0.736                                     &                                            & 0.836                                       & 0.784                               & 0.869                                  \\ \hline
\textbf{Consolidation}                  & 0.527                                     & 0.591                                      & \textbf{}                                   & \textbf{0.631}                      & 0.790                                  \\ \hline
\textbf{Edema}                          & \textbf{0.625}                            & \textbf{0.665}                             & \textbf{0.758}                              &                                     & \textbf{0.822}                         \\ \hline
\textbf{Effusion}                       & 0.485                                     & 0.567                                      & 0.731                                       & 0.633                               &                                        \\ \hline
\end{tabular}
\end{center}
\end{table*}

\begin{table*}[t!]
\centering
\caption{Positive dependency (co-prevalence) matrix for labels in the Pascal-VOC dataset.}
\label{tab:pascal_dependency_matrix}
\adjustbox{max width=\textwidth}{\begin{tabular}{|l|l|l|l|l|l|l|l|l|l|l|l|l|l|l|l|l|l|l|l|l|}
\hline
 & \multicolumn{1}{|c|}{\textbf{aeroplane}} & \multicolumn{1}{c|}{\textbf{bicycle}} & \multicolumn{1}{c|}{\textbf{bird}} & \multicolumn{1}{c|}{\textbf{boat}} & \multicolumn{1}{c|}{\textbf{bottle}} & \multicolumn{1}{c|}{\textbf{bus}} & \multicolumn{1}{c|}{\textbf{car}} & \multicolumn{1}{c|}{\textbf{cat}} & \multicolumn{1}{c|}{\textbf{chair}} & \multicolumn{1}{c|}{\textbf{cow}} & \multicolumn{1}{c|}{\textbf{diningtable}} & \multicolumn{1}{c|}{\textbf{dog}} & \multicolumn{1}{c|}{\textbf{horse}} & \multicolumn{1}{c|}{\textbf{motorbike}} & \multicolumn{1}{c|}{\textbf{person}} & \multicolumn{1}{c|}{\textbf{pottedplant}} & \multicolumn{1}{c|}{\textbf{sheep}} & \multicolumn{1}{c|}{\textbf{sofa}} & \multicolumn{1}{c|}{\textbf{train}} & \multicolumn{1}{c|}{\textbf{tvmonitor}} \\ \hline
\textbf{aeroplane} & 1.000 & 0.000 & 0.001 & 0.007 & 0.000 & 0.004 & 0.061 & 0.000 & 0.001 & 0.000 & 0.000 & 0.000 & 0.000 & 0.001 & 0.133 & 0.003 & 0.000 & 0.000 & 0.000 & 0.000 \\ \hline
\textbf{bicycle} & 0.000 & 1.000 & 0.004 & 0.005 & 0.060 & 0.030 & 0.163 & 0.005 & 0.039 & 0.004 & 0.009 & 0.005 & 0.002 & 0.035 & 0.564 & 0.042 & 0.000 & 0.007 & 0.004 & 0.009 \\ \hline
\textbf{bird} & 0.001 & 0.003 & 1.000 & 0.012 & 0.004 & 0.001 & 0.005 & 0.003 & 0.008 & 0.004 & 0.001 & 0.006 & 0.004 & 0.000 & 0.052 & 0.008 & 0.006 & 0.000 & 0.000 & 0.001 \\ \hline
\textbf{boat} & 0.010 & 0.006 & 0.017 & 1.000 & 0.008 & 0.008 & 0.039 & 0.000 & 0.019 & 0.006 & 0.006 & 0.008 & 0.000 & 0.010 & 0.378 & 0.012 & 0.000 & 0.004 & 0.002 & 0.004 \\ \hline
\textbf{bottle} & 0.000 & 0.044 & 0.004 & 0.005 & 1.000 & 0.001 & 0.029 & 0.022 & 0.236 & 0.000 & 0.254 & 0.038 & 0.003 & 0.007 & 0.592 & 0.065 & 0.005 & 0.070 & 0.000 & 0.100 \\ \hline
\textbf{bus} & 0.007 & 0.040 & 0.002 & 0.009 & 0.002 & 1.000 & 0.384 & 0.000 & 0.000 & 0.000 & 0.000 & 0.002 & 0.002 & 0.026 & 0.477 & 0.009 & 0.002 & 0.000 & 0.007 & 0.000 \\ \hline
\textbf{car} & 0.033 & 0.076 & 0.003 & 0.016 & 0.018 & 0.134 & 1.000 & 0.005 & 0.020 & 0.007 & 0.006 & 0.030 & 0.018 & 0.089 & 0.445 & 0.020 & 0.004 & 0.002 & 0.018 & 0.005 \\ \hline
\textbf{cat} & 0.000 & 0.003 & 0.002 & 0.000 & 0.016 & 0.000 & 0.006 & 1.000 & 0.066 & 0.001 & 0.016 & 0.028 & 0.000 & 0.000 & 0.076 & 0.027 & 0.001 & 0.060 & 0.000 & 0.020 \\ \hline
\textbf{chair} & 0.001 & 0.017 & 0.005 & 0.008 & 0.139 & 0.000 & 0.019 & 0.055 & 1.000 & 0.000 & 0.335 & 0.065 & 0.002 & 0.005 & 0.428 & 0.159 & 0.000 & 0.156 & 0.002 & 0.152 \\ \hline
\textbf{cow} & 0.000 & 0.006 & 0.010 & 0.010 & 0.000 & 0.000 & 0.026 & 0.003 & 0.000 & 1.000 & 0.000 & 0.013 & 0.006 & 0.003 & 0.159 & 0.000 & 0.003 & 0.000 & 0.000 & 0.000 \\ \hline
\textbf{diningtable} & 0.000 & 0.008 & 0.002 & 0.005 & 0.304 & 0.000 & 0.011 & 0.027 & 0.679 & 0.000 & 1.000 & 0.017 & 0.000 & 0.000 & 0.551 & 0.159 & 0.000 & 0.090 & 0.000 & 0.051 \\ \hline
\textbf{dog} & 0.000 & 0.002 & 0.004 & 0.003 & 0.022 & 0.001 & 0.029 & 0.023 & 0.065 & 0.003 & 0.008 & 1.000 & 0.003 & 0.002 & 0.240 & 0.024 & 0.008 & 0.076 & 0.000 & 0.013 \\ \hline
\textbf{horse} & 0.000 & 0.002 & 0.006 & 0.000 & 0.004 & 0.002 & 0.046 & 0.000 & 0.006 & 0.004 & 0.000 & 0.008 & 1.000 & 0.002 & 0.472 & 0.004 & 0.004 & 0.000 & 0.002 & 0.002 \\ \hline
\textbf{motorbike} & 0.002 & 0.037 & 0.000 & 0.009 & 0.009 & 0.021 & 0.203 & 0.000 & 0.011 & 0.002 & 0.000 & 0.006 & 0.002 & 1.000 & 0.575 & 0.032 & 0.006 & 0.002 & 0.000 & 0.000 \\ \hline
\textbf{person} & 0.021 & 0.074 & 0.009 & 0.045 & 0.104 & 0.047 & 0.125 & 0.019 & 0.127 & 0.011 & 0.081 & 0.071 & 0.052 & 0.070 & 1.000 & 0.041 & 0.013 & 0.069 & 0.036 & 0.050 \\ \hline
\textbf{pottedplant} & 0.004 & 0.042 & 0.011 & 0.011 & 0.088 & 0.007 & 0.042 & 0.051 & 0.363 & 0.000 & 0.180 & 0.055 & 0.004 & 0.030 & 0.319 & 1.000 & 0.000 & 0.164 & 0.014 & 0.106 \\ \hline
\textbf{sheep} & 0.000 & 0.000 & 0.015 & 0.000 & 0.012 & 0.003 & 0.015 & 0.003 & 0.000 & 0.003 & 0.000 & 0.031 & 0.006 & 0.009 & 0.169 & 0.000 & 1.000 & 0.000 & 0.000 & 0.000 \\ \hline
\textbf{sofa} & 0.000 & 0.006 & 0.000 & 0.003 & 0.078 & 0.000 & 0.003 & 0.094 & 0.291 & 0.000 & 0.083 & 0.142 & 0.000 & 0.001 & 0.433 & 0.134 & 0.000 & 1.000 & 0.001 & 0.125 \\ \hline
\textbf{train} & 0.000 & 0.004 & 0.000 & 0.002 & 0.000 & 0.005 & 0.040 & 0.000 & 0.004 & 0.000 & 0.000 & 0.000 & 0.002 & 0.000 & 0.287 & 0.015 & 0.000 & 0.002 & 1.000 & 0.000 \\ \hline
\textbf{tvmonitor} & 0.000 & 0.008 & 0.002 & 0.003 & 0.129 & 0.000 & 0.010 & 0.037 & 0.331 & 0.000 & 0.055 & 0.029 & 0.002 & 0.000 & 0.365 & 0.101 & 0.000 & 0.146 & 0.000 & 1.000 \\ \hline
\end{tabular}}
\end{table*}

\begin{table*}[t!]
\centering
\caption{Positive dependency (co-prevalence) matrix for labels in the MIMIC-CXR dataset.}
\label{tab:mimic_dependency_matrix}
\adjustbox{max width=\textwidth}{\begin{tabular}{|l|l|l|l|l|l|l|l|l|l|l|l|l|l|}
\hline
 & \multicolumn{1}{|c|}{\textbf{Enlarged Cardiomediastinum}} & \multicolumn{1}{c|}{\textbf{Cardiomegaly}} & \multicolumn{1}{c|}{\textbf{Lung Opacity}} & \multicolumn{1}{c|}{\textbf{Lung Lesion}} & \multicolumn{1}{c|}{\textbf{Edema}} & \multicolumn{1}{c|}{\textbf{Consolidation}} & \multicolumn{1}{c|}{\textbf{Pneumonia}} & \multicolumn{1}{c|}{\textbf{Atelectasis}} & \multicolumn{1}{c|}{\textbf{Pneumothorax}} & \multicolumn{1}{c|}{\textbf{Pleural Effusion}} & \multicolumn{1}{c|}{\textbf{Pleural Other}} & \multicolumn{1}{c|}{\textbf{Fracture}} & \multicolumn{1}{c|}{\textbf{Support Devices}} \\ \hline
\textbf{Enlarged Cardiomediastinum} & 1.000 & 0.254 & 0.353 & 0.056 & 0.173 & 0.078 & 0.076 & 0.349 & 0.093 & 0.391 & 0.017 & 0.032 & 0.527 \\ \hline
\textbf{Cardiomegaly} & 0.041 & 1.000 & 0.263 & 0.021 & 0.246 & 0.059 & 0.079 & 0.316 & 0.044 & 0.400 & 0.012 & 0.022 & 0.477 \\ \hline
\textbf{Lung Opacity} & 0.050 & 0.231 & 1.000 & 0.059 & 0.157 & 0.057 & 0.166 & 0.278 & 0.047 & 0.343 & 0.015 & 0.019 & 0.381 \\ \hline
\textbf{Lung Lesion} & 0.063 & 0.145 & 0.466 & 1.000 & 0.073 & 0.082 & 0.115 & 0.194 & 0.049 & 0.284 & 0.026 & 0.017 & 0.228 \\ \hline
\textbf{Edema} & 0.046 & 0.410 & 0.297 & 0.017 & 1.000 & 0.092 & 0.116 & 0.257 & 0.029 & 0.531 & 0.008 & 0.012 & 0.440 \\ \hline
\textbf{Consolidation} & 0.052 & 0.245 & 0.271 & 0.049 & 0.231 & 1.000 & 0.224 & 0.230 & 0.057 & 0.532 & 0.010 & 0.016 & 0.525 \\ \hline
\textbf{Pneumonia} & 0.033 & 0.212 & 0.505 & 0.044 & 0.186 & 0.143 & 1.000 & 0.222 & 0.018 & 0.307 & 0.012 & 0.012 & 0.285 \\ \hline
\textbf{Atelectasis} & 0.055 & 0.311 & 0.311 & 0.027 & 0.152 & 0.054 & 0.082 & 1.000 & 0.068 & 0.501 & 0.007 & 0.024 & 0.471 \\ \hline
\textbf{Pneumothorax} & 0.063 & 0.186 & 0.224 & 0.030 & 0.073 & 0.057 & 0.028 & 0.291 & 1.000 & 0.332 & 0.011 & 0.047 & 0.602 \\ \hline
\textbf{Pleural Effusion} & 0.052 & 0.333 & 0.325 & 0.034 & 0.266 & 0.106 & 0.096 & 0.423 & 0.066 & 1.000 & 0.010 & 0.021 & 0.481 \\ \hline
\textbf{Pleural Other} & 0.061 & 0.267 & 0.402 & 0.084 & 0.108 & 0.057 & 0.102 & 0.159 & 0.062 & 0.281 & 1.000 & 0.077 & 0.321 \\ \hline
\textbf{Fracture} & 0.050 & 0.211 & 0.213 & 0.024 & 0.073 & 0.038 & 0.044 & 0.239 & 0.109 & 0.243 & 0.033 & 1.000 & 0.304 \\ \hline
\textbf{Support Devices} & 0.066 & 0.372 & 0.338 & 0.025 & 0.206 & 0.098 & 0.083 & 0.373 & 0.112 & 0.450 & 0.011 & 0.024 & 1.000 \\ \hline
\end{tabular}}
\end{table*}

\begin{table*}[t!]
\centering
\caption{Positive dependency (co-prevalence) matrix for labels in the CheXpert dataset.}
\label{tab:cxpert_dependency_matrix}
\adjustbox{max width=\textwidth}{\begin{tabular}{|l|l|l|l|l|l|l|l|l|l|l|l|}
\hline
 & \multicolumn{1}{|c|}{\textbf{Atelectasis}} & \multicolumn{1}{c|}{\textbf{Cardiomegaly}} & \multicolumn{1}{c|}{\textbf{Consolidation}} & \multicolumn{1}{c|}{\textbf{Edema}} & \multicolumn{1}{c|}{\textbf{Enlarged Cardiomediastinum}} & \multicolumn{1}{c|}{\textbf{Fracture}} & \multicolumn{1}{c|}{\textbf{Lung Lesion}} & \multicolumn{1}{c|}{\textbf{Lung Opacity}} & \multicolumn{1}{c|}{\textbf{Pneumonia}} & \multicolumn{1}{c|}{\textbf{Pneumothorax}} & \multicolumn{1}{c|}{\textbf{Pleural Effusion}} \\ \hline
\textbf{Atelectasis} & 1.000 & 0.107 & 0.055 & 0.056 & 0.048 & 0.053 & 0.056 & 0.414 & 0.033 & 0.084 & 0.471 \\ \hline
\textbf{Cardiomegaly} & 0.118 & 1.000 & 0.049 & 0.192 & 0.101 & 0.039 & 0.050 & 0.354 & 0.030 & 0.023 & 0.354 \\ \hline
\textbf{Consolidation} & 0.118 & 0.095 & 1.000 & 0.047 & 0.042 & 0.024 & 0.100 & 0.384 & 0.128 & 0.036 & 0.426 \\ \hline
\textbf{Edema} & 0.105 & 0.327 & 0.041 & 1.000 & 0.042 & 0.019 & 0.042 & 0.431 & 0.045 & 0.024 & 0.483 \\ \hline
\textbf{Enlarged Cardiomediastinum} & 0.107 & 0.205 & 0.044 & 0.050 & 1.000 & 0.045 & 0.086 & 0.343 & 0.017 & 0.049 & 0.258 \\ \hline
\textbf{Fracture} & 0.119 & 0.079 & 0.026 & 0.023 & 0.046 & 1.000 & 0.068 & 0.284 & 0.014 & 0.054 & 0.193 \\ \hline
\textbf{Lung Lesion} & 0.085 & 0.068 & 0.071 & 0.033 & 0.058 & 0.045 & 1.000 & 0.533 & 0.063 & 0.054 & 0.323 \\ \hline
\textbf{Lung Opacity} & 0.136 & 0.106 & 0.059 & 0.076 & 0.051 & 0.041 & 0.116 & 1.000 & 0.072 & 0.062 & 0.384 \\ \hline
\textbf{Pneumonia} & 0.088 & 0.073 & 0.160 & 0.064 & 0.020 & 0.017 & 0.112 & 0.585 & 1.000 & 0.017 & 0.235 \\ \hline
\textbf{Pneumothorax} & 0.149 & 0.037 & 0.030 & 0.023 & 0.039 & 0.042 & 0.064 & 0.336 & 0.011 & 1.000 & 0.381 \\ \hline
\textbf{Pleural Effusion} & 0.186 & 0.127 & 0.079 & 0.102 & 0.046 & 0.034 & 0.085 & 0.463 & 0.035 & 0.085 & 1.000 \\ \hline
\end{tabular}}
\end{table*}

\begin{table*}[t!]
\centering
\caption{Positive dependency (co-prevalence) matrix for labels in the PadChest dataset.}
\label{tab:padchest_dependency_matrix}
\adjustbox{max width=\textwidth}{\begin{tabular}{|l|l|l|l|l|l|l|l|l|l|l|l|l|l|l|l|}
\hline
 & \multicolumn{1}{|c|}{\textbf{Atelectasis}} & \multicolumn{1}{c|}{\textbf{Cardiomegaly}} & \multicolumn{1}{c|}{\textbf{Consolidation}} & \multicolumn{1}{c|}{\textbf{Edema}} & \multicolumn{1}{c|}{\textbf{Effusion}} & \multicolumn{1}{c|}{\textbf{Emphysema}} & \multicolumn{1}{c|}{\textbf{Fibrosis}} & \multicolumn{1}{c|}{\textbf{Fracture}} & \multicolumn{1}{c|}{\textbf{Hernia}} & \multicolumn{1}{c|}{\textbf{Infiltration}} & \multicolumn{1}{c|}{\textbf{Mass}} & \multicolumn{1}{c|}{\textbf{Nodule}} & \multicolumn{1}{c|}{\textbf{Pleural Thickening}} & \multicolumn{1}{c|}{\textbf{Pneumonia}} & \multicolumn{1}{c|}{\textbf{Pneumothorax}} \\ \hline
\textbf{Atelectasis} & 1.000 & 0.130 & 0.018 & 0.002 & 0.138 & 0.013 & 0.005 & 0.041 & 0.023 & 0.171 & 0.019 & 0.045 & 0.041 & 0.074 & 0.005 \\ \hline
\textbf{Cardiomegaly} & 0.062 & 1.000 & 0.010 & 0.011 & 0.077 & 0.002 & 0.013 & 0.035 & 0.031 & 0.129 & 0.007 & 0.028 & 0.036 & 0.029 & 0.001 \\ \hline
\textbf{Consolidation} & 0.071 & 0.081 & 1.000 & 0.005 & 0.227 & 0.013 & 0.009 & 0.032 & 0.009 & 1.000 & 0.018 & 0.065 & 0.018 & 0.452 & 0.004 \\ \hline
\textbf{Edema} & 0.036 & 0.462 & 0.025 & 1.000 & 0.416 & 0.000 & 0.010 & 0.015 & 0.010 & 0.736 & 0.015 & 0.030 & 0.041 & 0.137 & 0.005 \\ \hline
\textbf{Effusion} & 0.165 & 0.194 & 0.068 & 0.025 & 1.000 & 0.013 & 0.007 & 0.053 & 0.008 & 0.301 & 0.028 & 0.064 & 0.043 & 0.095 & 0.007 \\ \hline
\textbf{Emphysema} & 0.055 & 0.015 & 0.014 & 0.000 & 0.049 & 1.000 & 0.015 & 0.048 & 0.007 & 0.145 & 0.018 & 0.077 & 0.102 & 0.054 & 0.015 \\ \hline
\textbf{Fibrosis} & 0.030 & 0.159 & 0.013 & 0.003 & 0.037 & 0.021 & 1.000 & 0.037 & 0.039 & 0.569 & 0.009 & 0.043 & 0.085 & 0.036 & 0.003 \\ \hline
\textbf{Fracture} & 0.063 & 0.113 & 0.012 & 0.001 & 0.068 & 0.017 & 0.010 & 1.000 & 0.030 & 0.084 & 0.011 & 0.058 & 0.058 & 0.026 & 0.007 \\ \hline
\textbf{Hernia} & 0.062 & 0.174 & 0.006 & 0.001 & 0.018 & 0.004 & 0.018 & 0.053 & 1.000 & 0.073 & 0.063 & 0.043 & 0.049 & 0.014 & 0.001 \\ \hline
\textbf{Infiltration} & 0.088 & 0.140 & 0.129 & 0.019 & 0.130 & 0.017 & 0.049 & 0.028 & 0.014 & 1.000 & 0.012 & 0.067 & 0.043 & 0.292 & 0.004 \\ \hline
\textbf{Mass} & 0.093 & 0.071 & 0.022 & 0.004 & 0.115 & 0.020 & 0.007 & 0.035 & 0.115 & 0.117 & 1.000 & 0.125 & 0.030 & 0.071 & 0.015 \\ \hline
\textbf{Nodule} & 0.049 & 0.064 & 0.018 & 0.002 & 0.058 & 0.019 & 0.008 & 0.041 & 0.017 & 0.142 & 0.028 & 1.000 & 0.063 & 0.070 & 0.005 \\ \hline
\textbf{Pleural Thickening} & 0.051 & 0.095 & 0.006 & 0.003 & 0.045 & 0.029 & 0.018 & 0.048 & 0.023 & 0.104 & 0.008 & 0.072 & 1.000 & 0.035 & 0.003 \\ \hline
\textbf{Pneumonia} & 0.086 & 0.071 & 0.133 & 0.008 & 0.093 & 0.014 & 0.007 & 0.020 & 0.006 & 0.665 & 0.017 & 0.075 & 0.033 & 1.000 & 0.000 \\ \hline
\textbf{Pneumothorax} & 0.081 & 0.027 & 0.018 & 0.004 & 0.108 & 0.063 & 0.009 & 0.085 & 0.004 & 0.130 & 0.054 & 0.076 & 0.036 & 0.000 & 1.000 \\ \hline
\end{tabular}}
\end{table*}

We compare multi-label PATE executed for each label vs using the semantics and querying the first label (Atelectasis) only, followed by (1) skipping the remaining labels and setting them as negative if the first label is negative, or (2) querying the other labels if the first label is positive. As expected, leveraging the semantics increases the number of answered queries from 35 to 127 for the same privacy budget $\varepsilon=8$ of at the cost of lower performance (less accurate answers to the queries). However, increasing the number of answered queries by adding more privacy noise ($\sigma=67.5$) causes the answered queries to be less accurate than by exploiting the label dependencies. We show a detailed comparison in Table~\ref{tab:extreme-multilabel}.

\begin{table*}
\caption{\textbf{Exploit label dependencies for the multi-label classification}.}
\label{tab:extreme-multilabel}
\vskip -0.3in
\begin{center}
\begin{small}
\begin{sc}
\begin{tabular}{cccccc}
\toprule
\multicolumn{1}{c}{\textbf{}}          & \multicolumn{1}{c}{\textbf{\# of queries answered}} & \multicolumn{1}{c}{\textbf{ACC}} & \multicolumn{1}{c}{\textbf{BAC}} & \multicolumn{1}{c}{\textbf{AUC}} & \multicolumn{1}{c}{\textbf{MAP}} \\
\toprule
\textbf{Answer all labels}             & \textit{35}                                         & \textit{0.84}                    & \textit{0.82}                    & \textit{0.82}                    & \textit{0.64}                    \\
\textbf{Increse privacy noise}         & 127                                                 & 0.63                             & 0.61                             & 0.61                             & 0.44                             \\
\textbf{Exploit negative dependencies} & 127   & 0.68                             & 0.72                    & 0.72                    & 0.48      \\                   
\bottomrule
\end{tabular}
\end{sc}
\end{small}
\end{center}
\vskip -0.2in
\end{table*}

Note that in the above example we consider the first five labels from the CheXpert dataset. We use the same setup as for the comparison between DPSGD and multi-label PATE~\ref{sec:dpsgd-pate-chexpert}. The metrics are computed on the same 127 queries (to obtain 127 answered queries for the \textit{Answer all labels} we increase its privacy budget from 8 to 26.5).

\section{Datasets and Model Architectures}
\label{app:datasets}


\begin{figure*}[ht]
\begin{center}
\includegraphics[width=0.8\textwidth]{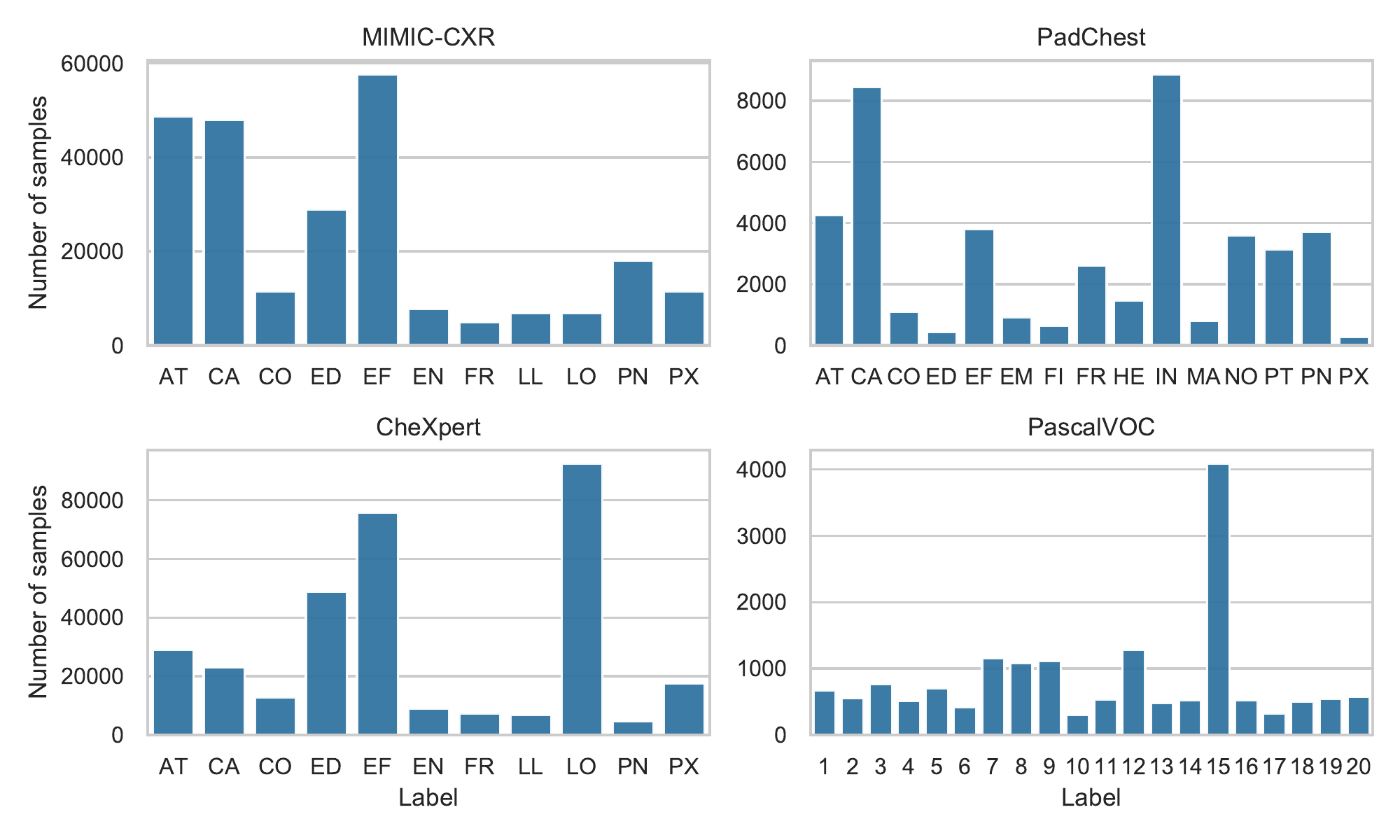}
\caption{The label distribution for each multi-label dataset (Pascal VOC, CheXpert, MIMIC-CXR, and PadChest). 
}
\label{fig:datasets_class_distribution}
\end{center}
\end{figure*}

We experiment on four \multi datasets. First, we use the common computer vision \multi dataset, Pascal VOC 2012. Three other of these datasets are privacy sensitive large-scale medical datasets that are commonly used in ML for healthcare: CheXpert~\cite{chexpert2019},  MIMIC-CXR~\cite{johnson2019mimiccxrjpg} and PadChest~\cite{padchest2020}.
These medical datasets present a realistic and large-scale application for \multi \capc. 

\textbf{Pascal VOC 2012} contains $11,540$ images that are split into $5,717$ images for training and $5,823$ images for validation~\cite{everingham2010pascal}. There are 20 classes of object labels (with their index in parentheses) -- aeroplane (1), bicycle (2), bird (3), boat (4), bottle (5), bus (6), car (7), cat (8), chair (9), cow (10), dining table (11), dog (12), horse (13), motorbike (14), person (15), potted plant (16), sheep (17), sofa (18), train (19), and tv monitor (20). We use a ResNet-50 model~\cite{he2016deep} that was pre-trained on ImageNet~\cite{deng2009imagenet}.

\textbf{Medical datasets} contain chest radiographs (X-ray) images. CheXpert~\cite{chexpert2019} has $224,316$ radiographs. MIMIC-CXR-JPG~\cite{johnson2019mimiccxrjpg} contains $377,110$, and PadChest~\cite{padchest2020} has $160,868$. The goal in each dataset is to predict presence of pathologies. However, there are differences between pathology labels across these three datasets. 
X-ray images of CheXpert are annotated with 11 pathologies--
Atelectasis, Cardiomegaly, Consolidation, Edema, Effusion, Enlarged Cardiomediastinum, Fracture, Lung Lesion, Lung Opacity, Pneumonia and Pneumothorax.
MIMIC-CXR-JPG includes 11 pathologies--
Enlarged Cardiomediastinum, Cardiomegaly, Lung Opacity, Lung Lesion,
Edema, Consolidation, Pneumonia, Atelectasis, Pneumothorax, Pleural Effusion, Pleural Other, Fracture, Support Devices.
PadChest includes 15 pathologies-- 
Atelectasis, Cardiomegaly, Consolidation, Edema, Effusion, Emphysema, Fibrosis, Fracture, Hernia, Infiltration, Mass, Nodule, Pleural\_Thickening, Pneumonia and Pneumothorax.

The pathology and its code (in parenthesis) is as follows: Atelectasis (AT), Cardiomegaly (CA), Consolidation (CO), Edema (ED), Effusion (EF), Emphysema (EM), Enlarged Cardiomediastinum (EN), Fibrosis (FI), Fracture (FR), Hernia (HE), Infiltration (IN), Lung Lesion (LL), Lung Opacity (LO), Mass (MA), Nodule (MO), Pleural\_Thickening (PT), Pneumonia (PN), Pneumothorax (PX).

All datasets obtain labels from associated reports. Both CheXpert and MIMIC-CXR-JPG use the CheXpert labelling system, which is a rule based approach.
PadChest obtains reports annotated by trained radiologists, then trains an attention-based recurrent neural network to predict on these annotations, and labels the remaining data. On all three datasets, we train a DenseNet-121~\cite{densenet2017} model and filter only frontal images (AP and PA). The main difference between our setup and the one from~\cite{xrayCrossDomain2020} is that we use both frontal views \textit{AP} and \textit{PA}, while the cited work uses only one of the frontal views, thus either \textit{AP} and \textit{PA}. 

The label distribution for each multi-label dataset (Pascal VOC, CheXpert, MIMIC-CXR, and PadChest) is presented in Figure~\ref{fig:datasets_class_distribution}. The label distribution for medical datasets is more unbalanced than for Pascal VOC, due to prevalence or rarity of certain diseases.

\section{Evaluation Metrics}
\label{sec:evalution-metrics}

The \textit{accuracy metric (ACC)} used refers to the average accuracy over all labels, i.e., the accuracy is measured for each label individually over the samples in the test set and then averaged to give an overall accuracy for the classifier. This corresponds to micro averaging in~\cite{ConsistentMultiLabel}. To put the above definition in words, the per label accuracy is the proportion of correct predictions (both true positives and true negatives) among the total number of cases examined. We use the function \textit{sklearn.metrics.accuracy\_score} from~\cite{scikit-learn} in the code.

The \textit{balanced accuracy (BAC)} is the macro-average of recall and true-negative-rate scores per class. In the binary case, balanced accuracy is equal to the arithmetic mean of sensitivity (true-positive-rate, recall) and specificity (true-negative-rate). We use the standard method \textit{sklearn.metrics.balanced\_accuracy\_score} from~\cite{scikit-learn} to compute the balanced accuracy per label and then average the scores across all labels. If the data set is balanced or the classifier performs equally well on either class, the balanced accuracy metric reduces to the conventional accuracy (i.e., the number of correct predictions divided by the total number of predictions).

The \textit{Area-Under-the-Curve (AUC)} refers to the area under the receiver operating characteristic (ROC) curve. The ROC curve plots the TP rate against the FP rate. The AUC can be interpreted as the probability that the model ranks a random positive example more highly than a random negative example. We use the method \textit{sklearn.metrics.roc\_auc\_score} from~\cite{scikit-learn} to compute the AUC per label which we then average over all of the labels. The AUC metric is useful in that it does not depend on the classification threshold used. 

The \textit{Mean Average Precision metric (MAP)} is the mean of the average precision of the model over all classes. The average precision for each class is found as the area under the precision-recall curve. The full code showing the functions for MAP and AP can be found in our code at \textit{/datasets/deprecated/coco/helper\_functions/helper\_functions.py} which was originally written for~\cite{ben2020asymmetric}.

For a comprehensive evaluation of different algorithms, it is a common practice to test their performance on various measures and a better algorithm is the one that performs well on most of the measures~\cite{multi-label-hamming}.

In the code, 
we followed the evaluation procedure from~\cite{ben2020asymmetric} and also computed these metrics: true-positives, false-positives, false-negatives, true-negatives, average per-Class precision (CP), recall (CR), F1 (CF1), and the average Overall precision (OP), recall (OR) and F1 (OF1).

\section{Experimental Details}
\label{app:train-details}
Our experiments were performed on machines with Intel\textregistered Xeon\textregistered Silver 4210 processor, 128 GB of RAM, and four NVIDIA GeForce RTX 2080 graphics cards, running Ubuntu 18.04.

We train ResNet-$50$~\cite{he2016deep} models on the Pascal VOC 2012~\cite{everingham2010pascal} dataset by minimizing multi-label soft margin loss function in $1000$ epochs. As the optimiser, we use SGD with learning rate and weight decay of $0.001$ and $1e-4$, respectively. 
We split the $5,717$ images of the training set to $50$ private training sets in order to train $50$ Pascal VOC teacher models. 

We train the medical datasets CheXpert, MIMIC-CXR, and PadChest using the standard DenseNet-$121$ architecture trained for $100$ epochs, with a weight decay of $1e-5$ and the Adam optimizer (learning rate=$0.001$) with the adam\_amsgrad paremeter set. We reduce the learning rate when it plateaus, batch size is $64$. The loss type is BCE (Binary Cross Entropy) with logits, and probability threshold is adjusted per label (where the standard value is $0.5$). For $50$ CheXpert models, each model is trained on $3750$ private train samples. PadChest uses $9223$ private samples per each of $10$ teacher models.

Regarding the metrics, balanced accuracy (BAC), area under the curve (AUC), and mean average precision are common metrics used to asses performance of the multi-label classification~\cite{ben2020asymmetric}. Note that because we are in a sparse-hot multi-label setting (where a sample has many 0's - negative labels) the accuracy is not a preferred metric because a model can achieve high accuracy by returning all 0's; we include the accuracy metric for completeness.

\textbf{For our \capc experiments}, we sample uniformly, without replacement, data points from the respective training data distribution to create disjoint partitions $D_i$ of equal size for each party $i$. We use $50$ total parties for Pascal VOC, CheXpert, and MIMIC-CXR, as well as $10$ for PadChest; this choice depended on the sizes of the datasets and the performance of the models on the resulting partitioned data (here, we ensured no single model dropped below $60\%$ BAC). 
We use $Q=3$ querying parties and sample (without replacement) at least $1000$ data points from the test distribution to form the unlabeled set for each querying party. We leave at least $1000$ held-out data points for evaluating models. 
We first train party $i$'s model on their private data $D_i$, then simulate \multi \capc learning by having each querying party complete the protocol with all other parties as answering parties. Using the new labelled data provided by \multi \capc, we retrain each querying party's model on their original data $D_i$ plus the new labelled data. We average metrics over at least 3 runs with each using a different random seed.

\section{Additional Experiments}\label{app:additional-experiments}

\begin{table*}[t]
\caption{\textbf{Pascal VOC: Performance of Binary PATE vs Powerset PATE} w.r.t. number of answered queries, ACC, BAC, AUC, and MAP as measured on the test set with the specified $\sigma_{\text{GNMax}}$. PB ($\varepsilon$) is the privacy budget. We use 50 teacher models. When we limit number of labels per dataset, we select the first $k$ labels.}
\label{tab:powerset-binary-pate-pascal}
\vskip -0.3in
\begin{center}
\begin{small}
\begin{sc}
\begin{tabular}{cccccccccc}\toprule Dataset & Mode & \shortstack{\# of \\ Labels} &  \shortstack{Queries \\ answered} & PB ($\varepsilon$) & $\sigma_{\text{GNMax}}$ & ACC &                  BAC &          AUC &          MAP \\
\hline
Pascal VOC & Binary & 1 & \textbf{5464} & 20 & 2 & .98 & .84 & .84 & .75 \\
Pascal VOC & Powerset & 1 & \textbf{5464} & 20 & 2 & .98 & .84 & .84 & .75 \\
Pascal VOC & Binary & 2 & \textbf{5442} & 20 & 5 & .97 & .74 & .74 & .54 \\
Pascal VOC & Powerset & 2 & \textbf{5464} & 20 & 5 & .97 & .74 & .74 & .55 \\
Pascal VOC & Binary & 3 & \textbf{3437} & 20 & 7 & .97 & .72 & .72 & .51 \\
Pascal VOC & Powerset & 3 & \textbf{3416} & 20 & 7 & .97 & .72 & .72 & .51 \\
Pascal VOC & Binary & 5 & \textbf{2398} & 20 & 8 & .96 & .66 & .66 & .39 \\
Pascal VOC & Powerset & 5 & \textbf{2040} & 20 & 8 & .96 & .65 & .65 & .33 \\
Pascal VOC & Binary & 8 & \textbf{1543} & 20 & 7 & .96 & .69 & .69 & .46 \\
Pascal VOC & Powerset & 8 & \textbf{1192} & 20 & 7 & .95 & .68 & .68 & .37 \\
Pascal VOC & Binary & 11 & \textbf{1190} & 20 & 11 & .95 & .66 & .66 & .37 \\
Pascal VOC & Powerset & 11 & \textbf{702} & 20 & 4 & .95 & .66 & .66 & .37 \\
Pascal VOC & Binary & 14 & \textbf{888} & 20 & 12 & .95 & .65 & .65 & .36 \\
Pascal VOC & Powerset & 14 & \textbf{417} & 20 & 3 & .95 & .65 & .65 & .37 \\
Pascal VOC & Binary & 15 & \textbf{465} & 20 & 8 & .94 & .66 & .66 & .36 \\
Pascal VOC & Powerset & 15 & \textbf{165} & 20 & 3 & .94 & .66 & .66 & .36 \\
Pascal VOC & Binary & 16 & \textbf{461} & 20 & 7 & .95 & .65 & .65 & .38 \\
Pascal VOC & Powerset & 16 & \textbf{94} & 20 & 2 & .95 & .65 & .65 & .38 \\
Pascal VOC & Binary & 18 & \textbf{446} & 20 & 2 & .95 & .65 & .65 & .36 \\
Pascal VOC & Powerset & 18 & \textbf{94} & 20 & 2 & .95 & .65 & .65 & .36 \\
Pascal VOC & Binary & 20 & \textbf{427} & 20 & 7 & .95 & .65 & .65 & .37 \\
Pascal VOC & Powerset & 20 & \textbf{78} & 20 & 2 & .95 & .65 & .65 & .33 \\
\bottomrule
\end{tabular}
\end{sc}
\end{small}
\end{center}
\vskip -0.0in
\end{table*}

\begin{table*}[t]
\caption{\textbf{CheXpert: Performance of Binary PATE vs Powerset PATE} w.r.t. number of answered queries, ACC, BAC, AUC, and MAP (as measured on the test set with the specified $\sigma_{\text{GNMax}}$. PB ($\varepsilon$) is the privacy budget. We use 50 teacher models. When we limit number of labels per dataset, we select the first $k$ labels.}
\label{tab:powerset-binary-pate-cxpert}
\vskip -0.0in
\begin{center}
\begin{small}
\begin{sc}
\begin{tabular}{cccccccccc}\toprule Dataset & Mode & \shortstack{\# of \\ Labels} &  \shortstack{Queries \\ answered} & PB ($\varepsilon$) & $\sigma_{\text{GNMax}}$ & ACC &                  BAC &          AUC &          MAP \\
\hline
CheXpert & Binary & 1 & \textbf{988} & 20 & 8 & .71 & .71 & .71 & .63 \\
CheXpert & Powerset & 1 & \textbf{988} & 20 & 8 & .71 & .71 & .71 & .63 \\

CheXpert & Binary & 2 & \textbf{898} & 20 & 18 & .73 & .72 & .72 & .62 \\
CheXpert & Powerset & 2 & \textbf{872} & 20 & 12 & .72 & .72 & .72 & .60 \\

CheXpert & Binary & 3 & \textbf{399} & 20 & 13 & .75 & .77 & .77 & .55 \\
CheXpert & Powerset & 3 & \textbf{674} & 20 & 10 & .74 & .76 & .76 & .55 \\

CheXpert & Binary & 4 & \textbf{554} & 20 & 17 & .74 & .75 & .75 & .60 \\
CheXpert & Powerset & 4 & \textbf{323} & 20 & 5 & .76 & .78 & .78 & .60 \\

CheXpert & Binary & 5 & \textbf{320} & 20 & 17 & .74 & .76 & .76 & .59 \\
CheXpert & Binary & 5 & \textbf{932} & 20 & 29 & .70 & .70 & .70 & .54 \\
CheXpert & Powerset & 5 & \textbf{582} & 20 & 10 & .74 & .76 & .76 & .54 \\

CheXpert & Binary & 6 & \textbf{299} & 20 & 18 & .74 & .75 & .75 & .56 \\
CheXpert & Powerset & 6 & \textbf{392} & 20 & 7 & .74 & .75 & .75 & .55 \\

CheXpert & Binary & 7 & \textbf{157} & 20 & 12 & .73 & .74 & .74 & .51 \\
CheXpert & Powerset & 7 & \textbf{338} & 20 & 7 & .73 & .72 & .72 & .51 \\

CheXpert & Binary & 8 & \textbf{145} & 20 & 13 & .73 & .73 & .73 & .51 \\
CheXpert & Powerset & 8 & \textbf{200} & 20 & 5 & .73 & .71 & .71 & .50 \\

CheXpert & Binary & 9 & \textbf{159} & 20 & 16 & .71 & .70 & .70 & .50 \\
CheXpert & Powerset & 9 & \textbf{241} & 20 & 6 & .71 & .70 & .70 & .50 \\

CheXpert & Binary & 10 & \textbf{105} & 20 & 11 & .74 & .72 & .72 & .50 \\
CheXpert & Powerset & 10 & \textbf{146} & 20 & 4 & .73 & .71 & .71 & .50 \\

CheXpert & Binary & 11 & \textbf{201} & 20 & 5 & .71 & .68 & .68 & .45 \\
CheXpert & Powerset & 11 & \textbf{152} & 20 & 20 & .71 & .68 & .68 & .44 \\
\bottomrule
\end{tabular}
\end{sc}
\end{small}
\end{center}
\vskip -0.0in
\end{table*}

\subsection{Powerset with $\tau$-voting}
\label{sec:powerset-tau-appendix}

The integration of $\tau$-voting into the \powerset methods does not lower the sensitivity of the mechanism because changing one of the teachers potentially decreases the vote count from one class and increase it in another one. Thus, the sensitivity remains 2. 

Intuitively, up to $\tau$ positive labels could be selected based on the confidence of the positive value for a label, where we would choose the top $\tau$ labels with the highest confidence of being positive. However, given that most deep multi-label models use independent predictive heads (i.e., a separate sigmoid activation per output label), naively comparing these values may not lead to the best performance. Thus, no clear notion of what the desired subset is in the case that $>\tau$ candidates are present.

For the Pascal VOC dataset, we set the same threshold of $0.5$ probability per label, so the prediction heads for each label are aligned. For the CheXpert dataset, different probability thresholds are set per label, so the problem of deciding which of the labels should remain positive (if there are more positive labels than the max $\tau$ of positive labels) is difficult to resolve. This would likely require some form of domain knowledge.

For the original train set of the Pascal VOC dataset (with 5823 samples), the average number of positive labels per example is 1.52, with 20 total labels. More detailed statistics on the train set are presented in Table~\ref{tab:pascal-voc-positive-labels-per-data-point}.

When the $\tau$-voting is applied with the fixed privacy budget, we are able to answer more queries, while the performance metrics (e.g., accuracy) remain comparable. 
When the initial number of labels is set to $k=5$, for $\tau=2$ the number of classes is 16 instead of 32 ($\tau=5$, which is equivalent to no clipping) and we are able to answer 5\% more queries, for $\tau=1$, we have only 6 classes and almost 14\% more answered queries.
For $\tau=5$ there is only $1$ more class than for $\tau=4$ and we observe a very small difference in the number of answered queries (namely 8). The number of answered queries using the Binary method is still higher than with Powerset even when $\tau=1$ (the number of classes is $C=k+1$).


\begin{table}[h]
\caption{Statistics about the positive labels per data point in the Pascal VOC dataset.}
\label{tab:pascal-voc-positive-labels-per-data-point}
\begin{center}
    \begin{small}
    \begin{sc}
    \begin{tabular}{|c|c|}
    \hline
    \textbf{\# of positive labels} & \textbf{\# of train samples} \\
    \hline
    6 & 1 \\
    5 & 17 \\
    4 & 105 \\
    3 & 484 \\
    2 & 1677 \\
    1 & 3539 \\
    0 & 0 \\
    \toprule
    \end{tabular}
    \end{sc}
    \end{small}
\end{center}
\end{table}

\begin{table}[h]
\caption{\textbf{Parameter $\delta$ vs the number of answered queries.} We show how the number of answered queries decreases gradually with orders of magnitude lower $\delta$ values for the PascalVOC dataset with $\varepsilon=16,\sigma_{GNMax}=9$.}
\label{tab:pascal-voc-delta-values}
\begin{center}
    \begin{small}
    \begin{sc}
    \begin{tabular}{|c|c|}
    \hline
    \textbf{$\delta$ value for $(\varepsilon,\delta)$-DP} & \textbf{\# of answered queries} \\
    \hline
    $10^{-5}$ & 170 \\
    $10^{-6}$ & 158 \\
    $10^{-7}$ & 151 \\
    $10^{-8}$ & 139 \\
    \toprule
    \end{tabular}
    \end{sc}
    \end{small}
\end{center}
\end{table}

We present performance of Powerset PATE with $\tau$-clipping w.r.t. number of answered queries in Table~\ref{tab:powerset-tau-pate-pascal}.


\begin{table*}[t]
\caption{\textbf{Pascal VOC: Performance of Powerset PATE with $\tau$-clipping} w.r.t. number of answered queries with the specified $\sigma_{\text{GNMax}}$. The ACC, BAC, AUC, and MAP are measured on the answered queries from the test set. PB ($\varepsilon$) is the privacy budget. We use 50 teacher models. When we limit the number of labels per dataset, we select the first $k$ labels. The $\tau$ denotes a maximum number of positive labels that a teacher is allowed to return per data sample.}
\label{tab:powerset-tau-pate-pascal}
\vskip -0.3in
\begin{center}
\begin{small}
\begin{sc}
\begin{tabular}{ccccccccc}
\toprule 
\shortstack{\# of \\ Labels} &  $\tau$ & \shortstack{Queries \\ answered} & PB ($\varepsilon$) & $\sigma_{\text{GNMax}}$ & ACC &                  BAC &          AUC &          MAP \\
\hline
1 & 1 & \textbf{5464} & 20 & 2 & .98 & .84 & .84 & .75 \\
2 & 1 & \textbf{5464} & 20 & 5 & .97 & .75 & .75 & .55 \\
2 & 2 & \textbf{5464} & 20 & 5 & .97 & .74 & .74 & .55 \\
3 & 1 & \textbf{3437} & 20 & 7 & .97 & .71 & .71 & .48 \\
3 & 2 & \textbf{3421} & 20 & 7 & .97 & .71 & .71 & .49 \\
3 & 3 & \textbf{3416} & 20 & 7 & .97 & .72 & .72 & .51 \\
4 & 1 & \textbf{2547} & 20 & 7 & .97 & .68 & .68 & .43 \\
4 & 2 & \textbf{2527} & 20 & 7 & .97 & .69 & .69 & .43 \\
4 & 3 & \textbf{2485} & 20 & 7 & .97 & .69 & .69 & .43 \\
4 & 4 & \textbf{2485} & 20 & 7 & .97 & .69 & .69 & .43 \\
\hline
5 & 1 & \textbf{2321} & 20 & 8 & .96 & .65 & .65 & .36 \\
5 & 2 & \textbf{2151} & 20 & 8 & .96 & .65 & .65 & .34 \\
5 & 3 & \textbf{2077} & 20 & 8 & .96 & .65 & .65 & .34 \\
5 & 4 & \textbf{2048} & 20 & 8 & .96 & .65 & .65 & .33 \\
5 & 5 & \textbf{2040} & 20 & 8 & .96 & .65 & .65 & .33 \\
\hline
8 & 8 & \textbf{1192} & 20 & 7 & .95 & .68 & .68 & .37 \\
10 & 1 & \textbf{947} & 20 & 5 & .97 & .65 & .65 & .36 \\
10 & 2 & \textbf{896} & 20 & 5 & .96 & .65 & .65 & .36 \\
10 & 3 & \textbf{895} & 20 & 5 & .96 & .65 & .65 & .33 \\
10 & 4 & \textbf{885} & 20 & 5 & .96 & .65 & .65 & .33 \\
10 & 5 & \textbf{877} & 20 & 5 & .96 & .65 & .65 & .33 \\
10 & 6 & \textbf{861} & 20 & 5 & .96 & .64 & .64 & .30 \\
10 & 7 & \textbf{849} & 20 & 5 & .96 & .64 & .64 & .29 \\
10 & 8 & \textbf{848} & 20 & 5 & .96 & .66 & .66 & .34 \\
10 & 9 & \textbf{848} & 20 & 5 & .96 & .64 & .64 & .32 \\
10 & 10 & \textbf{848} & 20 & 5 & .96 & .64 & .64 & .30 \\
\hline
11 & 11 & \textbf{702} & 20 & 4 & .95 & .66 & .66 & .37 \\
14 & 14 & \textbf{417} & 20 & 3 & .95 & .65 & .65 & .37 \\
15 & 15 & \textbf{165} & 20 & 3 & .94 & .66 & .66 & .36 \\
16 & 16 & \textbf{94} & 20 & 2 & .95 & .65 & .65 & .38 \\
18 & 18 & \textbf{94} & 20 & 2 & .95 & .65 & .65 & .36 \\
\hline
20 & 1 & \textbf{111} & 20 & 2 & .95 & .63 & .63 & .33 \\
20 & 2 & \textbf{99} & 20 & 2 & .96 & .67 & .67 & .41 \\
20 & 3 & \textbf{81} & 20 & 2 & .96 & .61 & .61 & .32 \\
20 & 4 & \textbf{78} & 20 & 2 & .96 & .63 & .63 & .33 \\
20 & 5 & \textbf{78} & 20 & 2 & .96 & .63 & .63 & .34 \\
20 & 10 & \textbf{78} & 20 & 2 & .95 & .65 & .65 & .34 \\
20 & 20 & \textbf{78} & 20 & 2 & .95 & .65 & .65 & .33 \\
\bottomrule
\end{tabular}
\end{sc}
\end{small}
\end{center}
\vskip -0.2in
\end{table*}

\FloatBarrier
\subsection{DPSGD vs PATE on CheXpert: setup}
\label{sec:dpsgd-pate-chexpert}

The Adaptive DPSGD method from~\cite{AdaptiveDPSGD2021} does not publish the code, thus the following is our best effort and the results for Adaptive DPSGD come from Figure 5 in~\cite{AdaptiveDPSGD2021}.
We use the following parameters to obtain our results for the Binary multi-label PATE:
\begin{verbatim}
sigma gnmax = 7.0
sigma threshold = 0
threshold = 0
method = multilabel
batch size = 20
learning rate = 0.001
epochs = 100
weight decay = 0
X-ray views = ['AP', 'PA']
Multilabel_prob_threshold = [
0.53, 0.5, 0.18, 0.56, 0.56]
\end{verbatim}


\subsubsection{CDF of Gaps}
We show the Cumulative Distribution Function (CDF) 
for the gaps in Figure~\ref{fig:cdf-gaps}. We observe that the ensemble of teachers is confident about the answers for most labels for the Binary PATE and there are very small gaps for most queries when using Powerset PATE, which shows much lower confidence of teachers in choosing the same (super) classes, which are created from aggregated label predictions (the values of the labels are collected in a binary vector that constitutes a class).

\begin{figure}[!t]
\begin{center}
\centerline{
\begin{tabular}{cc}
\includegraphics[width=0.5\columnwidth]{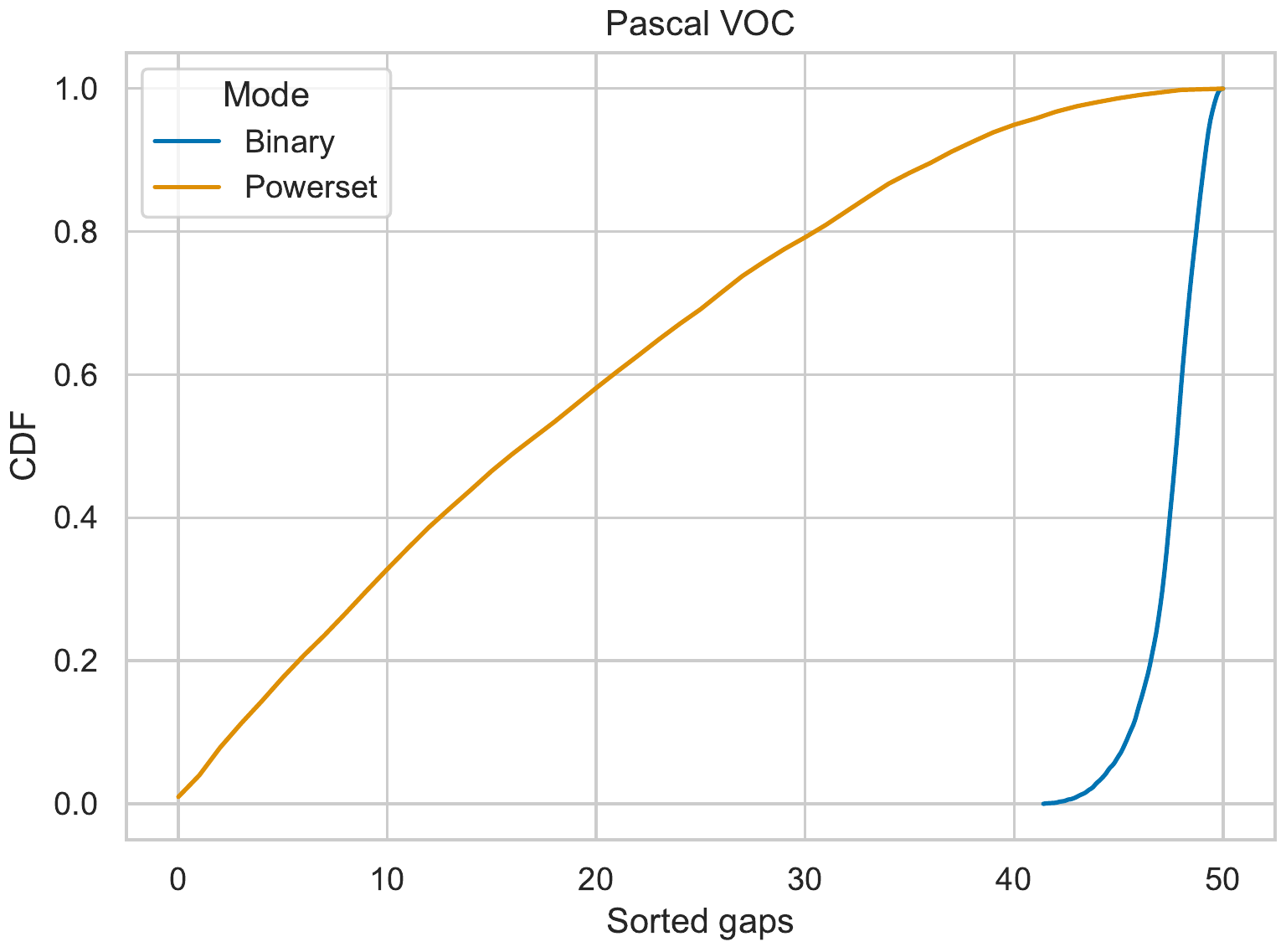} &
\includegraphics[width=0.5\columnwidth]{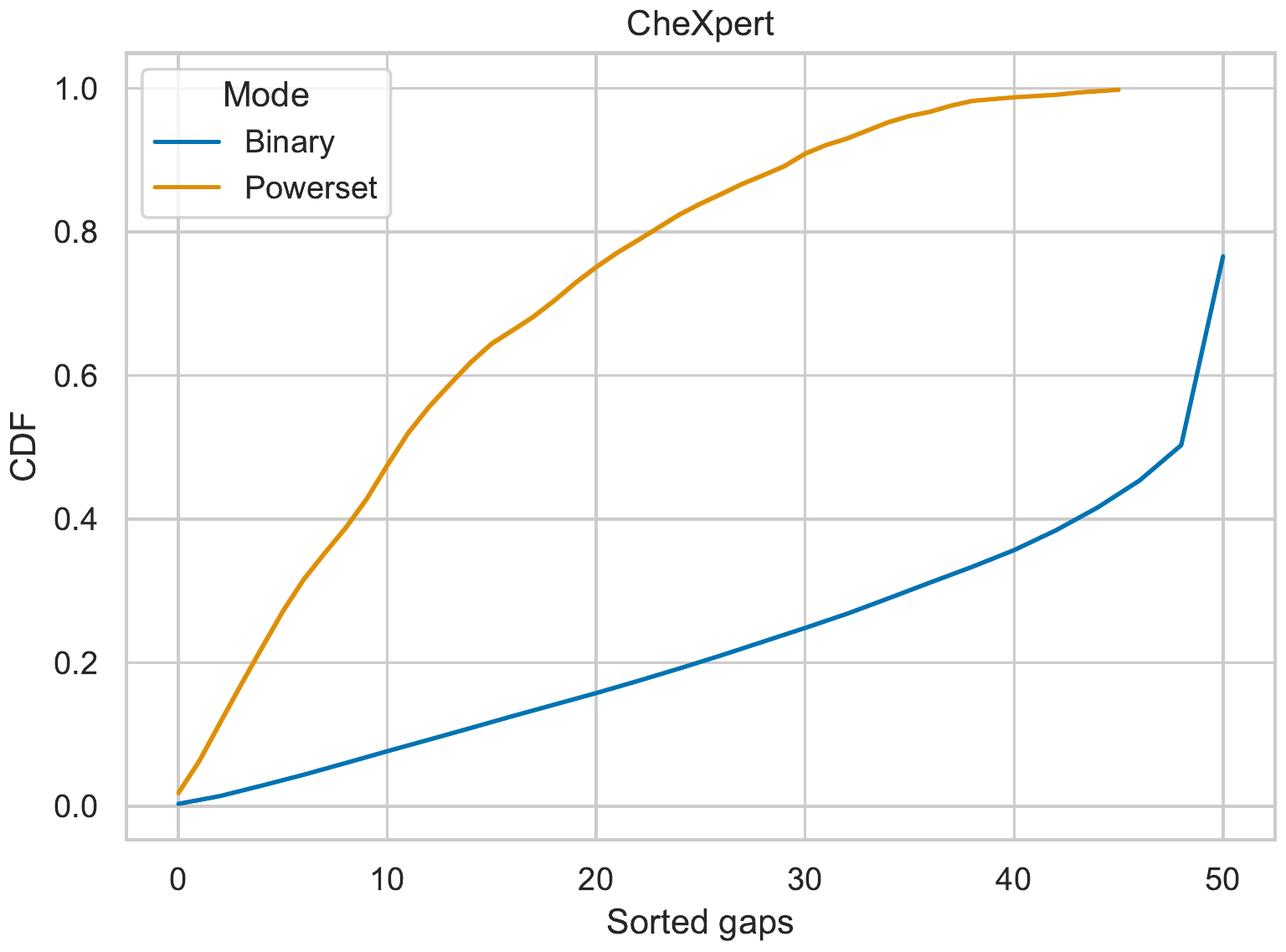} \\
\end{tabular}

}
\caption{\textbf{Binary vs Powerset PATE: CDF of gaps} (differences between vote counts). We use 50 teacher models trained on the Pascal VOC and CheXpert datasets. These are raw gaps without adding any noise to the vote histograms. Most of the gaps are relatively large ($>40$) for the Binary PATE per label, which shows that teachers are confident about the answers to queries. The average gap for the Powerset method is relatively low (only $18$ for Pascal VOC  and $13$ for CheXpert).}    
\label{fig:cdf-gaps}
\end{center}
\vskip 0.1in
\end{figure}

\subsubsection{Gaps and Performance Metrics}
With more labels for the Powerset method, we have more possible classes and the votes become more spread out. This can be measured by the gap, which is the difference between the maximum number of votes per class and the runner-up (the number of votes for the next class with the highest number of votes). In Figures~\ref{fig:powerset-binary-gaps-metrics-num-labels-pascal} and~\ref{fig:powerset-binary-gaps-metrics-num-labels-cxpert}, we plot the average gap (on the y-axis) for the test set across all histograms for a given number of labels (presented on the x-axis). We compare the Powerset PATE (denoted as Powerset) vs the Binary PATE per label (denoted as Binary). 
The average gap between votes is comparable for different number of labels of Binary PATE since this method considers each label separately and there are always only two classes. The average gap between votes decreases very fast for Powerset because of the exponential growth of number of classes with more labels considered. Additionally, the performance metrics: acc (accuracy), bac (balanced accuracy), area under the curve (auc), and mean average precision (map), are also higher for the Binary than Powerset method in case of CheXpert dataset and comparable for Pascal VOC.

\begin{figure*}[!t]
\begin{center}
\centerline{
\includegraphics[width=1.0\linewidth]{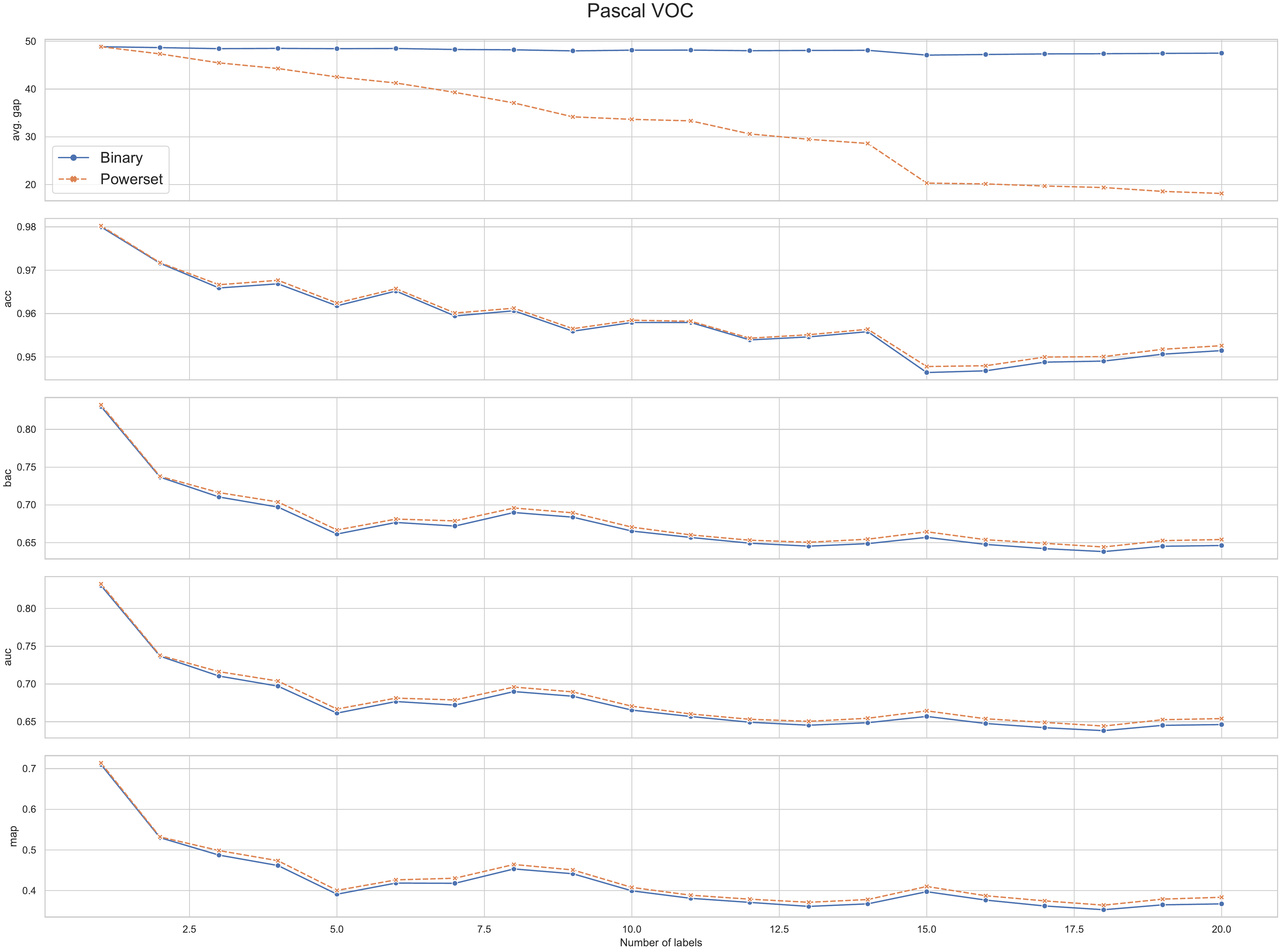}
}
\caption{\textbf{Binary vs Powerset PATE: gaps and metrics for Pascal VOC}. We use 50 teacher models. These are raw gaps without adding any noise to the vote histograms. 
}
\label{fig:powerset-binary-gaps-metrics-num-labels-pascal}
\end{center}
\vskip 0.1in
\end{figure*}

\begin{figure*}[!t]
\begin{center}
\centerline{
\includegraphics[width=1.0\linewidth]{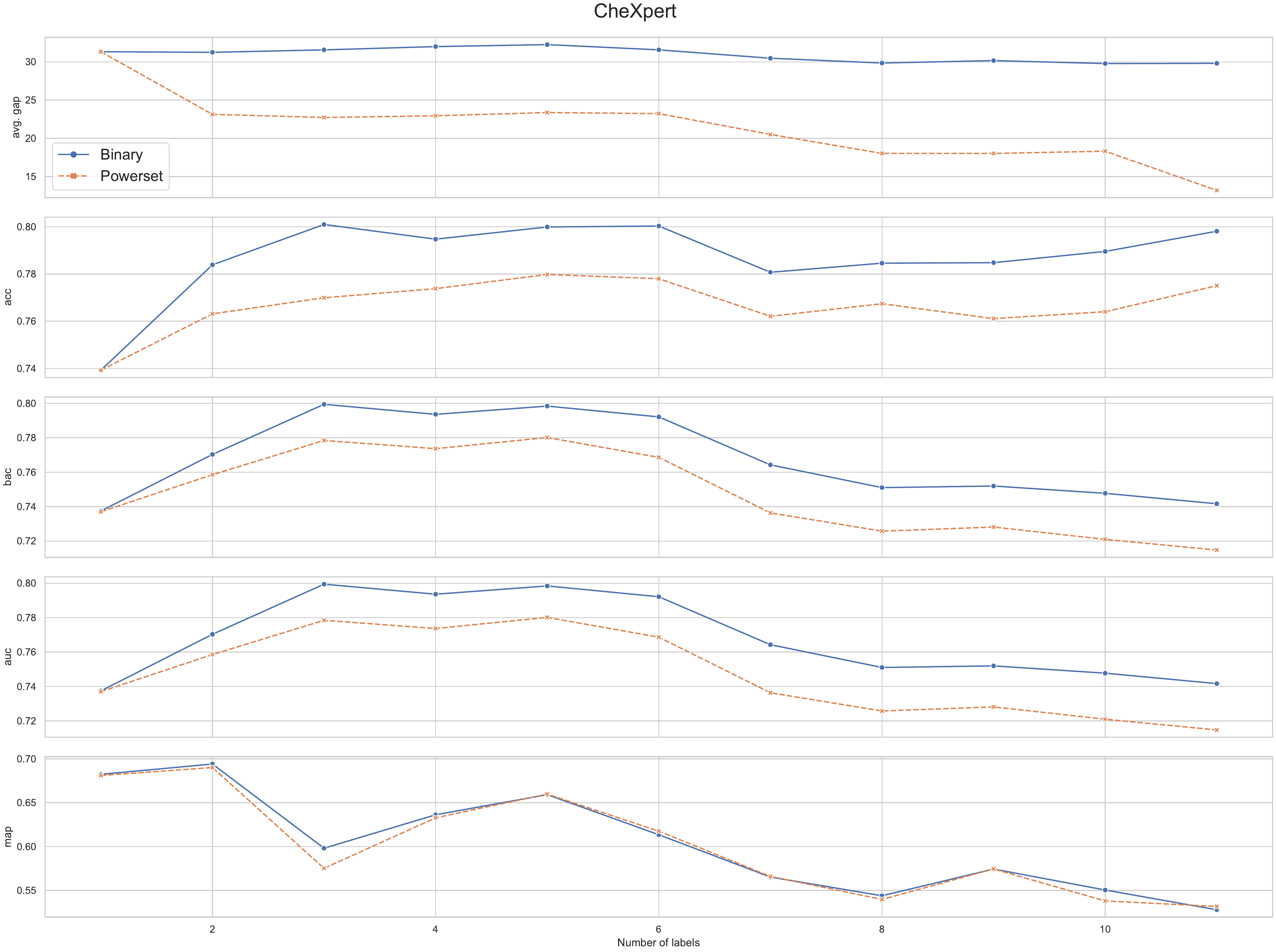}
}
\caption{\textbf{Binary vs Powerset PATE: gaps and metrics for CheXpert}. We use 50 teacher models. These are raw gaps without adding any noise to the vote histograms.}
\label{fig:powerset-binary-gaps-metrics-num-labels-cxpert}
\end{center}
\vskip 0.1in
\end{figure*}

\subsection{General Tuning of Hyper-Parameters}

We set the hyper-parameters by considering the utility and privacy of our proposed method. Based on Lemma~\ref{lemma:rdp-gaussian-multilabel}, we minimize the privacy budget by maximizing the $\sigma$ parameter of the Gaussian noise and minimizing the $\tau$ parameters. In Figure 1, we tune the value of the $\sigma$ parameter of the Gaussian noise based on the utility metrics: accuracy (ACC), balanced accuracy (BAC), mean average precision (MAP). We set $\sigma$ as 9, 10, and 7 for Pascal VOC, MIMIC, and CheXpert datasets, respectively. This allows us to answer many queries with high performance in terms of ACC, BAC, and MAP.

Similarly, we tune the $\tau$ values for $\ell_2$ norm in Figure~\ref{fig:app-tune-tau2} and for $\ell_1$ norm in Figure~\ref{fig:app-tune-tau1}.

Regarding the values of parameters $\sigma_G$, $T$, and $\sigma_T$, we follow the original work on PATE~\cite{papernot2017semi} and perform a grid search over a range of plausible values for each of the hyper-parameters.

The detailed analysis of tuning the hyper-parameters are for: $\tau$-s in Section~\ref{app:tuning-tau}, $\sigma_G$ in Section~\ref{app:tuning-sigma}, PATE’s thresholding ($T$, and $\sigma_T$) in Section~\ref{sec:tune-pate}, and probability thresholds in Section~\ref{sec:prob-threshold}.

\subsection{Tuning $\tau$-clipping}\label{app:tuning-tau}
We determine the minimum value of $\tau$ for clipping in $\ell_2$ and $\ell_1$ norms for each dataset so that the performance as measured by accuracy, BAC, AUC, and MAP, is preserved. We present the analysis in Figures~\ref{fig:app-tune-tau2} and~\ref{fig:app-tune-tau1}.

\begin{figure}[h]
\begin{center}
\begin{tabular}{cc}
\includegraphics[width=0.48\columnwidth]{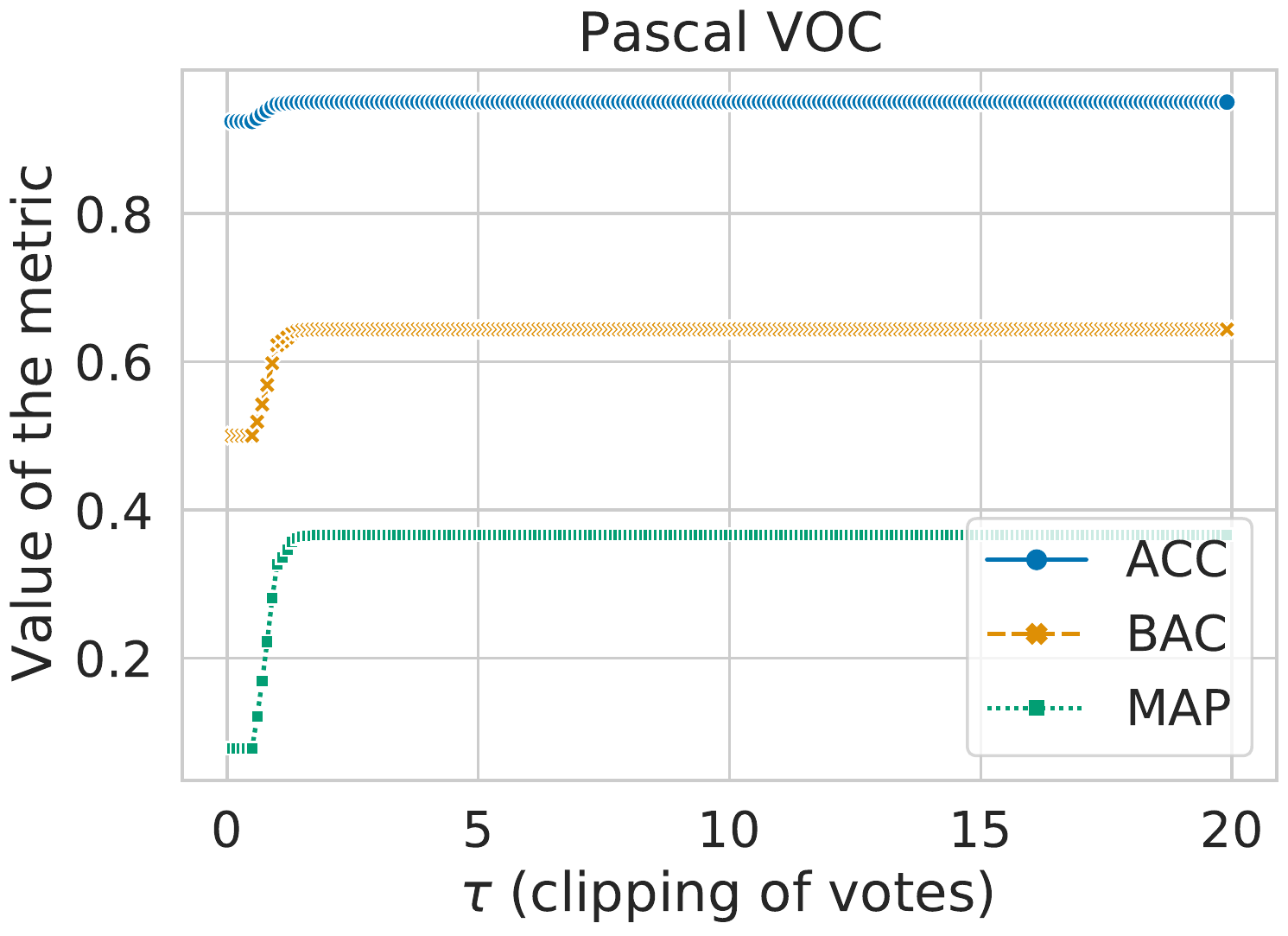} & \includegraphics[width=0.48\columnwidth]{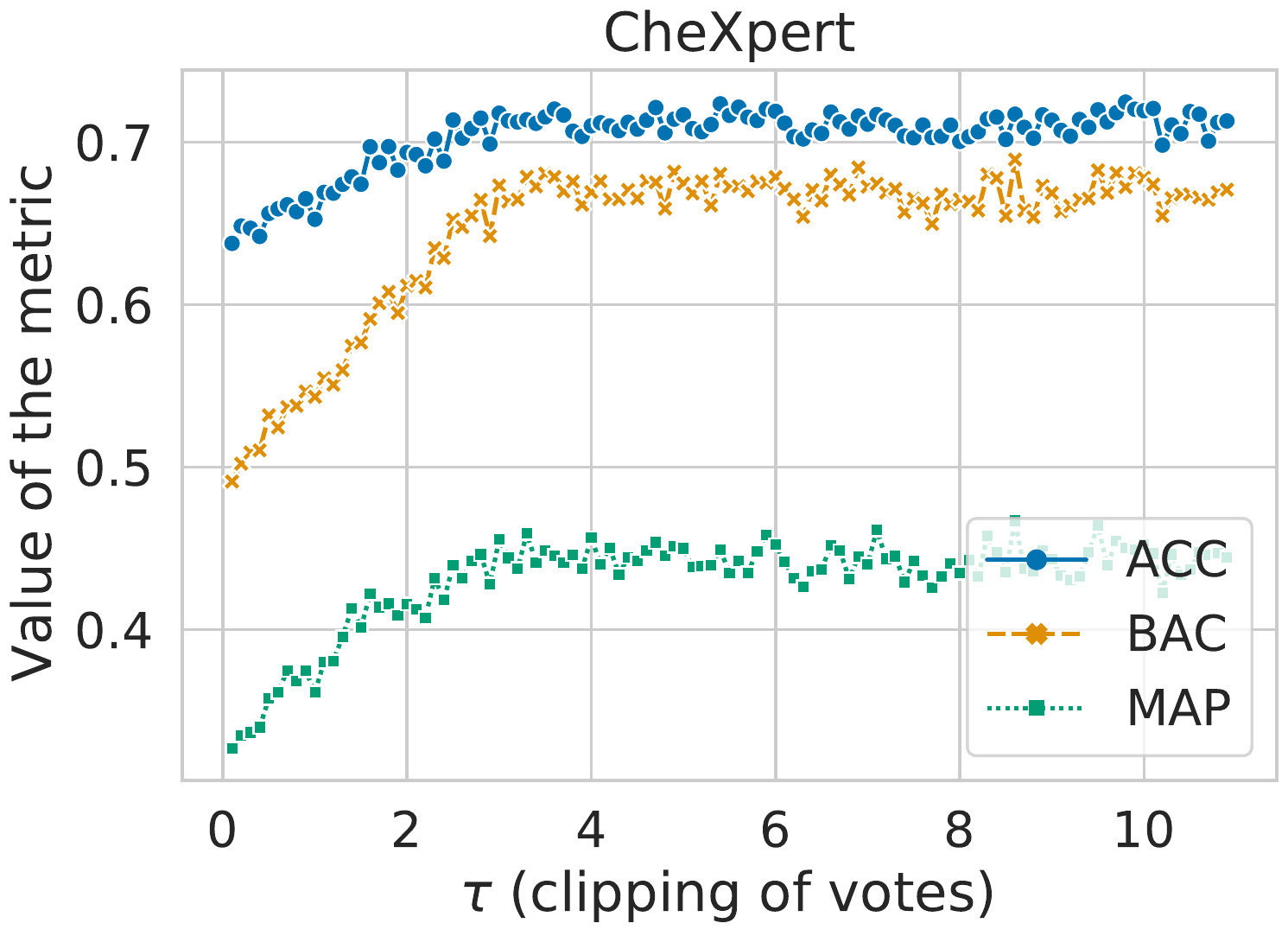} \\
  \includegraphics[width=0.48\columnwidth]{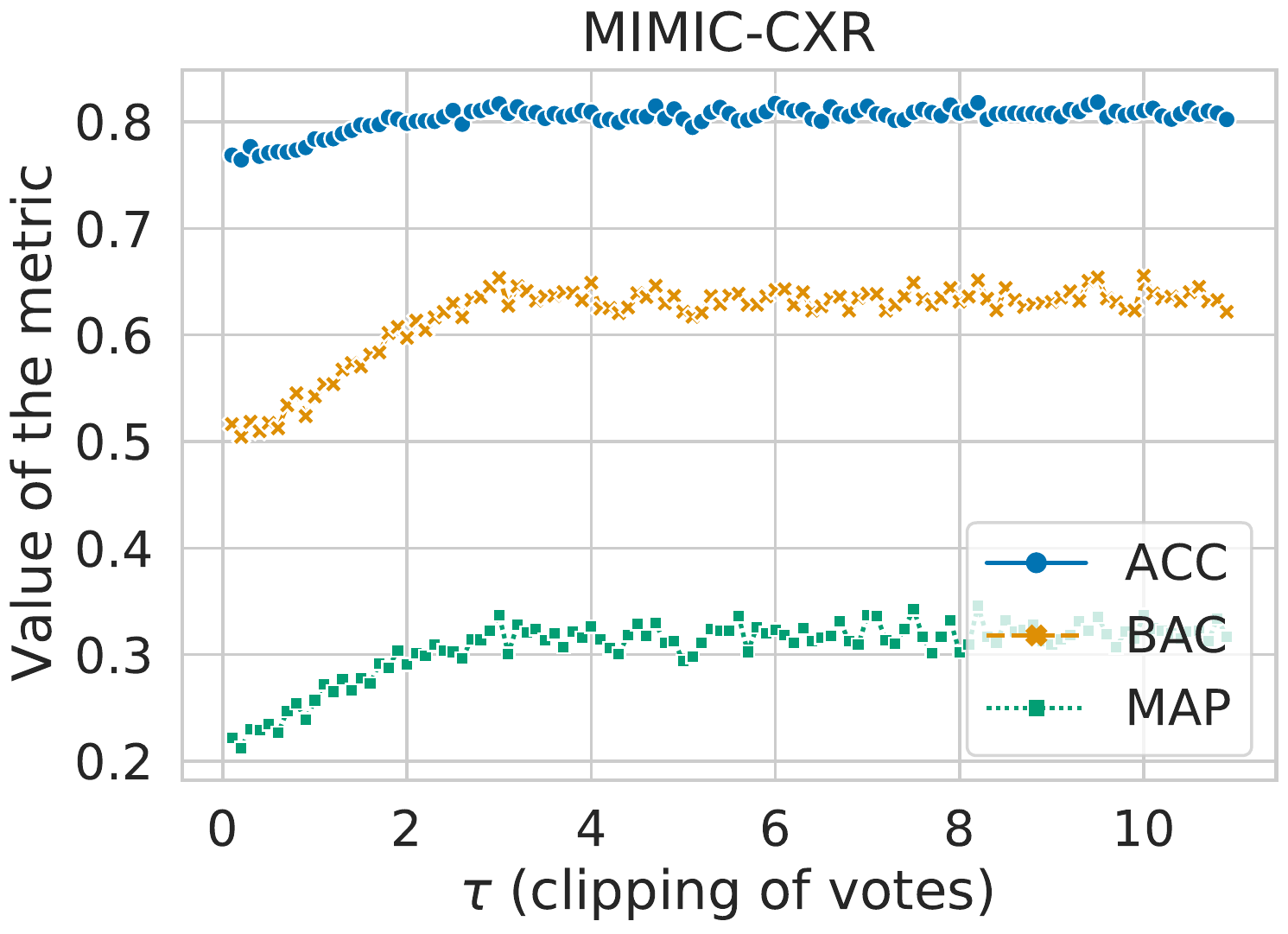}& \includegraphics[width=0.48\columnwidth]{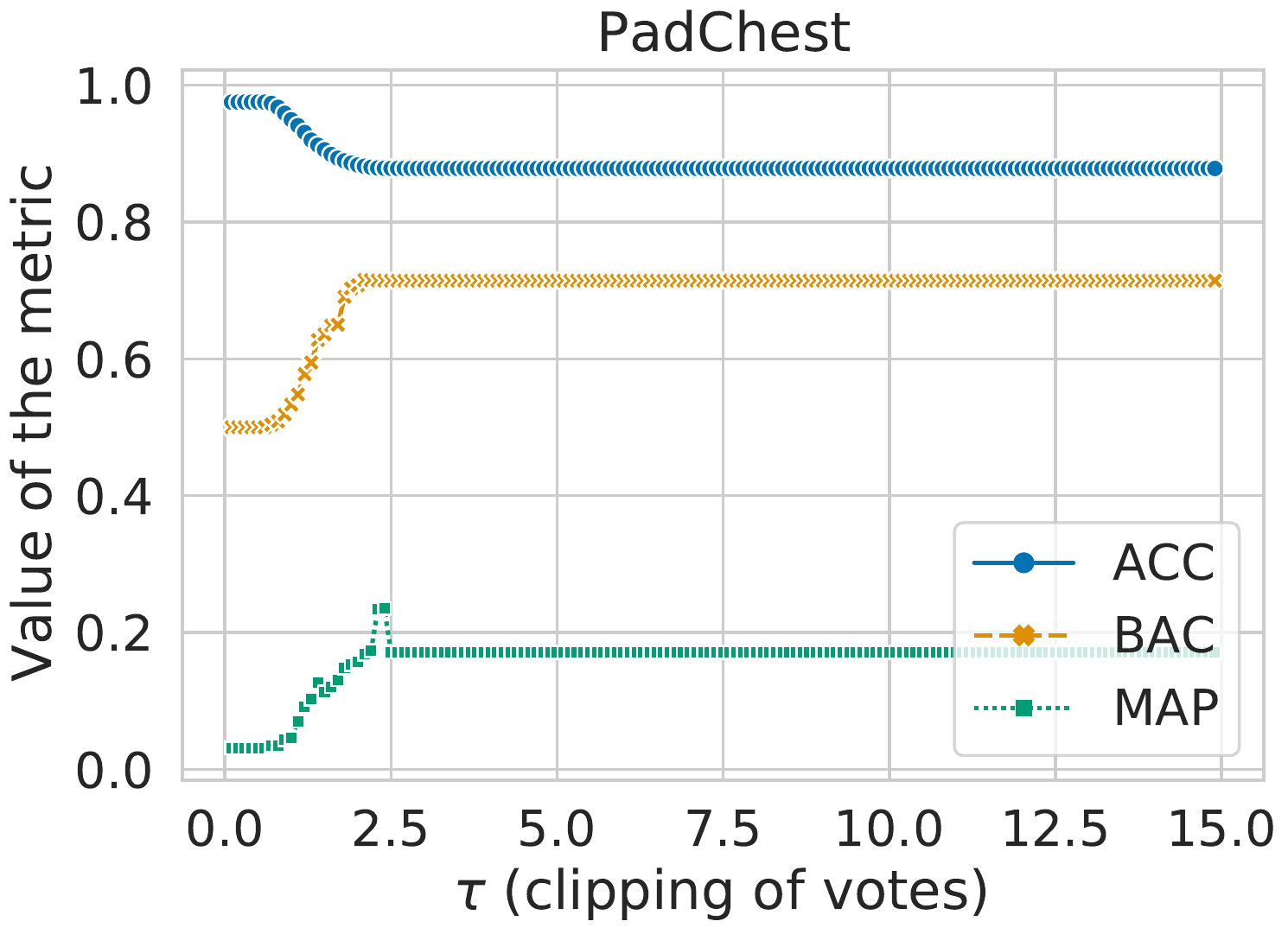} 
\end{tabular}
\caption{\textbf{Value of the metric vs $\tau$-clipping of votes in $\ell_2$ norm.} For a given $\tau$, we plot accuracy (ACC), balanced accuracy (BAC), and mean average precision (MAP).\label{fig:app-tune-tau2}
}
\end{center}
\end{figure}

\begin{figure}[!t]
\begin{center}
\begin{tabular}{cc}
\includegraphics[width=0.48\columnwidth]{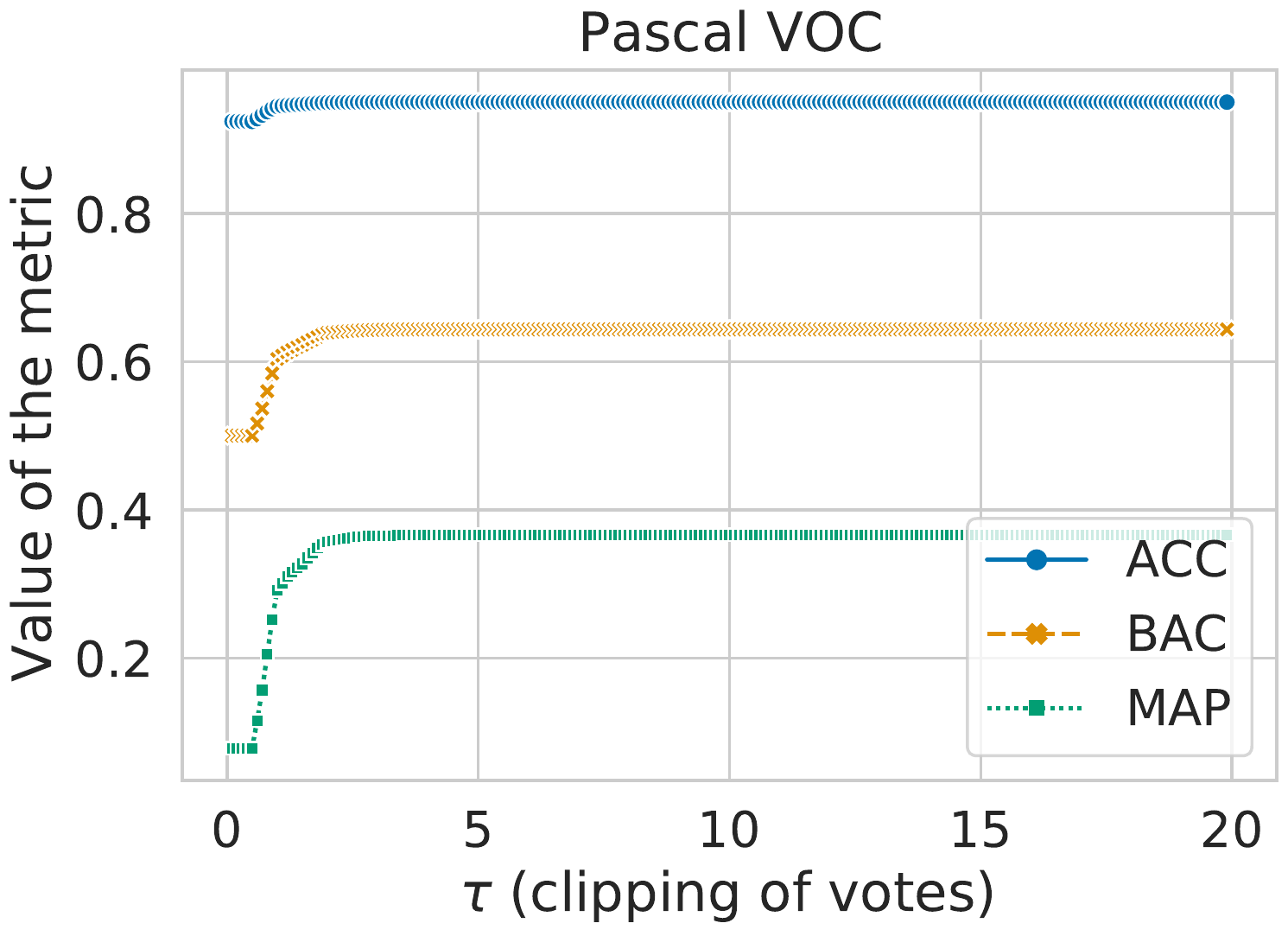} & \includegraphics[width=0.48\columnwidth]{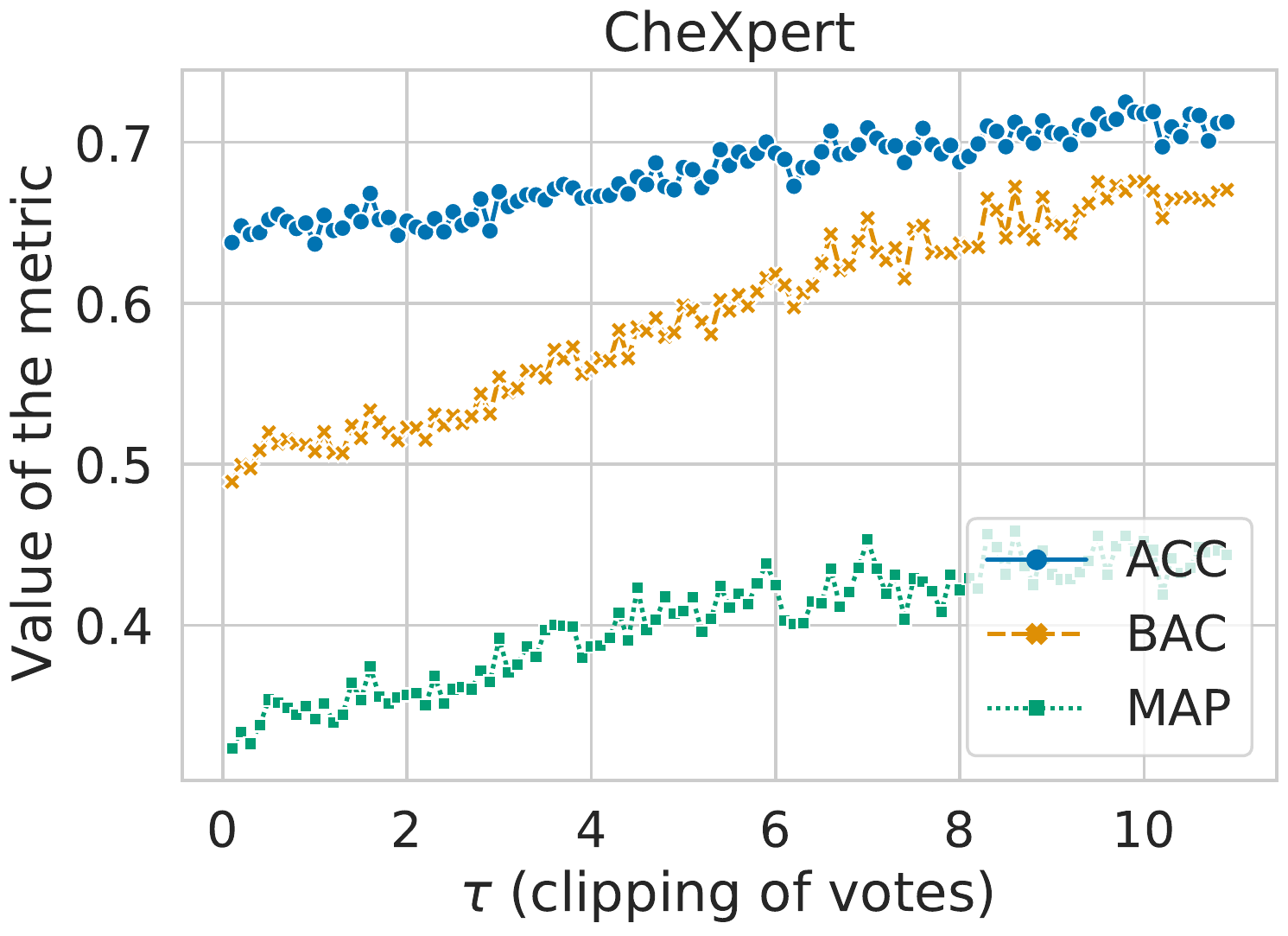} \\
\includegraphics[width=0.48\columnwidth]{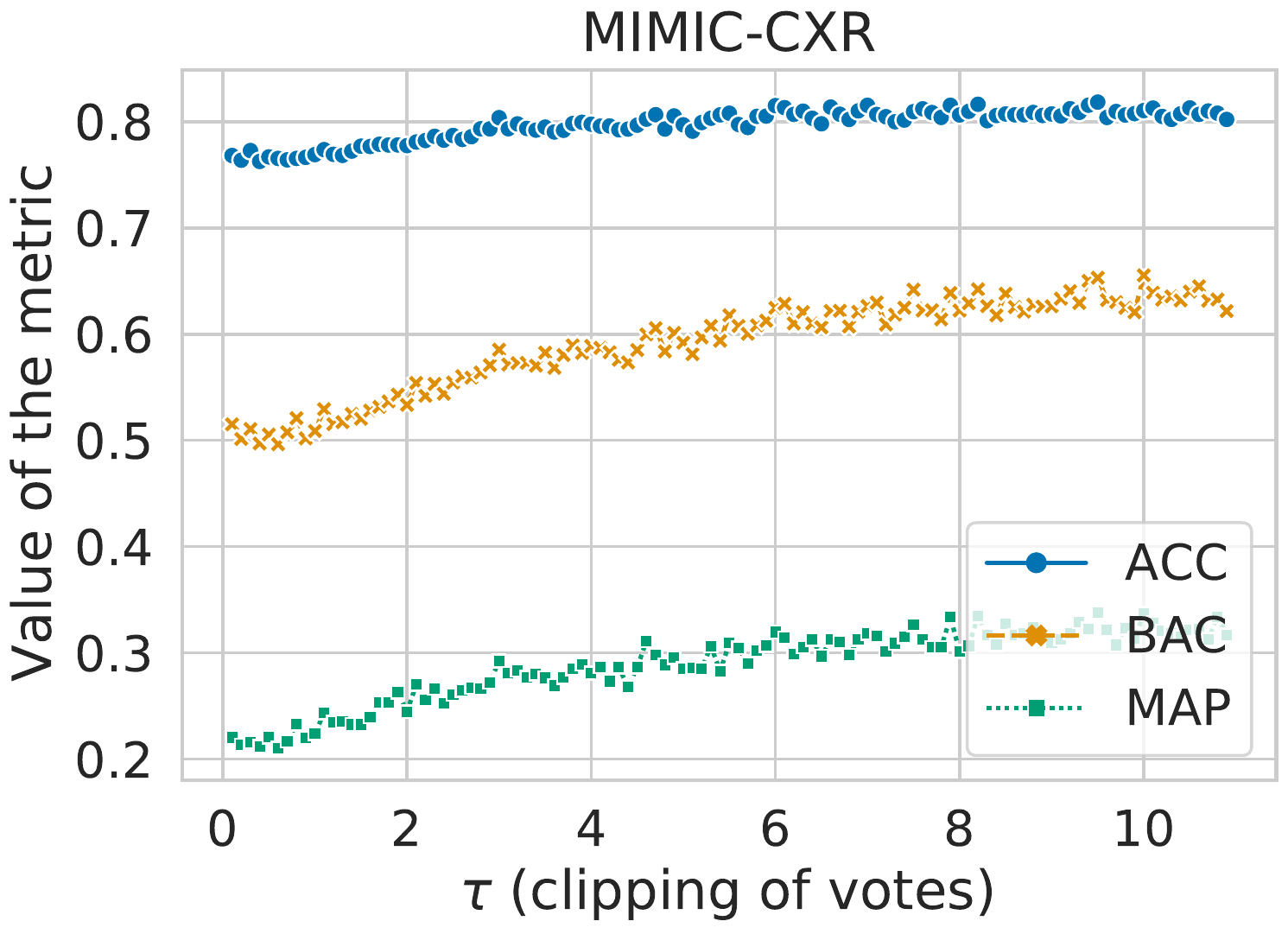} &
  \includegraphics[width=0.48\columnwidth]{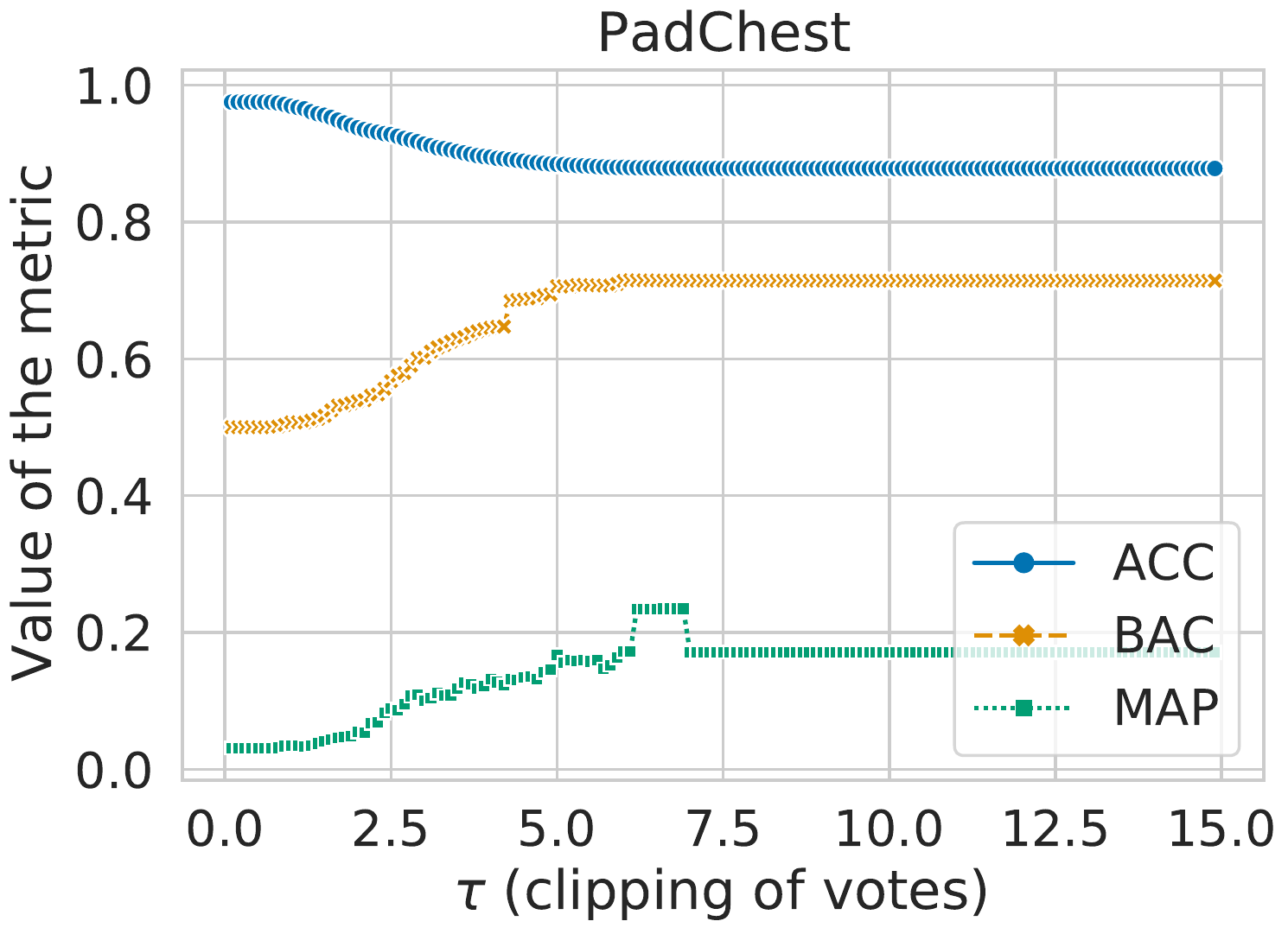}
\end{tabular}
\caption{\textbf{Value of the metric vs $\tau$-clipping of votes in $\ell_1$ norm.} For a given $\tau$, we plot accuracy (ACC), balanced accuracy (BAC), and mean average precision (MAP).\label{fig:app-tune-tau1}
}
\end{center}
\end{figure}

\subsection{Tuning $\sigma_{GNMax}$}
\label{app:tuning-sigma}
We determine how much of the privacy noise, expressed as the scale of Gaussian noise $\sigma$, can be added so that the performance as measured by accuracy, BAC, AUC, and MAP, is preserved. The experiment is run using the binary PATE per label and presented in Figure~\ref{fig:app-tune-sigma}.

\begin{figure*}[!t]
\begin{center}
\includegraphics[width=1.0\linewidth]{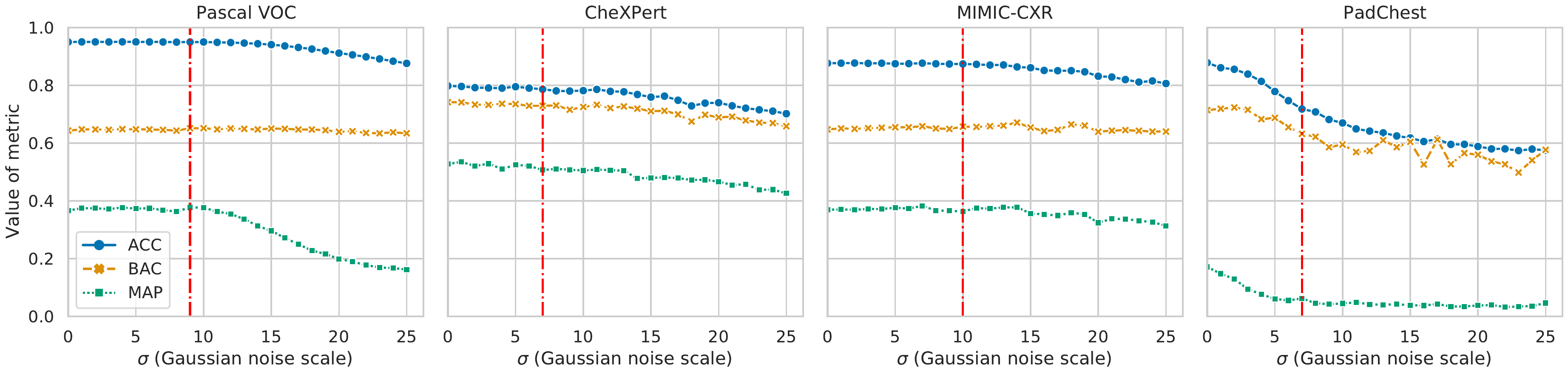}
\caption{\textbf{Value of the metric versus noise standard deviation $\sigma$.} For a given $\sigma$, we plot accuracy (ACC), balanced accuracy (BAC), and mean average precision (MAP).\label{fig:app-tune-sigma}
}
\end{center}
\end{figure*}

\subsection{Tuning PATE}
\label{sec:tune-pate}

There are three parameters to tune: the differential privacy Gaussian noise standard deviation $\sigma_{G}$, the count threshold $T$, and the thresholding Gaussian noise standard deviation $\sigma_{T}$. We first disable the thresholding ($\sigma_{T}$ and $T$) and tune $\sigma_{G}$ to achieve a high BAC of the noisy ensemble while maximizing the number of answered queries. 
In Figure~\ref{fig:ap-chexpert-ensemble-self-pate}
we tune the $\sigma_{G}$ parameter from PATE for different configurations of retraining. The goal is to maintain a high accuracy while selecting $\sigma_{G}$ with high value so that as many queries as possible are answered.
After tuning $\sigma_{G}$,
we grid-search $\sigma_{T}$ and $T$. Thresholding is useful if we want to add a relatively small amount of Gaussian noise ($\sigma_{G}$) in the noisy max. For example, for $\sigma_{G} = 7$ and $50$ teacher models trained on CheXpert, the maximum number of labels answered without thresholding is $737$ ($67$ queries) with a BAC above $0.69$. With thresholding, e.g., at $T=50$ and $\sigma_T=30$, we can answer an average of $756$ labels, up to $834$.
%
A well tuned threshold can also help reduce $\sigma_G$, which is the major influencing factor on the final BA of the noisy ensemble. Note that tuning the threshold benefits from a confident ensemble: here, we find the ensemble is often confident with a positive-negative vote difference of $40$ of max $49$. Also note that the BAC can be higher on easier queries.

\subsection{Compare Different Methods}
In Figure~\ref{fig:app-compare-methods-directly}, we directly compare the Binary PATE (denoted as PATE), \tpate, and clipping in $\ell_2$ and $\ell_1$ norms. The error regions are for the three querying parties.

\begin{figure}[!t]
\begin{center}
\includegraphics[width=1.0\linewidth]{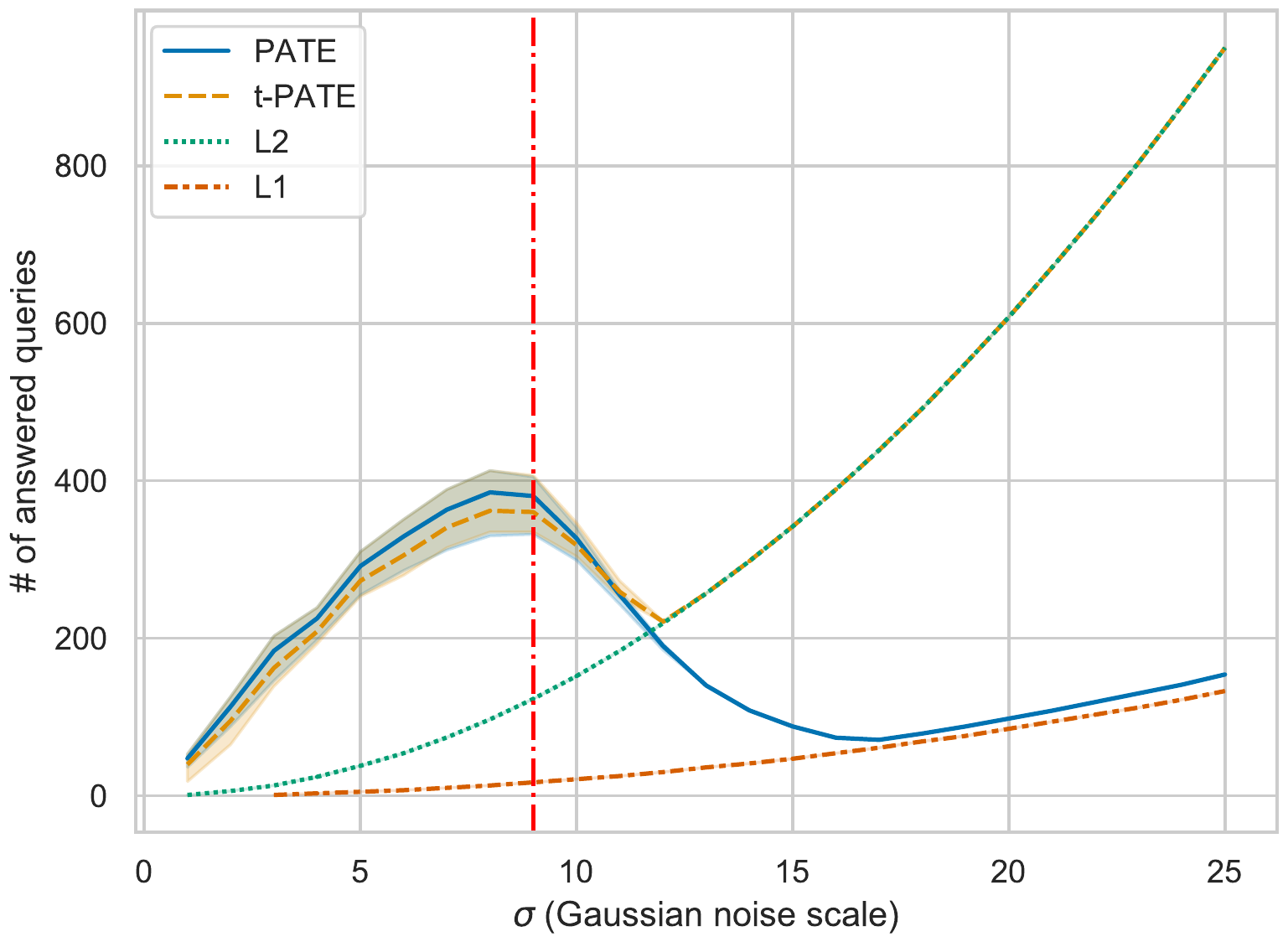}
\caption{\textbf{Compare methods: number of answered queries vs $\sigma$} - using the Pascal VOC dataset, with $\tau_1=3.4$ (set for $\ell_1-$norm clipping, and $\tau_2=1.8$ set for \tpate and $\ell_2-$norm clipping. The red vertical line denotes the selected value of $\sigma$.\label{fig:app-compare-methods-directly}
}
\end{center}
\end{figure}

\begin{table*}
\caption{\textbf{Performance of \multi CaPC} w.r.t. ACC, BAC, AUC, and MAP, on Pascal VOC (PA), CheXpert (CX), MIMIC (MC), and PadChest (PC). (-) denotes N/A. PB ($\varepsilon$) is the privacy budget.}
\label{tab:diff-epsilon}
\vskip -0.3in
\begin{center}
\begin{small}
\begin{sc}
\begin{tabular}{cclccccc}\toprule Dataset & \shortstack{\# of \\ Models} &  State & PB ($\varepsilon$) &          ACC &                  BAC &          AUC &          MAP \\
\midrule
\multirow{4}{*}{PA} & 1 &  Initial &                 - & .97 &  .85 & .97 & .85 \\
& 50 &  Before CaPC &  - & .93$\pm$.02 &          .59$\pm$.01 &  .88$\pm$.01 &  .54$\pm$.01 \\   & 50 &   After CaPC &                10 &  \textbf{.94$\pm$.01} &          .\textbf{62$\pm$.01} &  .88$\pm$.01 &  .54$\pm$.01 \\   & 50 &   After CaPC &                20 &  \textbf{.94$\pm$.01} &          \textbf{.64$\pm$.01} & \textbf{.89$\pm$.01} &  \textbf{.55$\pm$.01} \\\hline\multirow{3}{*}{CX} & 1 &      Initial &                 - &          .79 &                  .78 &          .86 &          .72 \\   & 50 &  Before CaPC &                 - &  .77$\pm$.06 &          .66$\pm$.02 &  .75$\pm$.02 &  .58$\pm$.02 \\   & 50 &   After CaPC &                20 &  .76$\pm$.07 &  \textbf{.69$\pm$.01} &  \textbf{.77$\pm$.01} &  \textbf{.59$\pm$.01} \\\hline\multirow{3}{*}{MC} & 1 &      Initial &                 - &         .90 &                  .74 &          .84 &          .51 \\   & 50 &  Before CaPC &                 - &  .84$\pm$.07 &          .63$\pm$.03 &  .78$\pm$.03 &  .43$\pm$.02 \\   & 50 &   After CaPC &                20 &  \textbf{.85$\pm$.05} &          \textbf{.64$\pm$.01} &  \textbf{.79$\pm$.01} &  \textbf{.45$\pm$.03 }\\\hline\multirow{3}{*}{PC} & 1 &      Initial &                 - &          .86 &                  .79 &          .90 &          .37 \\   & 10 &  Before CaPC &                 - &  .90$\pm$.01 &          .64$\pm$.01 &  .79$\pm$.01 &  .16$\pm$.01 \\   & 10 &   After CaPC &                20 &  .88$\pm$.01 &          .64$\pm$.01 &  .75$\pm$.01 &  .14$\pm$.01 \\
\bottomrule

\end{tabular}
\end{sc}
\end{small}
\end{center}
\end{table*}
\begin{table*}
\caption{\textbf{Performance of \multi CaPC with $\tau$-PATE} w.r.t. ACC, BAC, AUC, and MAP, on Pascal VOC (PA), CheXpert (CX), MIMIC (MC), and PadChest (PC). (-) denotes N/A. PB ($\varepsilon$) is the privacy budget. T (Y/N) in the table refers to the PATE thresholding i.e. corresponding to the parameters $T$ and $\sigma_T$ being used (Y) or not used (N). Note that the probability thresholding per layer was used in these experiments (as described in Section~\ref{sec:prob-threshold}).}
\label{tab:diff-epsilon}
\vskip -0.3in
\begin{center}
\begin{small}
\addtolength{\tabcolsep}{-3.5pt}
\begin{sc}
\begin{tabular}{cclcccccc}\toprule Dataset & \shortstack{\# of \\ Models} &  State & T(Y/N) & PB ($\varepsilon$) &          ACC &                  BAC &          AUC &          MAP \\
\midrule
\multirow{3}{*}{PA} & 1 &      Initial &                 - &            - &                  .88 &          .97 &          .91 \\
& 50 &  Before CaPC &  - & - & .93$\pm$.01 &          .59$\pm$.01 &  .89$\pm$.01 &  .54$\pm$.02 \\ 
& 50 &  After CaPC &  N & 20 & \textbf{.94$\pm$.01} &          \textbf{.60$\pm$.01} &  .89$\pm$.01 &  .54$\pm$.02 \\ 
\hline
\multirow{4}{*}{CX} & 1 &      Initial & - &                - &          .79 &                  .78 &          .86 &          .72 \\   & 50 &  Before CaPC &     - &            - &  .75$\pm$.02 &          .69$\pm$.01 &  .77$\pm$.01 &  .59$\pm$.01 \\   & 50 &   After CaPC &    Y    &        20 &  .73$\pm$.01 &  \textbf{.69$\pm$.01} &  .76$\pm$.01 &  \textbf{.59$\pm$.01} \\ & 50
& After CaPC &    N    &        20 &  .74$\pm$.01 &  \textbf{.70$\pm$.01} &  \textbf{.77$\pm$.01} &  \textbf{.59$\pm$.01}
\\\hline\multirow{3}{*}{MC} & 1 &      Initial &   - &              - &         .90 &                  .74 &          .84 &          .51 \\   & 50 &  Before CaPC &       - &          - &  .84$\pm$.07 &          .63$\pm$.03 &  .78$\pm$.03 &  .43$\pm$.02 \\   & 50 &   After CaPC &  Y &              20 &  .84$\pm$.02 &          \textbf{.64$\pm$.04} &  .77$\pm$.02 &  \textbf{.44$\pm$.01 }\\\hline\multirow{4}{*}{PC} & 1 &      Initial &        - &         - &          .86 &                  .79 &          .90 &          .37 \\   & 10 &  Before CaPC &   - &              - &  .82$\pm$.01 &          .64$\pm$.01 &  .79$\pm$.01 &  .17$\pm$.01 \\ & 10 &   After CaPC &         Y &       20 &  \textbf{.86$\pm$.04} &          .61$\pm$.01 &  .71$\pm$.03 &  .14$\pm$.02 \\   & 10 &   After CaPC &         N &       20 &  \textbf{.86$\pm$.02} &          .60$\pm$.01 &  .72$\pm$.02 &  .15$\pm$.03 \\
\bottomrule

\end{tabular}
\end{sc}
\end{small}
\end{center}
\end{table*}

\begin{figure}[ht]
\vskip 0.1in
\begin{center}
\centerline{\includegraphics[width=\columnwidth]{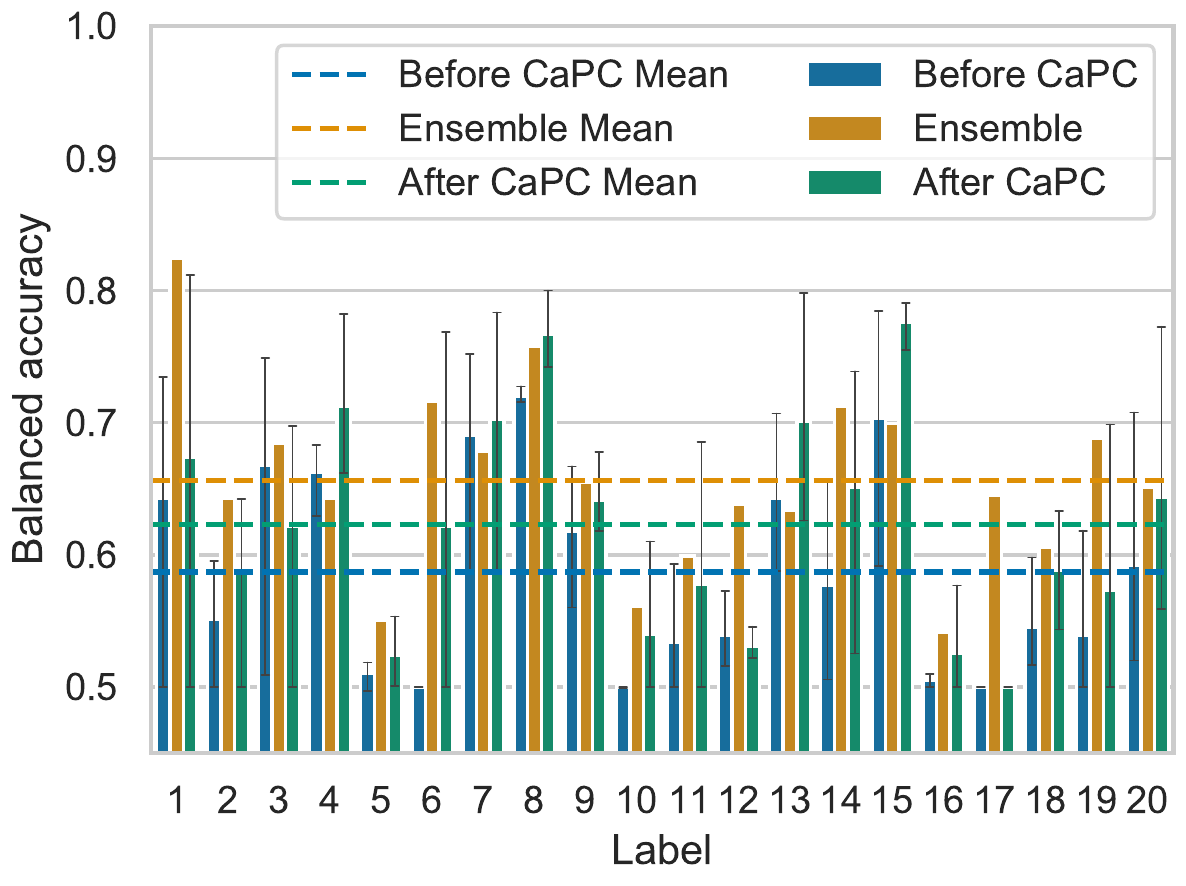}}
\caption{\textbf{Retraining with privacy budget $\varepsilon=10$} for the Pascal VOC dataset.}
\label{fig:PascalVOC-eps10}
\end{center}
\end{figure}

In Table~\ref{tab:diff-epsilon}, we show the performance of the retrained models when using \capc with various privacy budgets $\varepsilon$. In Figure~\ref{fig:PascalVOC-eps10} we show the detailed per label change in BAC for the retraining with privacy budget $\varepsilon=10$. We observe an increase in average BA by around $0.03$ after retraining when the privacy budget is set to $\varepsilon=10$. We present detailed analysis of the Binary PATE performance in Table~\ref{tab:binary-pate-vary-epsilon} on the Pascal VOC dataset when selecting the privacy budget $\varepsilon$ in the range from 1 to 20. The number of answered queries is proportional to the square of the privacy budget.

We also compare the performance after re-training with and without the (confidence) thresholding mechanism used in PATE (represented by $\sigma_{T}$ and threshold $T$ parameters). We observe that in case of the medical datasets, there can be a slightly higher increase of the metrics (accuracy, BAC, AUC, MAP), when we do not perform the thresholding. For example, for the CheXpert dataset, when thresholding is not used, the improvement is higher by about one percentage point across all metrics when compared to the option with thresholding. For PadChest, such increase is for AUC and MAP, for BAC we see a drop by one percentage point, and there is no difference in terms of accuracy (remains at the level of about $0.86$). This requires further investigation. Intuitively, we observe that the metrics can increase substantially with thresholding for the labels (pathologies) that are easy to classify. However, the most difficult to predict pathologies are left without answer and no improvement is made on them. On the other hand, without thresholding, we have to answer all labels and correct predictions on the hardest labels can produce substantial improvement in the classification of these hard labels. 

\subsection{Tuning Probability Threshold}
\label{sec:prob-threshold}
We tune the global (applied to all labels) probability threshold $\gamma$ that determines if a given probability denotes positive ($P > \gamma$), or a negative ($P \le \gamma$) vote. The best global $\gamma$ for the $50$ models trained on the CheXpert dataset is $0.32$ with the balanced accuracy of $0.702$ and the $\mu$ value of $0.45$, where $\mu$ is the average probability of predictions after applying the element-wise sigmoid function to the logits. 

We also tune the $\gamma$ threshold per label for CheXpert. We find the following probability thresholds per label using the validation set on the ensemble: $0.53, 0.5, 0.18, 0.56, 0.56, 0.21, 0.23, 0.46, 0.7, 0.2, 0.32$. Next, we select the 3 teacher models with the lowest BAC and retrain them with \capc. The average metrics BAC, AUC, and MAP before retraining are $0.63, 0.70, 0.53$, and after retraining, we observe a significant improvement to $0.68, 0.75, 0.58$, respectively. Thus, the value for each metric increases by around $0.05$. 
We find the biggest performance improvement for these weakest models after learning from better teachers and retraining. We present detailed results per label and for BA as well as AUC metrics in Figure~\ref{fig:retrain-cxpert-tau}.

\subsection{New Data Independent Bound}

We also proved a new data-independent bound that achieves tighter guarantees (Lemma~\ref{lemma:rdp-gaussian-multilabel}). The results for the method are displayed in Figure~\ref{fig:tpate} under the label \textit{L2-DI}, which denotes $\ell_2$-norm $\tau$ clipping of the ballots. \textit{L2-DI} can be compared with the Binary PATE data-dependent analysis, which is under the label Binary PATE. For the data sets like Pascal VOC, on which we achieve higher value of the metrics (Accuracy (ACC) 94\%, BAC 64\%, AUC 89\%, MAP 55\%), the data-dependent analysis allows us to release around 3.5x more queries (427 queries whereas the data-independent one only 123 queries, with all the parameters, such as amount of Gaussian noise added, being equal). On the other hand, for the data sets like MIMIC with lower performance on the metrics (Accuracy (ACC) 85\%, BAC 64\%, AUC 79\%, MAP 45\%), the difference is less pronounced and the data-dependent analysis allows us to release around 1.9x more queries (94 queries whereas the data-independent one releases 48 queries).

\subsection{Cross-Domain Retraining}

In real-world healthcare scenarios, there are often rare diseases that may be difficult to accurately model with machine learning. Even coalitions of hospitals sharing similar data distributions of the same domain may not see benefits, due to poor aggregate performance. However, cross-domain collaboration through \multi \capc can help improve performance in these cases. The hospital annotation discrepancies may pose barriers: here, we simulate this by the different X-ray image labels between PadChest and CheXpert (see Supplement Section~\ref{app:datasets}). 

To overcome this, we take the union of labels between the medical datasets and follow the experimental setup of~\cite{xrayCrossDomain2020}.
We observe poor performing models
on the PadChest dataset, with a low average performance of BAC $=0.57$. Because of this, the benefits of \multi \capc within this coalition of hospitals are limited, since the ensemble is only marginally better. However, if this group of hospitals collaborated with another from a different but related domain, here represented by CheXpert, they may be able to see additional benefits. The models trained on this dataset achieve a higher BA (particularly on the first $5$ shared pathologies as presented in Figure~\ref{fig:cross-domain}). Thus, the ensemble of all models engages in \multi \capc, and the models trained on CheXpert act as answering parties to provide labels for querying parties from PadChest. Using a $\sigma_G=9$, $T=50$, and $\sigma_{T}=30$, we observe a higher BAC on all of those $5$ pathologies, where we do not see a significant decrease on the other pathologies. Thus, \multi \capc provided benefits for the PadChest models in this cross-domain scenario.

\begin{figure}[t]
\begin{center}
\centerline{\includegraphics[width=1.0\linewidth]{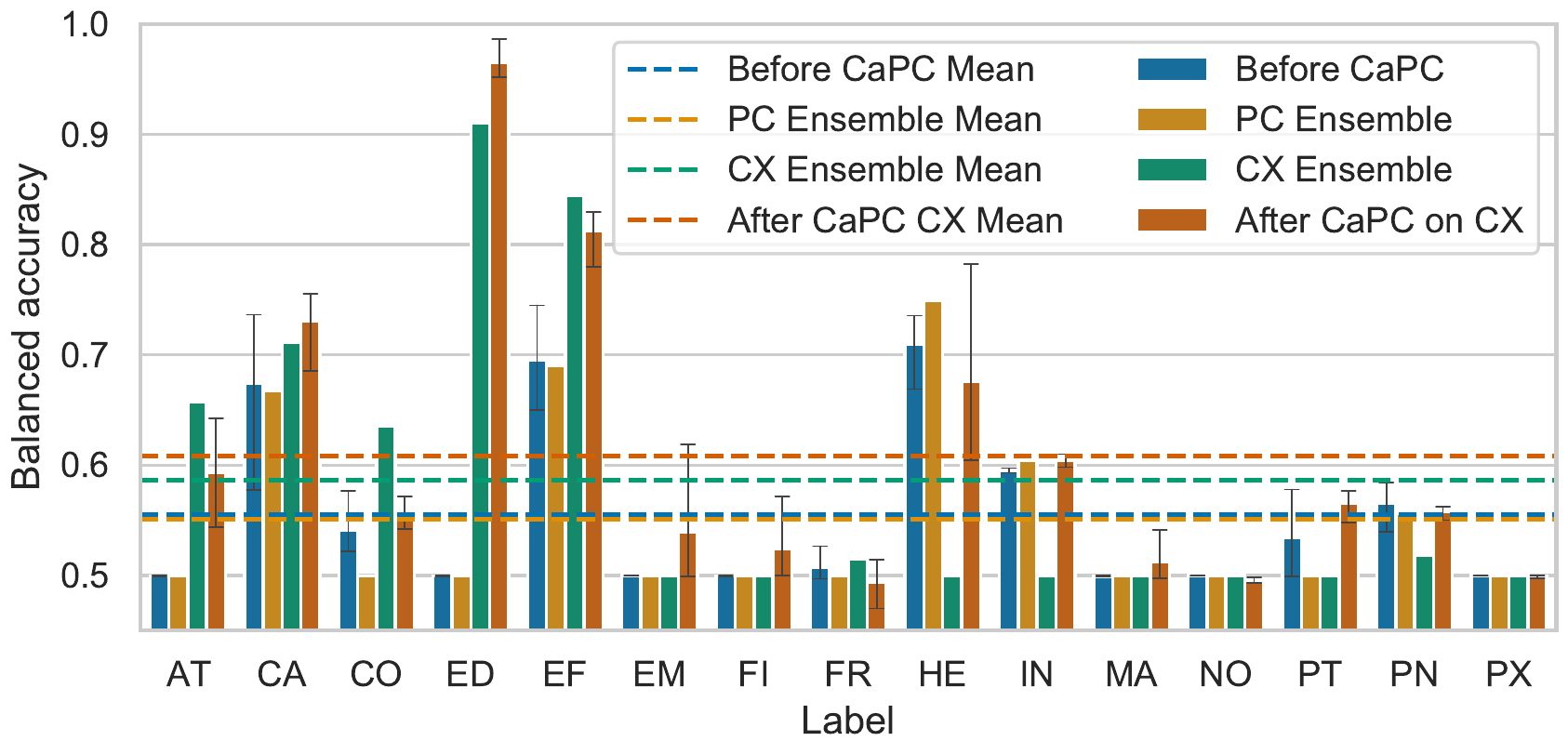}}
\caption{\textbf{Cross-domain retraining with  \capc}. We train 10 models on PadChest (PC) and compare their performance on PadChest test set against the ensemble of these models (PC Ensemble), the ensemble of 50 CheXpert models (CX Ensemble), and finally retrain the 10 PadChest models via binary \multi PATE using the CheXpert ensemble (after \capc on CX).\label{fig:cross-domain}}
\end{center}
\end{figure}


\begin{figure*}[ht]
\begin{center}
\centerline{\includegraphics[width=1.0\linewidth]{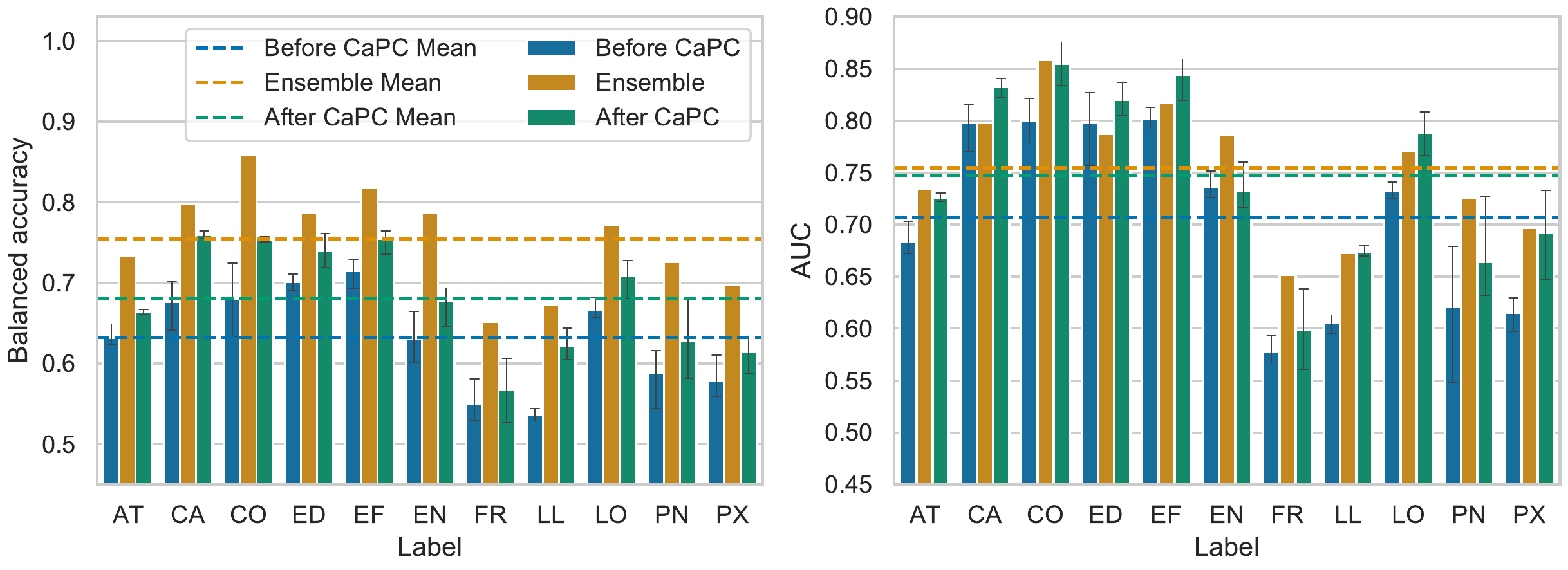}}
\caption{\textbf{Using \capc to improve the weakest models.} \emph{Dashed lines represent mean values of the metrics: Balanced Accuracy (BAC) and AUC.}\label{fig:retrain-cxpert-tau} We retrain a given model using additional CheXpert data labelled by the ensemble of all the other models trained on CheXpert. All metrics are improved after retraining by around $0.05$ on average. 
}
\end{center}
\end{figure*}

\begin{figure*}[h]
\begin{center}

\begin{tabular}{ccc}
\includegraphics[width=0.31\linewidth]{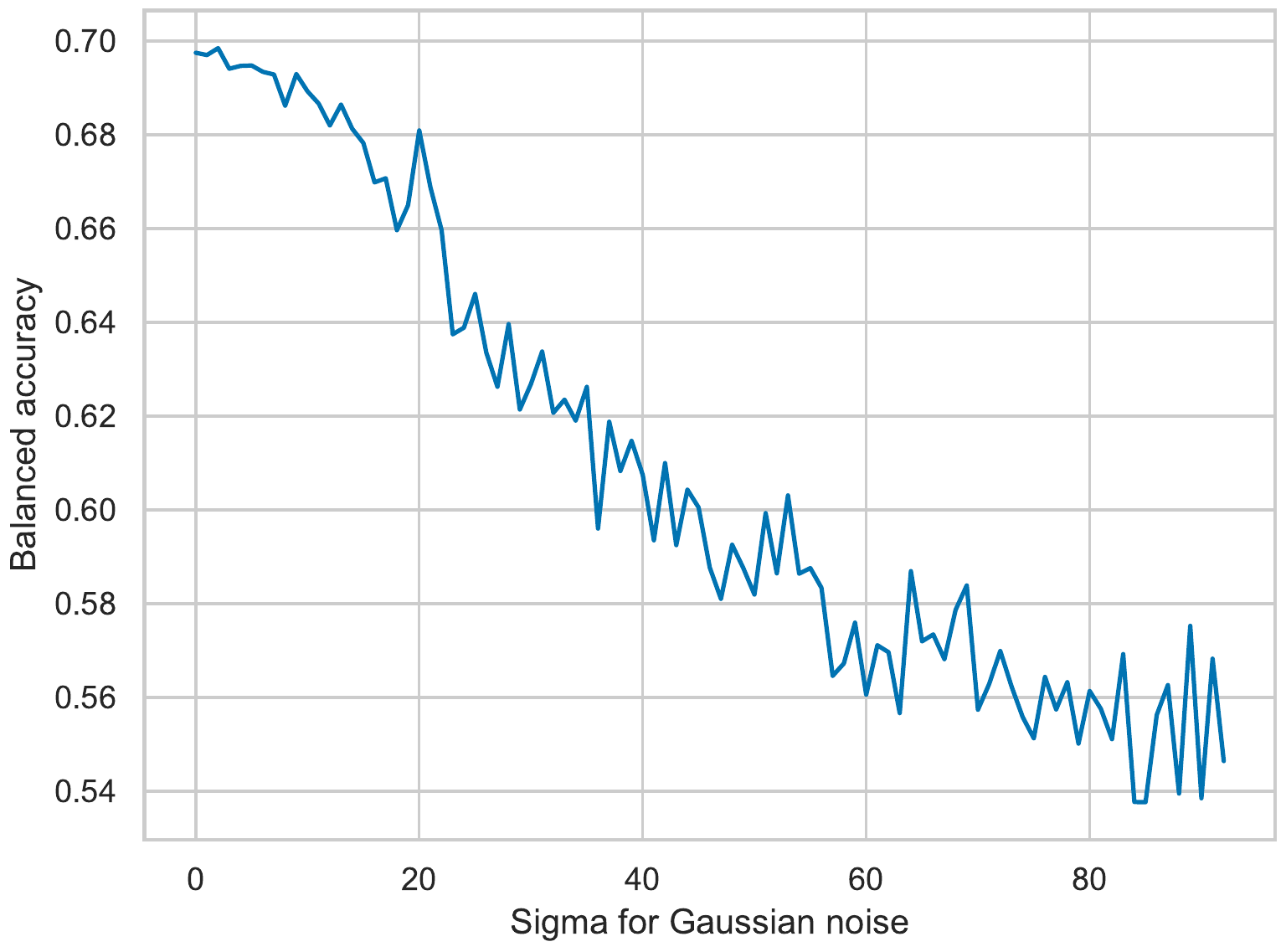} & \includegraphics[width=0.31\linewidth]{images/chexpert_ensemble_sigma_gnmax.pdf} & \includegraphics[width=0.31\linewidth]{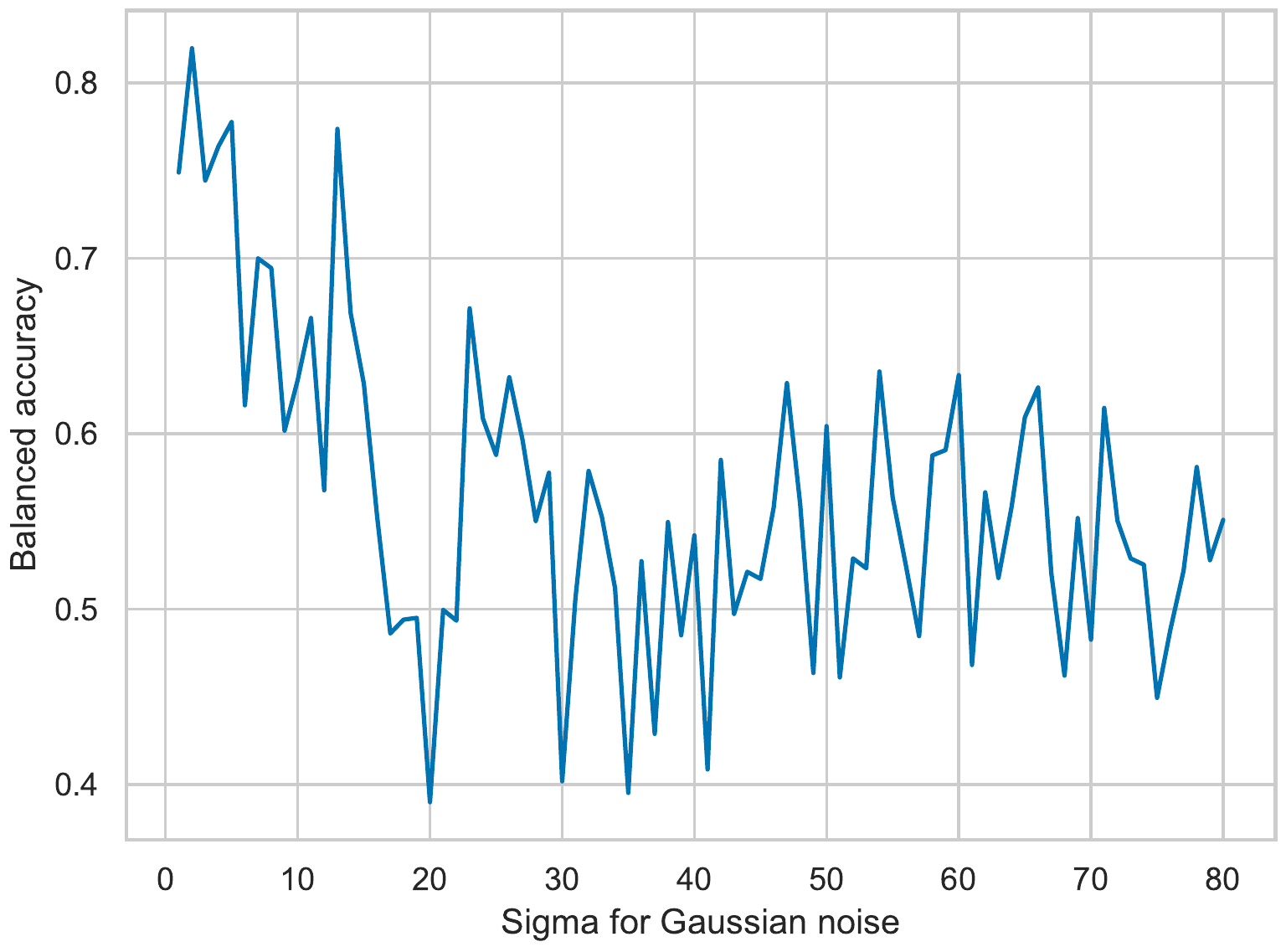} \\
\end{tabular}

\caption{The analysis of the balanced accuracy (y-axis) of votes from the ensemble as we increase the (sigma of the Gaussian) noise (x-axis) from the binary \multi PATE. We search for variance of the noise $\sigma_{G}$ that can preserve high BAC for the ensemble. \textbf{Left}: 50 CheXpert models with train and test on CheXpert, $\sigma_{G} \le 9.0$ preserves more than 0.69 of BAC. \textbf{Middle}: 50 CheXpert models with train on CheXpert and test on PadChest test set, $\sigma_{G} \le 9.0$ preserves more than 0.75 of BAC. \textbf{Right}: 10 PadChest models with train and test on PadChest, $\sigma_{G} \le 5.0$ preserves more than 0.75 of BAC.
\label{fig:ap-chexpert-ensemble-self-pate}}
\end{center}
\end{figure*}

\section{Previous Algorithm for the \Multi Classification}\label{app:prior-work-definitions}

For completeness, we present the previously proposed algorithm~\ref{alg:tau-clipping-privateknn} from~\cite{privateknn2020cvpr}.

\begin{algorithm}
  \caption{\Multi classification with $\tau$-clipping in $\ell_1$-norm from Private kNN by~\cite{privateknn2020cvpr}.}
  \label{alg:tau-clipping-privateknn}
  
  \algorithmicrequire{Data point $x$, clipping threshold $\tau_1$, Gaussian noise scale $\sigma_G$, $n$ teachers, each with model $f_j(x)\in \{0,1\}^k$, where $j \in [n]$.}
  
  \algorithmicensure{Aggregated vector $V \in \{0,1\}^k$ with $V_i=1$ if returned label present, otherwise $V_i=0$, where $i \in [k]$.} 

  \begin{algorithmic}[1]
      \ForAll{teachers $j \in [n]$}
         \State $v_j \gets \min(1, \frac{\tau_1}{\left\lVert f_j(x) \right\rVert_1}) f_j(x)$\Comment{$\tau$-clipping in $\ell_1$ norm}
      \EndFor
      \State $V^1 = \sum_{j=1}^{n} v_j$ \Comment{Number of positive votes per label}
      \State $V^0 = n - V^1$
      \State $V^0 \gets V^0 +  \mathcal{N}(0,\sigma_G)$ \Comment{Add Gaussian noise for privacy protection}
      \State $V^1 \gets V^1 +  \mathcal{N}(0,\sigma_G)$ 
      \ForAll{labels $i \in [k]$}
         \If {$V^1_i > V^0_i$} \Comment{Decide on the output vote}
            \State $V_i = 1$
         \Else
            \State $V_i = 0$
         \EndIf
      \EndFor
  \end{algorithmic}
\end{algorithm}




\section{Data-Dependent Privacy Analysis}
\label{app:data-dependent}
Our binary voting mechanism does leverage both the smooth sensitivity and propose-test-release methods.
For example, the Confident GNMax (proposed by~\cite{papernot2018scalable}) is a form of the propose-test-release method described in Section 3.2 in~\cite{Vadhan2017}.

First explored by~\cite{Cormode2012DifferentiallyPS}, \textbf{differential privacy guarantees can be data-dependent}. These guarantees can lead to better utility with a long history, with several works using them~\cite{Cormode2012DifferentiallyPS,papernot2017semi,papernot2018scalable,chowdhury20aDataDependentDP}. One main caveat is that the released epsilon score must now also be noised because it is itself a function of the data. \cite{papernot2018scalable} provide a way to do this via the smooth sensitivity (Section B in Appendix). Indeed our data-dependent differential privacy guarantees are formally proven and are based on the following intuition. Take the exponential mechanism which gives a uniform privacy guarantee. For instance, when the top score and the second score are very close, applying the exponential mechanism to this data there is a nearly uniform chance of picking either coordinate. However, for some inputs, the utility can be very strong---when the top score is much higher than the second score, then this mechanism is exponentially more likely to pick the top score than the second. This does not require local sensitivity of stability based methods, but is rather derived from the likelihood of picking either coordinate for this mechanism given the gap in the scores.

\section{From Stirling's Approximation to Upper Bound of Factorial}
\label{sec:sterling-factorial}

Stirling's approximation:
\begin{align*}
\sqrt{2\pi n} (n/e)^n \leq n! \leq e^{1/12n} \sqrt{2\pi n} (n/e)^n
\end{align*}

Upper bound of factorial:
\begin{align*}
\binom{n}k&=\frac{n!}{k!(n-k)!}\\
&\le\frac{e^{1/12n} \sqrt{2\pi n} (n/e)^n}{\sqrt{2\pi k}(k/e)^k\sqrt{2\pi(n-k)}((n-k)/e)^{n-k}}\\
&=\frac{e^{1/12n}\sqrt{n}}{\sqrt{2\pi k(n-k)}}\left(\frac{n/e}{k/e}\right)^k\left(\frac{n/e}{(n-k)/e}\right)^{n-k}\\
&=\frac{e^{1/12n}\sqrt{n}}{\sqrt{2\pi k(n-k)}}\left(\frac{n}{k}\right)^k\left(\frac{n}{n-k}\right)^{n-k}\\
&(\text{Note: } \left(\frac{(n-k)+k}{n-k}\right)^{n-k}=\left(1+\frac{k}{n-k}\right)^{n-k}\le e^k)\\
&\le\frac{e^{1/12n}\sqrt{n}}{\sqrt{2\pi k(n-k)}}\left(\frac{n}{k}\right)^k e^k\\
&\le\frac{e^{1/12n}\sqrt{n}}{\sqrt{2\pi(n-1)}}\left(\frac{en}k\right)^k\\
&\le\left(\frac{en}k\right)^k
\end{align*}

\begin{table*}[t]
\caption{\textbf{Model improvements through retraining with \multi CaPC}. DD - denotes Data Dependent. We also present the results with sanitized epsilon values.}
\label{tab:data-stats-final}
\vskip -1.1in
\begin{center}
\begin{small}
\begin{sc}
\begin{tabular}{cclcccccc}\toprule Dataset & \shortstack{\# of \\ Models} &  State & \shortstack{DD \\ $\varepsilon$} & \shortstack{Sanitized \\ $\varepsilon$} &          ACC &                  BAC &          AUC &          MAP \\
\midrule
\multirow{4}{*}{Pascal VOC} & 1 &  Initial &                 - & -& .97 &  .85 & .97 & .85 \\
 & 50 &  Before CaPC &  - & & .93$\pm$.02 &          .59$\pm$.01 &  .88$\pm$.01 &  .54$\pm$.01 \\
& 50 &   After CaPC &                10 & 12.97 & \textbf{.94$\pm$.01} &          .\textbf{62$\pm$.01} &  .88$\pm$.01 &  .54$\pm$.01 \\
& 50 &   After CaPC &                20 & 26.00 &  \textbf{.94$\pm$.01} &          \textbf{.64$\pm$.01} & \textbf{.89$\pm$.01} &  \textbf{.55$\pm$.01} \\
\hline
\multirow{3}{*}{CheXpert} & 1 &      Initial &                 - & &          .79 &                  .78 &          .86 &          .72 \\   
& 50 &  Before CaPC &                 - & -& .77$\pm$.06 &          .66$\pm$.02 &  .75$\pm$.02 &  .58$\pm$.02 \\   
& 50 &   After CaPC &                20 & 25.80 &  .76$\pm$.07 &  \textbf{.69$\pm$.01} &  \textbf{.77$\pm$.01} &  \textbf{.59$\pm$.01} \\
\hline
\multirow{3}{*}{MIMIC} & 1 &      Initial &                 - & -  &      .90 &                  .74 &          .84 &          .51 \\   
& 50 &  Before CaPC &                 - &- & .84$\pm$.07 &          .63$\pm$.03 &  .78$\pm$.03 &  .43$\pm$.02 \\   
& 50 &   After CaPC &                20 & 25.80 &  \textbf{.85$\pm$.05} &          \textbf{.64$\pm$.01} &  \textbf{.79$\pm$.01} &  \textbf{.45$\pm$.03 }\\
\hline
\multirow{3}{*}{PadChest} & 1 &      Initial &                 - &  - &        .86 &                  .79 &          .90 &          .37 \\   
& 10 &  Before CaPC &                 - & -&   .90$\pm$.01 &          .64$\pm$.01 &  .79$\pm$.01 &  .16$\pm$.01 \\   
& 10 &   After CaPC &                20 & 25.80 &  .88$\pm$.01 &          .64$\pm$.01 &  .75$\pm$.01 &  .14$\pm$.01 \\
\bottomrule
\end{tabular}
\end{sc}
\end{small}
\end{center}
\vskip -0.2in
\end{table*}

\end{document}